\documentclass[12pt]{article}

\newcommand{\blind}{1}

\usepackage{hyperref,amssymb,amsmath,amsthm,bm,bbm,mathtools}

\hypersetup{colorlinks=true,
	citecolor=blue,
	anchorcolor = red,
}

\usepackage{fullpage}
\usepackage{epsfig}
\usepackage{enumerate}
\usepackage{enumitem}
\usepackage{xcolor}

\usepackage{caption}
\usepackage{subcaption}
\usepackage{comment}

\renewcommand{\hat}{\widehat}
\renewcommand{\tilde}{\widetilde}

\newcommand{\bfm}[1]{\ensuremath{\boldsymbol{#1}}} 

\def\bbone{\mathbbm{1}} 

\def\ba{\bfm a}   \def\bA{\bfm A}

     \def\EE{\mathbb{E}}

     \def\PP{\mathbb{P}}
     
     \def\RR{\mathbb{R}}

\def\bu{\bfm u}     
\def\bv{\bfm v}     
     
\def\bx{\bfm x}     
\def\by{\bfm y}     
\def\bz{\bfm z}

\def\calD{{\cal  D}} 
 \def\cE{{\cal  E}}
 
\def\calG{{\cal  G}} \def\cG{{\cal  G}}
\def\calH{{\cal  H}} 
\def\calI{{\cal  I}} \def\cI{{\cal  I}}

\def\calL{{\cal  L}} 
 
\def\calN{{\cal  N}} 
 
\def\calP{{\cal  P}}

\def\calU{{\cal  U}}

\newcommand{\bfsym}[1]{\ensuremath{\boldsymbol{#1}}}

 \def\bbeta{\bfsym \beta}
 \def\bgamma{\bfsym \gamma}             
 \def\bdelta{\bfsym {\delta}}           
               
 \def\bmu{\bfsym {\mu}}                 
 
 \def\btheta{\bfsym {\theta}}

              \def\bSigma{\bfsym \Sigma}

\def\one {\mathbbm{1}}

                  \def\hbbeta{\hat{\bfsym \beta}}

               \def\hbtheta {\hat{\bfsym {\theta}}}

\providecommand{\abs}[1]{\left\lvert#1\right\rvert}
\providecommand{\norm}[1]{\left\lVert#1\right\rVert}

\providecommand{\angles}[1]{\left\langle #1 \right\rangle}
\providecommand{\paren}[1]{\left( #1 \right)}
\providecommand{\brackets}[1]{\left[ #1 \right]}
\providecommand{\braces}[1]{\left\{ #1 \right\}}

\providecommand{\defeq}{:=}

\usepackage{mathtools}
\DeclarePairedDelimiterX{\infdivx}[2]{(}{)}{%
  #1 \; \delimsize\| \; #2%
}

\DeclareMathOperator{\sgn}{sgn}

\DeclareMathOperator{\Var}{Var}

\DeclareMathOperator{\Reg}{Reg}
\DeclareMathOperator{\reg}{reg}
\newtheorem{definition}{Definition}

\newtheorem{lemma}[definition]{Lemma}

\newtheorem{theorem}[definition]{Theorem}

\newtheorem{remark}{Remark}

\theoremstyle{definition}

\newcommand{\bigO}[1]{ \mathcal{O} \left( #1 \right) }

\definecolor{royalpurple}{rgb}{0.47, 0.32, 0.66}

\definecolor{royalpurple}{rgb}{0.47, 0.32, 0.66}

\def\beq{\begin{equation}}
\def\eeq{\end{equation}}

\def\bet{\begin{theorem}}
\def\eet{\end{theorem}}

\def\bel{\begin{lemma}}
\def\eel{\end{lemma}}

\def\cond{\;|\;}

\usepackage{setspace}
\usepackage{etoolbox}
\BeforeBeginEnvironment{equation*}{\begin{singlespace}\vspace*{-\baselineskip}}
	\AfterEndEnvironment{equation*}{\end{singlespace}\noindent\ignorespaces}
\BeforeBeginEnvironment{equation}{\begin{singlespace}\vspace*{-\baselineskip}}
	\AfterEndEnvironment{equation}{\end{singlespace}\noindent\ignorespaces}
\BeforeBeginEnvironment{align}{\begin{singlespace}\vspace*{-\baselineskip}}
	\AfterEndEnvironment{align}{\end{singlespace}\noindent\ignorespaces}
\BeforeBeginEnvironment{align*}{\begin{singlespace}\vspace*{-\baselineskip}}
	\AfterEndEnvironment{align*}{\end{singlespace}\noindent\ignorespaces}
\BeforeBeginEnvironment{eqnarray}{\begin{singlespace}\vspace*{-\baselineskip}}
	\AfterEndEnvironment{eqnarray}{\end{singlespace}\noindent\ignorespaces}
\BeforeBeginEnvironment{eqnarray*}{\begin{singlespace}\vspace*{-\baselineskip}}
	\AfterEndEnvironment{eqnarray*}{\end{singlespace}\noindent\ignorespaces}

\usepackage[framemethod=TikZ]{mdframed} 

\usepackage{fancybox}

\usepackage[linesnumbered,ruled,vlined]{algorithm2e}

\SetCommentSty{mycommfont}

\SetKwInput{KwInput}{Input}                
\SetKwInput{KwOutput}{Output}              

\usepackage{listings}             
\newcommand{\oC}{\overline{C}}
\DeclareMathOperator*{\argmax}{\arg\!\max}

\usepackage[top=1in, bottom=1in, left=1in, right=1in]{geometry}

\usepackage{setspace}

\usepackage[authoryear]{natbib}  

\definecolor{DSgray}{cmyk}{0,1,0,0}

\usepackage{graphicx}
\graphicspath{{figs/}}

\begin{document}
\pagenumbering{arabic}

\def\spacingset#1{\renewcommand{\baselinestretch}%
{#1}\small\normalsize} \spacingset{1}

%
%
%

\def\TITLE{High-Dimensional Linear Bandits under Stochastic Latent Heterogeneity}

\if1\blind
{
\title{\bf \TITLE}
\author{
Elynn Chen$^\diamondsuit$ \hspace{2ex}
Xi Chen$^\natural$
\hspace{2ex}
Wenbo Jing$^\sharp$ \hspace{2ex}
Xiao Liu$^\flat$ \hspace{2ex} \\ \normalsize
\medskip
$^{\diamondsuit,\natural,\sharp,\flat}$ Stern School of Business, New York University 
}
\date{}
\maketitle
} \fi

\if0\blind
{
  \bigskip
  \bigskip
  \bigskip
  \begin{center}
    {\LARGE\bf \TITLE}
\end{center}
  \medskip
} \fi

\bigskip
\begin{abstract}
\spacingset{1.2}
This paper addresses the critical challenge of {\it stochastic latent heterogeneity} in online decision-making, where individuals' responses to actions vary not only with observable contexts but also with unobserved, randomly realized subgroups. Existing data-driven approaches largely capture observable heterogeneity through contextual features but fail when the sources of variation are latent and stochastic. We propose a {\it latent heterogeneous bandit framework} that explicitly models probabilistic subgroup membership and group-specific reward functions, using promotion targeting as a motivating example. Our phased EM-greedy algorithm jointly learns latent group probabilities and reward parameters in high dimensions, achieving optimal estimation and classification guarantees.

Our analysis reveals a new phenomenon unique to decision-making with stochastic latent subgroups: randomness in group realizations creates irreducible classification uncertainty, making sub-linear regret against a fully informed strong oracle fundamentally impossible. We establish matching upper and minimax lower bounds for both the strong and regular regrets, corresponding, respectively, to oracles with and without access to realized group memberships. The strong regret necessarily grows linearly, while the regular regret achieves a minimax-optimal sublinear rate. These findings uncover a fundamental stochastic barrier in online decision-making and point to potential remedies through simple strategic interventions and mechanism-design-based elicitation of latent information.
\end{abstract}

\noindent%
{\it Keywords}: High-dimensional estimation; Stochastic bandits; Latent structure modeling; Logistic models; Minimax regret
\vfill

%
%

\spacingset{1.78}


\section{Introduction}  \label{sec:intro}

Latent heterogeneity plays a critical role across domains such as economics \citep{blundell2007labor,bonhomme2022discretizing}, business \citep{cherry1999unobserved,lewis2024latent}, and healthcare \citep{zhou2018using,chen2024reinforcement}. It arises from the inherent complexity of human behavior, where individuals respond differently to the same policy or action due to unobserved psychological, experiential, or social factors. For example, consumer preferences often extend beyond measurable demographics, reflecting hidden attitudes or experiences that shape decision outcomes. Existing data-driven decision frameworks largely capture only observable heterogeneity through contextual features \citep{chen2021statisticala,xu2022langevin,chen2024nearly,xu2025linear,chen2024reinforcement,chen2025transfer,chai2025deep,chai2025transfer,wang2024online,zhou2025stochastic}, leaving latent and stochastic sources of variation largely unexplored.

To address this gap, we develop a new framework that explicitly accounts for {\it stochastic latent heterogeneity} in decision-making. Unlike traditional contextual bandits that assume fully observable features, our model allows each decision instance to arise from an {\it unobserved, randomly realized subgroup} that influences the reward structure. This stochastic latent subgroup formulation captures the randomness of human responses beyond measurable covariates and naturally encompasses scenarios such as personalized promotion targeting, dynamic pricing, and individualized treatment assignment.

We illustrate the proposed framework through a {\it promotion targeting} example, though it extends naturally to other sequential personalization problems. Consider an e-commerce platform that, at each time point $i \in [T]$, decides which of $K$ coupons to offer to an incoming customer. The platform observes contextual information $\bz_i$ describing customer, product, and environmental features, and constructs action-specific feature vectors $\bx_{i,k} \in \mathbb{R}^d$ combining coupon and customer attributes. The observed reward $y_{i,k}$ represents the customer's spending after receiving coupon $k$. Empirical studies show that purchase responses vary not only with observable covariates but also with {\it unobserved behavioral factors} such as lifestyle or brand affinity \citep{cherry1999unobserved,bonhomme2015grouped,lewis2024latent,hess2024latent}. We model this {\it latent heterogeneity} by assuming that each customer belongs to one of two latent subgroups, $g_i \in \{1,2\}$, which can be naturally extended to any finite number of latent subgroups within our framework. Our real data analysis (Section \ref{sec:real}) identifies these two latent groups as high-consumption users, who tend to spend substantially more and respond strongly to cash-bonus incentives, and low-consumption users, whose spending remains relatively stable and less sensitive to promotions. These empirically discovered groups align well with economically meaningful behavioral differences.

Each subgroup generates a distinct reward response for the available actions $k\in[K]$: 
\begin{equation}\label{eqn:hetero-f}
\begin{aligned}
(y_{i,k} \cond g_i = 1) = f_1(\bx_{i,k}) + \epsilon_i , \quad\text{and}\quad
(y_{i,k} \cond g_i = 2) = f_2(\bx_{i,k}) + \epsilon_i, 
\end{aligned}
\end{equation}
where $f_1(\cdot)\ne f_2(\cdot)$ and $\epsilon_i$ is a mean-zero random noise. 
Unlike mixture or clustering bandit models with fixed group assignments \citep{zhou2016latent,gentile2017context}, we assume each subgroup is {\it stochastically realized} according to a feature-dependent probability
\begin{equation} \label{eqn:hetero-status}
\begin{aligned}
\Pr(g_i = 1 \cond \bz_i) = p(\bz_i^\top \btheta^*), \quad\text{and}\quad
\Pr(g_i = 2 \cond \bz_i) = 1-p(\bz_i^\top \btheta^*),
\end{aligned}
\end{equation}
where $p(x) = 1/\paren{1 + \exp(-x)}$.
Although models \eqref{eqn:hetero-f} and \eqref{eqn:hetero-status} resemble the high-dimensional mixed regression formulations studied in \citet{zhang2020estimation} and \cite{javanmard2025prediction}, our setting is fundamentally distinct. Those works address supervised prediction under static samples, whereas we study a sequential decision problem in which different actions $k$ are chosen adaptively at each round. This coupling of latent-structure estimation and policy optimization introduces dynamic feedback between learning and decision-making that has no analogue in the standard mixed regression framework.

The proposed latent heterogeneous bandit model, characterized by \eqref{eqn:hetero-f} and \eqref{eqn:hetero-status}, differs fundamentally from the classic stochastic bandit in two ways: the reward functions $f_1$ and $f_2$ vary across latent groups, and group memberships $g_i$ are unobservable. This combination of hidden subgroup structure and group-specific rewards creates challenges that existing methods cannot address.
Simple classification followed by group-specific bandit learning is infeasible because the latent groups are never directly observed, and standard clustering techniques combined with group-specific bandit algorithms are inadequate because cluster membership may not align with the underlying reward associations, the key source of heterogeneity in decision outcomes.

In contrast to prior online clustering bandits \citep{gentile2014online,zhou2016latent,gentile2017context}, which assume {\it fixed but unknown} group assignments, our model allows group membership to be {\it stochastic and feature-dependent}. 
This probabilistic structure introduces an irreducible misclassification rate, rendering sub-linear strong regret theoretically unattainable. Moreover, their graph- and set-based clustering methods on coefficients rely on accurate per-user estimation before clustering, requiring many samples per user and performing poorly in high dimensions. Our framework departs from this paradigm by modeling latent heterogeneity at the population level and employing a regularized Expectation–Maximization (EM) procedure to jointly estimate subgroup probabilities and high-dimensional reward parameters without per-user pre-estimation.

By addressing these challenges, we make three main contributions.
First, we introduce a {\it new modeling framework} for online decision-making under {\it stochastic} latent heterogeneity, where subgroup memberships are probabilistically realized rather than fixed. This framework provides a principled foundation for analyzing sequential decisions when individuals with similar observable contexts may respond differently due to unobserved, random factors.

Second, focusing on the linear reward setting, we develop a phased EM–greedy algorithm that jointly learns latent group probabilities and group-specific reward parameters in high dimensions. The method integrates regularized estimation and sequential decision-making, yielding high-probability guarantees for parameter estimation and group classification.

Third, we propose new theoretical constructs (strong and regular oracles/regrets) that distinguish between policies with and without access to realized latent groups. We derive corresponding upper and minimax lower bounds for both regrets, revealing a {\it new fundamental phenomenon}: due to the randomness of subgroup realization, strong regret cannot be made sub-linear by algorithmic improvement. This stochastic barrier exposes an intrinsic limit of learning under latent uncertainty and motivates new directions that combine mechanism design and strategic information elicitation to achieve further performance gains.

\subsection{Related Works}\label{sec:related-works}

Our work intersects online stochastic contextual bandits, statistical learning, and decision-making under latent heterogeneity. We review the most relevant literature and distinguish our contributions.

\smallskip
\noindent
\textbf{Online Stochastic Contextual Bandits.}
The bandit problem has been extensively studied across machine learning and statistics \citep{li2019dimension,xu2022langevin,bian2025off,ren2024dynamic,chen2025express,xu2025multitask,huang2025optimal};
see \cite{lattimore2020bandit} for a comprehensive review. Contextual bandits incorporate additional information to predict action quality \citep{auer2002using,dani2008stochastic,li2010contextual,chu2011contextual,ji2022risk}. While adversarial settings achieve $\bigO{d\sqrt{T}}$ regret \citep{abbasi2011improved}, stochastic settings -- suitable for applications like news recommendation and clinical trials -- can improve bounds dependent on $T$ from $\sqrt{T}$ to $\log(T)$ for homogeneous bandits.

In low dimensions, \cite{goldenshluger2013linear} achieved $\bigO{d^3\log(T)}$ regret, though this becomes unfavorable as the dimension increases. For high-dimensional settings with sparsity $s$ ($s \ll d$), several approaches have emerged: \cite{bastani2020online}'s LASSO bandit achieved $\bigO{K s^2 \log^2(dT)}$, 
\cite{wang2024online}'s MCP-based method improved this to $\bigO{Ks^2(s+\log d)\log T}$. 
In contrast, \cite{oh2020sparsity} proposed a sparse-agnostic approach that achieves 
$\bigO{\sqrt{sT\log(dT)}}$ under restricted eigenvalue conditions. 

All of the aforementioned studies work within the classic stochastic bandit framework, which does not address the widely encountered setting of latent heterogeneity in business and economics. When such unobserved heterogeneity is ignored, a single linear expectation becomes misspecified, leading to biased parameter estimates and persistent per-round regret.
In contrast, our proposed latent heterogeneous bandit model explicitly incorporates an unobserved subgroup structure. The proposed algorithm is designed to jointly learn both the latent group memberships and the corresponding group-specific parameters, thereby mitigating model misspecification and enabling more accurate decision-making.

\smallskip
\noindent
\textbf{Statistical learning with Latent heterogeneity.}
In linear regression settings, researchers have explored two major approaches to latent heterogeneity. The non-parametric approach employs grouping penalization on pairwise differences \citep{shen2010grouping,ke2013homogeneity,ma2017concave}, avoiding distributional assumptions but limiting predictive capabilities for new samples. For prediction-oriented tasks, mixture-model-based approaches have shown greater promise. Key developments include: rigorous EM algorithm guarantees for symmetric mixed linear regressions (MLR) where the mixing proportion is known to be $1/2$ \citep{balakrishnan2017statistical}, efficient fixed-parameter algorithms \citep{li2018learning}, computational-statistical tradeoffs \citep{fan2018curse}, improved EM convergence analysis \citep{mclachlan2019finite,klusowski2019estimating}, robust estimation under corruptions \citep{shen2019iterative}, high-dimensional MLR with unknown but constant mixing proportions \citep{zhang2020estimation}, and convergence analysis for federated learning \citep{su2024global, niu2024collaborative}. Our model \eqref{eqn:hetero-f} and \eqref{eqn:hetero-status} extends this literature by addressing variable mixing proportions in high-dimensional MLR.

Compared with high-dimensional mixed regression models with feature-dependent mixing probabilities \citep{zhang2020estimation,javanmard2025prediction}, our setting is more difficult: those works study offline prediction with static samples, whereas we must integrate estimation with sequential decision-making, where even minor misclassification errors accumulate into long-run regret. This coupling motivates two distinct regret notions (i.e., strong and regular regret) and their corresponding minimax analyses that expose the fundamental performance limits under stochastic latent groups.

\smallskip
\noindent
\textbf{Decision with Heterogeneity.}
Research addressing learning and decision-making with heterogeneity remains relatively sparse. Notable work spans several domains: personalized dynamic pricing with high-dimensional features and heterogeneous elasticity \citep{ban2021personalized}; regime-switching bandits with temporal heterogeneity \citep{cao2019nearly,zhou2021regime}; and convergence analysis of Langevin diffusion under mixture distributions \citep{dong2022spectral}, where multiple density components can significantly impact sampling efficiency. In sequential decision settings, recent work has explored policy evaluation and optimization with latent heterogeneity in reinforcement learning \citep{chen2024reinforcement,bian2025off}. However, these approaches focus primarily on multi-stage aspects of reinforcement learning, leaving unexplored the fundamental challenge that latent heterogeneity poses to the exploration-exploitation trade-off. Our work addresses this gap through the lens of the bandit problem, essentially a one-step reinforcement learning.

\subsection{Notations and Organization}\label{sec:notations}

For a positive integer $n$, let $[n] := {1,\dots,n}$. For any vector $\bv$, $\norm{\bv}_0$, $\norm{\bv}_1$, and $\norm{\bv}_2$ denote the $\ell_0$ (number of non-zero elements), $\ell_1$, and $\ell_2$ norms respectively. For a matrix $\bA$, $\lambda_{\min}(\bA)$ and $\lambda_{\max}(A)$ denote its minimum and maximum eigenvalues. For positive sequences ${a_n}$ and ${b_n}$, we write $a_n \lesssim b_n$, $a_n=\bigO{b_n}$, or $b_n=\Omega(a_n)$ if there exists $C > 0$ such that $a_n \leq Cb_n$ for all $n$. We write $a_n \asymp b_n$ if $a_n \lesssim b_n$ and $b_n \lesssim a_n$.

The remainder of this paper is organized as follows. Section \ref{sec:model} formulates our latent heterogeneous bandit model and defines two types of regret measures. Section \ref{sec:method} presents our proposed methodology. Section \ref{sec:theory} establishes theoretical guarantees, including estimation error bounds, misclassification rates, and minimax optimal regret bounds. Sections \ref{sec:numerical} validate our approach through simulation studies and an empirical application using cash bonus data from a mobile commerce platform. Section \ref{sec:conclusion} concludes with discussions.

\section{Problem Formulation} \label{sec:model}

In this section, we formulate the linear bandits problem under latent heterogeneity. Section \ref{sec:problem} introduces the latent heterogeneous linear bandit model, which extends the classical stochastic linear bandit setting by incorporating unobserved group structures. Section \ref{sec:two-type-regret} defines two types of regret—strong regret and regular regret—that evaluate the performance of a policy in the presence of latent heterogeneity.

\subsection{Latent Heterogeneous Linear Bandits}\label{sec:problem}

The latent heterogeneous bandit model \eqref{eqn:hetero-f} and \eqref{eqn:hetero-status} introduced at the beginning of this paper is a general framework that allows for arbitrary functional forms of the mean rewards $f_1(\bx)$ and $f_2(\bx)$. We focus on linear functional form in this work and leave non-linear and non-parametric function approximation for future research.
Without loss of generality, we consider the setting for two latent subgroups, as the extension to any known finite number of latent subgroups follows naturally. 
Each customer $i\in[T]$ is characterized by a customer feature $\bz_i\in\RR^{d'}$. 
For any customer $i$, there are $K$ possible arms (coupons) to offer. 
The combined features of customer $i$ and an arm $k$ are denoted as $\bx_{i,k}\in\RR^d$. 
The latent heterogeneous linear bandits are characterized by
\begin{equation} \label{eqn:lhlb}
\begin{aligned}
&\text{(Subgroup model):} & \Pr(g_i = 1 \cond \bz_i) = p(\bz_i^\top \btheta^*), \quad
\Pr(g_i = 2 \cond \bz_i) = 1-p(\bz_i^\top \btheta^*), \\
&\text{(Reward model):} & (y_{i,k} \cond g_i = 1) = \angles{\bx_{i,k}, \bbeta_1^*} + \epsilon_i , \quad
(y_{i,k} \cond g_i = 2) = \angles{\bx_{i,k}, \bbeta_2^*} + \epsilon_i, 
\end{aligned}
\end{equation}
where $\bbeta_1^*\ne\bbeta_2^*$, $p(x) = 1/\paren{1 + \exp(-x)}$, and $\epsilon_i \sim \calN(0, \sigma^2)$. We refer to the two equations as the ``subgroup model'' and the ``reward model.''

For each customer $i$, while the contextual features $\big\{ \{\bx_{i,k}\}_{k\in[K]}, \; \bz_i \big\}$ are observable, the true group membership $g_i$ remains unknown. 
The objective of the platform is to select one coupon $k\in[K]$ for each customer $i$ to maximize the aggregated rewards across all $T$ customers. 
Our goal is to design a sequential decision rule (policy) $\pi$ that maximizes the expected cumulative reward over the time horizon while simultaneously estimating the model parameters and predicting the latent group $g_i$. 

\begin{remark}
The proposed model \eqref{eqn:lhlb} can be naturally extended to settings with more than two latent subgroups. 
Specifically, let $g_i \in \{1, \ldots, G\}$ denote the unobserved group membership, where 
$
\Pr(g_i = g \mid \bz_i) = \exp(\bz_i^{\top}\btheta_g^*) / \sum_{g=1}^G \exp(\bz_i^{\top}\btheta_g^*)$ for $ 
g = 1, \ldots, G$, and $
(y_{i, k} \mid g_i=g) = \bx_{i, k}^\top \bbeta_g^* + \epsilon_i$.
The parameters $\big\{(\btheta_g^*, \bbeta_g^*)\big\}_{g \in [G]}$ can be estimated 
via a regularized multi-class EM algorithm similar to Algorithm \ref{alg:em}. However, this generalization is technically straightforward and does not yield new conceptual insights. In fact,
the binary model already captures the key challenges of latent heterogeneity, 
while larger $G$ mainly increases notational complexity. 
Therefore, we focus on the case $G=2$ throughout the paper for clear presentation.

The model can also be generalized beyond the linear setting. 
In particular, one may consider a generalized linear formulation
$
\mathbb{E}[y_{i, k} \,|\,\bx_{i, k}, g_i] = h^{-1}(\bx_{i, k}^\top \bbeta_{g_i}^*),
$
where $h(\cdot)$ is a known link function.
Analyzing this model would require 
new techniques for high-dimensional generalized EM estimation 
and is therefore left for future research.
\end{remark}

\subsection{Two Types of Regrets under Stochastic Latent Heterogeneity}\label{sec:two-type-regret}

To evaluate the performance of any policy $\pi$, we must account for an important source of randomness: the subgroup model in \eqref{eqn:lhlb}  only specifies probabilities of group assignments rather than deterministic membership. 
As a result, when a customer $i$ arrives, there exist two different types of oracle: (1) the ``ex-post'' oracle who is able to precisely predict the true realized group membership $g_i$, and (2) the ``ex-ante'' oracle who knows the true parameter $\btheta^*$ of the subgroup model and thus knows the group probability $p(\bz_i^{\top}\btheta^*)$. 

This distinction gives rise to two types of regret measures when comparing against optimal policies derived from the two types of oracles. 
Let $\pi^{*}$ denote the {\it strong oracle rule}, which ``knows'' not only the true parameters $\bbeta_1^*$, $\bbeta_2^*$, and $\btheta^*$, but also the realized group $g_i$ beyond the probabilistic structure of the subgroup model. 
For each customer $i$, the strong oracle rule prescribes
\begin{equation}
a_i^{*} = \underset{k \in [K]}{\arg \max} \; \angles{\bx_{i,k}, \bbeta^*_{g_i}}.
\end{equation}
Alternatively, we let $\tilde{\pi} $ denote the  {\it regular oracle rule}, which ``knows'' the true parameters $\bbeta_1^*$, $\bbeta_2^*$, and $\btheta^*$, but {\em not} the realized group $g_i$. 
For each customer $i$, the regular oracle rule prescribes
\begin{equation}
\tilde{a}_i = \underset{k \in [K]}{\arg \max} \; \angles{\bx_{i,k}, \bbeta^*_{\tilde g_i}},
\end{equation}
where $\tilde g_i$ is estimated group using the oracle parameter $\btheta^*$  according to the decision rule in \eqref{eqn:lhlb}, i.e., $\tilde g_i = 1$ if $p(\bz_i^{\top}\btheta^*) \geq 1/2$ and $\tilde g_i=2$ otherwise.

To evaluate any allocation policy $\hat{\pi}$, we measure its performance relative to the two oracle rules.
Let $\hat{a}_i$ denote the action chosen by policy $\hat{\pi}$ for customer $i$. We define the {\em instant strong regret} comparing against the strong oracle,
\begin{equation}
\mathrm{reg}_i^{*} = 
	\underset{k\in [K]}{\max}\; \angles{\bx_{i,k}, \bbeta^*_{g_i}} - \angles{\bx_{i,\hat a_i}, \bbeta^*_{g_i}}, 
\end{equation}
and the  {\em instant regular regret} comparing against the regular oracle,
\begin{equation}  
\tilde{\mathrm{reg}}_i = 
	 \angles{\bx_{i,\tilde a_i}, \bbeta^*_{g_i}} - \angles{\bx_{i,\hat a_i}, \bbeta^*_{g_i}}.
\end{equation}
The expected cumulative strong and regular regret at time $T$ are respectively defined as
\begin{equation}
{\mathrm{Reg}}^*(T) = \EE\brackets{\sum_{i=1}^{T} \mathrm{reg}^*_i},  \quad\text{and}\quad
\tilde{\mathrm{Reg}} (T) = \EE\brackets{\sum_{i=1}^{T} \tilde{\mathrm{reg}}_i},
\end{equation}
where the expectation is taken over the randomness in the feature vectors $\{\bx_{i,k}\}_{1\le k \le K}$ and $\bz_i$, group membership $g_i$, and the stochastic noise $\epsilon_i$. 
Our objective is to develop a policy that minimizes both types of expected cumulative regrets. 

\section{Learning and Decision under Latent Heterogeneity} \label{sec:method}

In this section, we present our methodological framework for addressing latent heterogeneity in the linear bandits problem. We first propose our phased learning algorithm that accounts for the latent group structures in Section \ref{sec:phased}, followed by a tailored expectation-maximization (EM) algorithm for parameter estimation in Section \ref{sec:EM}.

\subsection{Phased Learning and Greedy Decisions}\label{sec:phased}

\begin{algorithm}[ht!]
	\caption{Phased Learning and Greedy Decision under Latent Heterogeneity}
	\label{alg:em-bandit}
	\SetKwInOut{Input}{Input}
	\SetKwInOut{Output}{Output}
	\Input{Features $\{\bx_{i, k}, k \in [K]\}$ and $\bz_i$ for sequentially arriving customers $i$, and the minimal episode length $n_0$.}
	
	
	
	\For{each episode $\tau=0, 1, 2, \dots$}{
		Set the length of the $\tau$-th episode as $N_{\tau} = 2^{\tau}n_0$ and define an index set $\calI_{\tau}$ with cardinality $N_{\tau}$;
		
		\If{$\tau=0$}{
			\For{$i\in\cI_{\tau}$}{
				Receive features $\{\bx_{i,k}, \bz_{i}\}_{k \in [K]}$ ;\\
				Select $a_i \sim {\rm Uniform}([K])$;\\
				Receive the reward $y_i$;
			}
		}
		\Else{
			
			
			Call Algorithm \ref{alg:em} ``Learning under Latent Heterogeneity'' to obtain $\hat{\btheta}^{(\tau)}$, $\hat{\bbeta}_1^{(\tau)}$, and $\hat{\bbeta}_2^{(\tau)}$ using data collected in the $(\tau - 1)$ -th episode, i.e.,  $\calD_{\tau-1}$;
			
			\For{$i\in\cI_{\tau}$}{
				
				Receive features $\{\bx_{i,k}, \bz_{i}\}_{k \in [K]}$;\\
				Predict the group membership $\hat{g}_i = 1$ if $\bz_i^{\top}\hat{\btheta}^{(\tau)}\ge 0$ and $\hat{g}_i = 2$ otherwise;
				
				Prescribe the optimal action based on estimation and prediction, that is, 
				\[
				a_i=\argmax_{k \in [K]}\; \angles{\bx_{i,k}, \hat{\bbeta}_{\hat{g}_i}^{(\tau)}};
				\]
				Receive the reward $y_i$;
			}
		}
		Collect the dataset $\calD_{\tau}=\{y_i, \bx_{i, a_i}, \bz_{i}\}_{i \in \cI_{\tau}}$;
	}
\end{algorithm}

Our proposed method exploits a key structural property of model \eqref{eqn:lhlb}: customers within the same latent group share common reward parameter $\bbeta^*_1$ or $\bbeta^*_2$ for different actions. The reward observed from taking an action provides information about the rewards of other potential actions due to the shared parametric structure. We leverage this property to develop an exploration-free algorithm that achieves minimax optimal regret.

The proposed Algorithm \ref{alg:em-bandit} implements a phased learning approach that divides the time horizon into non-overlapping episodes, indexed by $\tau\ge 0$, and let $i \ge 1$ index sequentially arriving customers.
In the initial episode 0, the actions are uniformly selected from the $K$ arms, which generates the necessary samples for the learning procedure in episode 1.

For subsequent episodes ($\tau \geq 1$), model parameters $\hat{\btheta}$, $\hat{\bbeta}_1$, and $\hat{\bbeta}_2$ are updated at the start of each episode using Algorithm \ref{alg:em}, which employs expectation-maximization (EM) iterations and is detailed in Section \ref{sec:EM}. The updates utilize only the samples $\calD_{\tau-1}$ collected from the previous episode.

With the updated parameter estimates, actions are chosen greedily by first predicting the customer's group membership $\hat{g}_i$ using current $\btheta^*$ estimates, then selecting the action that maximizes expected reward under the predicted group and current $\hbbeta_{\hat{g}_i}$ estimates.
The length of each episode, denoted by $N_{\tau}$,  increases geometrically as $N_{\tau} = n_0 2^{\tau}$, allowing for a more accurate estimate as the episodes progress.
While the algorithm terminates at the end of the horizon (time $T$), it does not require prior knowledge of $T$.

\subsection{Learning under Latent Heterogeneity and High-dimensionality} \label{sec:EM}

We now present the details of the learning procedure under latent heterogeneity in Algorithm \ref{alg:em}. 
Given the samples $\calD_{\tau-1}=\braces{ y_i, \bx_i := \bx_{i, a_i}, \bz_i}_{i \in \calI_{\tau-1}}$ collected in episode $\tau-1$, our goal is to estimate the unknown parameters in model \eqref{eqn:lhlb} via maximum likelihood estimator (MLE).
For notational clarity, we denote the index set of the samples input to Algorithm \ref{alg:em} as $\cI_{\tau-1}$ with size $N_{\tau-1}=\abs{\cI_{\tau-1}}$. 
The MLE maximizes the log-likelihood of the observed data $\calD_{\tau-1}$: 
\begin{equation}\label{eqn:llklh}
	\ell_{N_{\tau-1}}(\bgamma) = \frac{1}{N_{\tau-1}}\sum_{i \in \cI_{\tau-1}}\log\brackets{p\paren{\bz_i^\top\btheta} \cdot
		\phi\paren{ \frac{y_i - \bx_i^\top\bbeta_1}{\sigma} }
		+
		\paren{1-p\paren{\bz_i^\top\btheta}} \cdot
		\phi\paren{ \frac{y_i - \bx_i^\top\bbeta_2}{\sigma}}},
\end{equation}
where we denote the unknown parameter by $\bgamma=(\btheta,\bbeta_1,\bbeta_2)$ and the standard normal density function by $\phi(\cdot)$. 
Directly searching for the maximizer of $\ell_{N_{\tau-1}}(\bgamma)$ is computationally intractable due to its non-convexity. Moreover, in the early episodes ($\tau$ is small), the dimension $d$ can substantially exceed the available sample size, making the estimation problem statistically infeasible.

\begin{algorithm}[t!]
	\caption{Learning under Latent Heterogeneity in Episode $\tau$ ($\tau \geq 1$)}
	\label{alg:em}
	\SetKwInOut{Input}{Input}
	\SetKwInOut{Output}{Output}
	\Input{Batch data $\calD_{\tau-1}=\braces{y_i, \bx_i := \bx_{i, a_i}, \bz_i}_{i \in \cI_{\tau-1}}$, batch size $N_{\tau-1} = \abs{\cI_{\tau-1}}$, initial estimators $\bgamma^{(\tau, 0)}=\big(\btheta^{(\tau, 0)}, \bbeta_1^{(\tau, 0)},\bbeta_2^{(\tau, 0)}\big)$, maximum number of iterations $t_{\tau, \max}$, regularization parameters $\big\{\lambda_{n_{\tau}}^{(t)}\big\}_{t \in [t_{\tau, \max}]}$. \\
	}
	
	\Output{estimates $\hat{\bgamma}^{(\tau)}=\big(\hat{\btheta}^{(\tau)},\hat{\bbeta}_1^{(\tau)},\hat{\bbeta}_2^{(\tau)}\big)$.}
	
	Split $\calI_{\tau-1}$ into $t_{\tau, \max}$ subsets $\big\{\cI_{\tau-1}^{(t)}\big\}_{t \in [t_{\tau, \max}]}$, each of size $n_{\tau}=N_{\tau-1}/t_{\tau, \max}$; 
	
	\For{$t=1, \dots, t_{\tau, \max}$}{
		
		
		For each $i\in\calI_{\tau-1}^{(t)}$, calculate $\omega_{i}^{(\tau, t)} = \omega\paren{y_i, \bx_i,  \bz_i; \bgamma^{(\tau, t-1)}}$, where $\omega$ is defined by
		\begin{equation} \label{eqn:omega-main}
			\omega(y, \bx, \bz ; \bgamma) 
			= \frac{p(\bz^\top\btheta) \cdot
				\phi\paren{\frac{y - \bx^\top\bbeta_1}{\sigma}}}{p(\bz^\top\btheta) \cdot \phi\paren{\frac{y - \bx^\top\bbeta_1}{\sigma}} 
				+ \paren{1-p(\bz^\top\btheta)} \cdot \phi\paren{\frac{y - \bx^\top\bbeta_2}{\sigma}}},
		\end{equation}
		with $\bgamma=(\btheta,\bbeta_1,\bbeta_2)$,  $p(x) = 1/\paren{1 + \exp(-x)}$, and $\phi(x)=\frac{1}{\sqrt{2\pi}}e^{-x^2/2}$. 
		
		Update each elements of $\bgamma^{(\tau, t)}$ by
		\begin{equation} \label{eqn:coeff-updates}
			\begin{aligned}
				\bbeta_{1}^{(\tau, t)} & := \underset{\bbeta_1}{\arg\!\min} \;  -Q_{n_{\tau}1}\paren{\bbeta_1 \cond \bgamma^{(\tau, t-1)}}  
				+ \lambda_{n_{\tau}}^{(t)} \norm{\bbeta_1}_1,  \\
				\bbeta_2^{(\tau, t)} & := \underset{\bbeta_2}{\arg\!\min} \; -Q_{n_{\tau}2}\paren{\bbeta_2 \cond \bgamma^{(\tau, t-1)}}  
				+ \lambda_{n_{\tau}}^{(t)} \norm{\bbeta_2}_1,   \\
				\btheta^{(\tau, t)} & := \underset{\btheta}{\arg\!\min} \; -Q_{n_{\tau}3}\paren{\btheta \cond \bgamma^{(\tau, t-1)}} 
				+ \lambda_{n_{\tau}}^{(t)} \norm{\btheta}_1,
			\end{aligned}
		\end{equation}
		where $Q_{1n_{\tau}}$, $Q_{2n_{\tau}}$, and $Q_{3n_{\tau}}$ are defined in \eqref{eqn:Qn123} with the $t$-th subset $\cI_{\tau-1}^{(t)}$.
		
	}
	Assign $\hat{\bgamma}^{(\tau)} = \bgamma^{(\tau, t_{\tau, \max})}$. 
\end{algorithm}

Algorithm \ref{alg:em} addresses these challenges through an EM algorithm, handling both the non-convexity of $\ell_{N_{\tau-1}}(\bgamma; y_i,\bx_i,\bz_i)$ and the high-dimensionality of the parameter space.
The EM algorithm is essentially an alternating maximization method, iterating between identifying the latent group membership $\{g_i\}$ and estimating the unknown parameter $\bgamma=(\btheta,\bbeta_1,\bbeta_2)$.
To ensure independence between samples across iterations, we first partition the entire sample set into disjoint subsets. 
In the E-step of the $t$-th iteration during episode $\tau$, given the parameters $\bgamma^{(\tau, t-1)}=(\btheta^{(\tau, t-1)},\bbeta_1^{(\tau, t-1)},\bbeta_2^{(\tau, t-1)})$ estimated from the previous iteration, the conditional probability of the $i$-th sample belonging to group 1 given the observed data is 
\begin{equation}
	\omega_i^{(\tau, t)} = \omega_i(\bgamma^{(\tau, t-1)})
	= \PP(g_i=1\cond y_i, \bx_i, \bz_i;\bgamma^{(\tau, t-1)})  =  \omega(y_i, \bx_i,  \bz_i; \bgamma^{(\tau, t-1)}),    
\end{equation}
where $\omega(y, \bx, \bz; \bgamma)$ is defined in \eqref{eqn:omega-main}. 
Let $\ell\paren{y_i, \bx_i,  \bz_i, g_i; \bgamma}$ be the log-likelihood of complete data where $g_i$ is also observable. 
Thus, conditional on the current estimate $\bgamma^{(\tau, t-1)}$, the conditional log-likelihood function can be calculated as
\begin{equation}
	\begin{aligned}
		Q_{n_{\tau}}(\bgamma \cond \bgamma^{(\tau, t-1)}) & := \; 
		\sum_{i \in \cI_{\tau}^{(t)}}\EE\brackets{ \ell\paren{y_i, \bx_i, \bz_i, g_i; \bgamma} \cond y_i, \bx_i, \bz_i; \bgamma^{(\tau, t-1)}} \\
		& = \; Q_{n_{\tau}1}(\bbeta_1 \cond \bgamma^{(\tau, t-1)}) 
		+ Q_{n_{\tau}2}(\bbeta_2 \cond \bgamma^{(\tau, t-1)})
		+ Q_{n_{\tau}3}(\btheta \cond \bgamma^{(\tau, t-1)}),
	\end{aligned}
\end{equation}
where
\begin{equation} \label{eqn:Qn123}
	\begin{aligned}
		Q_{n_{\tau}1}(\bbeta_1 \cond \bgamma^{(\tau, t-1)}) & := - \frac{1}{2n_{\tau}} \sum_{i \in \cI_{\tau}^{(t)}} \omega_i^{(\tau, t)}\cdot  \frac{(y_i - \bx_i^\top\bbeta_1)^2 }{\sigma^2}, \\
		Q_{n_{\tau}2}(\bbeta_2 \cond \bgamma^{(\tau, t-1)}) & := - \frac{1}{2n_{\tau}} \sum_{i \in \cI_{\tau}^{(t)}} (1- \omega_i^{(\tau, t)}) \cdot \frac{(y_i - \bx_i^\top\bbeta_2)^2 }{\sigma^2}, \quad\text{and}\quad \\
		Q_{n_{\tau}3}(\btheta \cond \bgamma^{(\tau, t-1)}) & :=  \frac{1}{n_{\tau}} \sum_{i \in \cI_{\tau}^{(t)}} \paren{ \omega_i^{(\tau, t)} \cdot \log p(\bz_i^\top\btheta) + (1- \omega_i^{(\tau, t)}) \cdot \log(1-p(\bz_i^\top\btheta)) }.
	\end{aligned}
\end{equation}
The M-step proceeds by maximizing $Q_{n_{\tau}}(\bgamma \cond \bgamma^{(\tau, t-1)})$, which is equivalent to maximizing $Q_{n_{\tau}1}(\bbeta_1 \cond \bgamma^{(\tau,t-1)})$, $Q_{n_{\tau}2}(\bbeta_2 \cond \bgamma^{(\tau,t-1)})$, and $Q_{n_{\tau}3}(\btheta \cond \bgamma^{(\tau,t-1)})$ simultaneously. 
However, in the high-dimensional setting, direct maximization of these objectives tends to overfit the data. To address this challenge, our algorithm incorporates regularization terms $\norm{\bbeta_1}_1$, $\norm{\bbeta_2}_1$, and $\norm{\btheta}_1$ to induce sparsity in the parameter estimates.
The specific updates of the model parameters are presented in \eqref{eqn:coeff-updates}. 
The algorithm then proceeds iteratively, alternating between the E-step and M-step until convergence.

\section{Theoretical Analysis}\label{sec:theory}

 
In this section, we establish theoretical guarantees for the estimation error of  $\big(\hbtheta^{(\tau)}, \hbbeta_1^{(\tau)}, \hbbeta_2^{(\tau)}\big)$ learned in the $\tau$-th episode using $N_{\tau-1}=n_02^{\tau-1}$ samples collected from the previous episode.
We begin by introducing the parameter space that characterizes the sparsity of the true parameters $( \btheta^*, \bbeta_1^*, \bbeta_2^*)$:
\begin{equation}
\btheta^*, \bbeta_1^*, \bbeta_2^* \in \Theta(d, s) = \braces{ \btheta, \bbeta_1, \bbeta_2 \in \RR^d:\norm{\btheta}_0 \le s,
\norm{\bbeta_1}_0 \le s, 
\norm{\bbeta_2}_0 \le s
},
\end{equation}
where $\btheta^*$ and $\bbeta^*$ are assumed to share the same dimension $d$ and sparsity level $s$ without loss of generality. 
Our analysis relies on the following regularity conditions:
\begin{enumerate}[label=(A\arabic*)]
\item Assume that there exists $0< \xi < 1/2$ such that $\xi < p(\bz_i^\top\btheta^*) < 1 - \xi$ or equivalently, there exists $C_{\xi} > 0$ such that $\abs{\bz_i^{\top}\btheta^*} < C_{\xi}$ for all $i$. \label{A1}
\item Assume that the initial estimators in the first episode satisfy
$\big\|\bbeta_1^{(1, 0)}-\bbeta_1^*\big\|_2 + \big\|\bbeta_2^{(1, 0)}-\bbeta_2^*\big\|_2+ \norm{\btheta^{(1, 0)}-\btheta^*}_2 \le \delta_{1, 0}$ 
for a sufficiently small constant $\delta_{1, 0}$. \label{A2}
\item\label{A3} Define the signal strength $\Delta^* \defeq \norm{\bbeta_1^*-\bbeta_2^*}_2$. Assume that the signal-to-noise ratio (SNR), defined as $\Delta^*/ \sigma$, satisfies $\Delta^*/ \sigma \ge C_{\mathrm{SNR}}$ for a sufficiently large constant $C_{\mathrm{SNR}}$, where $\sigma$ is the standard deviation of the noise $\epsilon_i$. 
\item For each $k\in [K]$, assume that the i.i.d. covariates $(\bx_{i, k}, \bz_i)$ are sub-Gaussian. Moreover, let $\bSigma_{x, k}\defeq\EE[\bx_{i,k}\bx_{i, k}^{\top}]$ and $\bSigma_z\defeq\EE[\bz_{i}\bz_i^{\top}]$, and assume that there exists a costant $M>1$ such that  $1/ M < \lambda_{\min}(\bSigma_{x, k}) \le \lambda_{\max}(\bSigma_{x, k}) < M$ for all $k \in [K]$ and $1/ M < \lambda_{\min}(\bSigma_z) \le \lambda_{\max}(\bSigma_z) < M$. \label{A4}
\end{enumerate}

Assumption \ref{A1} requires non-degenerate group assignment probabilities, preventing the variance of the group indicator $g_i$ from vanishing, as $\Var(g_i) = p_i(1-p_i)$ is bounded away from zero. This condition is standard in high-dimensional logistic regression literature \citep{van2014asymptotically,abramovich2018high,athey2018approximate}.
Assumption \ref{A2} assumes the initial estimators to be within a neighborhood of the true parameters, which can be achieved through any consistent initialization procedure with sufficient samples in episode 0. 
We provide a detailed initialization algorithm in Remark \ref{rmk: Initialization}.
Assumption \ref{A3} characterizes the minimal signal strength required for effective group separation, which is necessary for distinguishing different groups in mixture models \citep{loffler2021optimality, ndaoud2022sharp}. For clear presentation, we defer the explicit specification of constants $\delta_{1, 0}$ and $C_{\mathrm{SNR}}$  to Section  \ref{sec:proof-coeff-bound}  of the supplementary material.
Assumption \ref{A4}  imposes standard regularity conditions that ensure the non-singularity and upper-boundedness of the population covariance matrices.

\begin{remark}[Initialization]\label{rmk: Initialization}
	The initial estimators $\left(\btheta^{(1, 0)}, \bbeta^{(1, 0)}_{1}, \bbeta^{(1, 0)}_{2}\right)$ in the first episode can be obtained through any algorithm that generates consistent estimators for $(\btheta^*, \bbeta^*_{1}, \bbeta^*_{2})$ that satisfy Assumption \ref{A2}.  We recommend the following approach, which proceeds in three stages. First, using all the samples $\{\by_i, \bx_{i}\}_{i \in \calI^{(0)}}$ collected in episode 0, we fit a single LASSO model to select the features in $\bx$ with non-zero coefficients.  Second, we apply Gaussian Mixture clustering to $\{y_i, \bx_{i}'\}_{i\in \calI^{(0)}}$, where $\bx_{i}'$ denotes the selected features observation $i$, to cluster the initial data into two groups. Third, we use these group labels to fit a logistic regression model, obtaining the estimator $\btheta^{(1, 0)}$, and separately fit two LASSO models within the two groups to obtain  $\bbeta^{(1, 0)}_{1}$ and $\bbeta^{(1, 0)}_{2}$. Prior research has demonstrated that Gaussian Mixture clustering can achieve high accuracy of the two groups when the signal-to-noise ratio is sufficiently large  \citep{loffler2021optimality, ndaoud2022sharp}, which, combined with \ref{A3} and sufficiently large $n_0$, ensures that our initial estimators satisfy \ref{A2}. 
\end{remark}

 \subsection{Learning Performance} \label{sec:theory-learning}

Building upon the above assumptions, we present our theoretical results, which characterize the estimation error of the learned parameters.

\begin{theorem}
	 \label{thm:coeff-bound-hd}
Suppose Assumptions \ref{A1}--\ref{A4} hold and $s^2\log d\log n_0 \lesssim n_0$. Let the initial estimators $\bgamma^{(\tau, 0)}=\hat\bgamma^{(\tau-1)}$ for $\tau \geq 2$. Furthermore, select $t_{\tau, \max}\asymp 
\log n_0$  for $\tau=1$ and $t_{\tau, \max}\asymp 
1$ for $\tau\geq 2$. By choosing appropriate regularization parameters $\big\{\lambda_{n_{\tau}}^{(t)}\big\}$, we have
\begin{equation}\label{eq:est-bound-l2}
\norm{\hat{\bbeta}_{1}^{(\tau)} - \bbeta_{1}^*}_2 + \norm{\hat{\bbeta}_{2}^{(\tau)} - \bbeta_{2}^*}_2 + \norm{\hat{\btheta}^{(\tau)} - \btheta^*}_2
\lesssim 
\sqrt{\frac{s\log d\log n_0}{N_{\tau-1}}},
\end{equation}
and
\begin{equation}\label{eq:est-bound-l1}
	\norm{\hat{\bbeta}_{1}^{(\tau)} - \bbeta_{1}^*}_1 + \norm{\hat{\bbeta}_{2}^{(\tau)} - \bbeta_{2}^*}_1 + \norm{\hat{\btheta}^{(\tau)} - \btheta^*}_1
	\lesssim 
	\sqrt{\frac{s^2\log d\log n_0}{N_{\tau-1}}},
\end{equation}
with probability at least $1-d^{-1}$, where $\big(\hat{\btheta}^{(\tau)},\hat{\bbeta}_{1}^{(\tau)},\hat{\bbeta}_{2}^{(\tau)}\big)$ are obtained from Algorithm \ref{alg:em} in the $\tau$-th episode.
\end{theorem}

Theorem \ref{thm:coeff-bound-hd} establishes the statistical convergence rates for both $\ell_2$ and $\ell_1$ estimation errors. The $\ell_2$ error bound scales as  $\bigO{\sqrt{s\log d \log n_0 / N_{\tau-1}}}$, while the $\ell_1$ estimation error scales as $\bigO{\sqrt{s^2\log d \log n_0 / N_{\tau-1}}}$, matching the minimax optimal rates for high-dimensional sparse estimation up to a logarithm factor in $n_0$. We note that Theorem \ref{thm:coeff-bound-hd} serves as a simplified version of our theoretical results, a complete version of which is provided as Theorem \ref{thm:1-detailed} in Section \ref{sec:proof-coeff-bound} of the supplemental materials, where we explicitly specify the choice of the regularization parameters $\big\{\lambda_{n_{\tau}}^{(t)}\big\}$, the concrete requirements for the constants $\delta_{1, 0}$ and $C_{\mathrm{SNR}}$, and more accurate probability bounds.

To better understand the convergence behavior, we provide a more detailed analysis in Remark \ref{rmk:logn0}, which elucidates the role of $t_{\tau,\max}$ and explains the origin of the additional $\sqrt{\log n_0}$ factor in the error bounds. This analysis reveals how the algorithm's phased learning structure and the initial estimation error influence the final convergence rates.

\begin{remark}\label{rmk:logn0}
	Our theoretical analysis in Section \ref{sec:proof-coeff-bound} of the supplemental material establishes a finer error bound for the estimators $\big(\hbtheta^{(\tau)}, \hbbeta_1^{(\tau)},  \hbbeta_2^{(\tau)}\big)$:
	\begin{equation}\label{eq:est-bound-t}
		\norm{\hbbeta_1^{(\tau)} - \bbeta_1^*}_2 + \norm{\hbbeta_2^{(\tau)} - \bbeta_2^*}_2 + \norm{\hbtheta^{(\tau)} - \btheta^*}_2 
		\lesssim
		\rho^{t_{\tau, \max}}\delta_{\tau, 0}
		+ \sqrt{\frac{st_{\tau, \max}\log d}{N_{\tau-1}}},
	\end{equation}
	where $\rho<1$ is a contraction factor, and $\delta_{\tau, 0}:=
	\big\|\bbeta_1^{(\tau, 0)} - \bbeta_1^*\big\|_2 + \big\|\bbeta_2^{(\tau, 0)} - \bbeta_2^*\big\|_2 + \norm{\btheta^{(\tau, 0)} - \btheta^*}_2$. The first term in \eqref{eq:est-bound-t} quantifies the contraction of initial estimation error, while the second term represents the statistical error rate. For $\tau=1$, since Assumption \ref{A2} assumes the initial estimation error $\delta_{1, 0}$ to be a constant, we require $t_{1, \max} \asymp \log n_0$ iterations to ensure the first term $\rho^{t_{\tau, \max}}\delta_{1, 0}$ is dominated by the second term, yielding a rate of  $\sqrt{\frac{s\log n_0 \log d}{N_{\tau-1}}}$. For $\tau \geq 2$, since the initial estimation error $\delta_{\tau, 0} \asymp  \sqrt{\frac{s\log n_0 \log d}{N_{\tau-2}}}$, a constant number of iterations suffices. If we strengthen \ref{A2} to require that $\delta_{1, 0}=\bigO{a_{n_0}}$ for some $a_{n_0}=o(1)$, then the extra $\sqrt{\log n_0}$ factor in \eqref{eq:est-bound-l2} and \eqref{eq:est-bound-l1} can be reduced to $\sqrt{\log (n_0a_{n_0}^2)}$.
	\end{remark}


\subsection{Classification Accuracy}\label{sec:theory-classification}

In this section, we provide a theoretical guarantee for the statistical accuracy of the latent group identification procedure Algorithm \ref{alg:em-bandit}. 
In the $\tau$-th episode, after obtaining $\big(\hat\btheta^{(\tau)}, \hat\bbeta_1^{(\tau)}, \hat\bbeta_2^{(\tau)}\big)$, Algorithm \ref{alg:em-bandit} employs a Bayes classifier $G\big(\bz_i; \hat\btheta^{(\tau)}\big)$, defined as  
\begin{align*} \label{eqn:bayes-classifier}
G(\bz_i; \btheta) = 
\begin{cases}
1, & p(\bz_i^\top\btheta)\ge 1/2, \\
2, & p(\bz_i^\top\btheta) < 1/2.
\end{cases}
\end{align*}

To characterize the classification performance, we define the optimal misclassification error achievable with the true parameter,
$
R(\btheta^*) := \EE\brackets{\bbone(G(\bz_i;\btheta^*) \ne g_i)},
$
and the misclassification error of the estimated classifier,
$
R\big(\hat\btheta\big) := \EE\brackets{\bbone(G(\bz_i;\hat\btheta) \ne g_i) \mid \hbtheta}.
$

\begin{theorem} \label{thm:miss-class-rate}
Let $\hat{\btheta}^{(\tau)}$ be the estimator obtained in the $\tau$-th episode for $\tau \geq 1$. Under the assumptions of Theorem \ref{thm:coeff-bound-hd}, we have that the excess misclassification error  satisfies
\begin{align*}
R\big(\hat\btheta^{(\tau)}\big)-R(\btheta^*) \lesssim \sqrt{\frac{s\log d\log n_0}{N_{\tau-1}}},
\end{align*}
with probability at least $1-d^{-1}$.
\end{theorem}

Theorem \ref{thm:miss-class-rate} reveals that the excess misclassification error converges at the same rate as the parameter estimation error bound established in Theorem \ref{thm:coeff-bound-hd}, which is useful for our subsequent regret analysis in Section \ref{sec:theory-regret}. 

\subsection{Regret Analysis}
\label{sec:theory-regret}

Recall that, for any policy $\hat{\pi}$, we define two types of regrets in Section \ref{sec:two-type-regret}: (1) the strong regret, $\reg^{*}_i = \underset{a\in [K]}{\max}\; \angles{\bx_{i,a}, \bbeta^*_{g_i}} - \angles{\bx_{i,\hat a_i}, \bbeta^*_{g_i}}$, where $\hat{a}_i$ is the action chosen by policy $\hat{\pi}$ for customer $i$, and (2) the regular regret, $\tilde{\reg}_i= \angles{\bx_{i,\tilde{a}_i}, \bbeta^*_{g_i}} - \angles{\bx_{i,\hat a_i}, \bbeta^*_{g_i}}$, where $\tilde a_i = \underset{a\in [K]}{\arg \max} \angles{\bx_{i, a}, \bbeta_{\tilde g_i}}$ and $\tilde g_i = G(\bz_i; \btheta^*)$. These two regret formulations arise from comparing $\hat{\pi}$ against two types of oracles: the strong oracle, which has access to the value of $(\bbeta_1^*, \bbeta_2^*)$ and the latent group labels $g_i$, and the regular oracle, which knows $(\btheta^*, \bbeta_1^*, \bbeta_2^*)$.  Notably, the regular oracle is weaker than the strong oracle since even with the known $\btheta^*$, there remains an inherent positive misclassification error $R(\btheta^*)$ due to the probabilistic nature of the subgroup model in \eqref{eqn:lhlb}. Consequently, the strong regret necessarily exceeds the regular regret.

The cumulative strong and  regular regrets over time horizon$T$ are defined, respectively, as 
\begin{equation*}
\Reg^{*} (T)= \sum_{\tau=0}^{\tau_{\max}}\sum_{i\in\calI_{\tau}}\EE[\reg^{*}_i], \quad\text{and}\quad 
\tilde{\Reg} (T)= \sum_{\tau=0}^{\tau_{\max}}\sum_{i\in\calI_{\tau}}\EE\left[\tilde{\reg}_i\right],
\end{equation*}
where $\tau_{\max}:=[\log_2(T/n_0+1)]-1$ is the maximum number of episodes within horizon length $T$. To establish theoretical bounds for $\Reg^{*} (T)$ and $\tilde{\Reg} (T)$, we require the following additional conditions:
\begin{enumerate}[label=(B\arabic*)]
\item  Assume that $\norm{\bx_{i, k}}_{\infty} \leq \overline{x}$ for some constant $\overline{x}>0$ and the coefficient $\big\|\bbeta_g^*\big\|_1 \leq \overline{L}$ for some constant $\overline{L}>0$ for $g=1,2$. Consequently, the reward function is bounded: $\abs{\angles{\bx_{i, k}, \bbeta_g^*}} \leq \overline{R}=\overline{x}\overline{L}$ for all $i$ and $g=1, 2$. \label{B1}
\item  Assume that there exist positive constants $C$ and $\overline{h}$ such that, for all $h\in\big[0, \overline{h}\big]$, we have $\PP\paren{\angles{\bx_{i, \tilde{a}_{i, g}}, \bbeta_g^*} \le \max\limits_{k\ne \tilde{a}_{i,g}}\angles{\bx_{i, k}, \bbeta^*_{g}} + h} \le Ch $ for $g=1,2$, where $\tilde{a}_{i, g}:=\underset{k\in [K]}{\arg \max} \angles{\bx_{i, k}, \bbeta_g^*}$. \label{B2}
\end{enumerate}

Assumptions \ref{B1} and \ref{B2} are both standard conditions that have been widely assumed in the linear bandit literature \citep{goldenshluger2013linear, bastani2020online, bastani2020mostly, li2021regret, wang2024online}. Specifically, Assumption \ref{B1} ensures that the maximum regret at each time is upper bounded.  Assumption \ref{B2} plays an important role in controlling the distribution of the covariates in the neighborhood of decision boundaries, where small perturbations in parameter estimates can lead to changes in the selected actions. This assumption is satisfied under relatively mild conditions, for example, the density of $\angles{x_{i, k}, \bbeta_g^*}$ is uniformly bounded for all $k \in [K]$. Remark \ref{rmk:margin} further explains the necessity of Assumption \ref{B2} in the context of latent heterogeneity.

We then establish theoretical upper bounds for both the strong and regular regrets of Algorithm \ref{alg:em-bandit}.

\begin{theorem} \label{thm:regret}
Assume the assumptions in Theorem \ref{thm:coeff-bound-hd} and Assumptions \ref{B1} and \ref{B2} hold. We have the cumulative strong regret 
\begin{equation}\label{eq:upper-strong}
\Reg^*(T) \lesssim  \overline{R}n_0 +\overline{x}^2s^2\log d \log n_0\log T+ \overline{x} \norm{\bbeta_2^* - \bbeta_1^*}_1R(\btheta^*) T.
\end{equation}
and the cumulative regular regret 
\begin{equation}\label{eq:upper-regular}
\tilde{\Reg}(T) \lesssim \overline{R}n_0+ \overline{x}^2 \norm{\bbeta_2^*-\bbeta_1^*}_1\sqrt{s^2\log d \log n_0 } \sqrt{T}.
\end{equation}
\end{theorem}

The upper bound for strong regret comprises three terms. The first term captures the cost of initial exploration in episode 0, and the second term stems from the small failure probabilities established in Theorems \ref{thm:coeff-bound-hd} and \ref{thm:miss-class-rate}. The third term, which dominates the other two terms for large $T$, reflects the impact of misclassification. As defined in Section \ref{sec:theory-classification}, the classifier $G\big(\bz_i; \btheta^{(\tau)}\big)$ misclassifies a customer $i$ in episode $\tau$ with probability $R\big(\btheta^{(\tau)}\big)$, which converges to a constant $R(\btheta^*)$  as characterized by Theorem \ref{thm:miss-class-rate}. Under such misclassification, the instant strong regret $\reg^{*}_i$ attains a constant level $\overline{x} \norm{\bbeta_2^* - \bbeta_1^*}_1$, leading to the linear term $\overline{x} \norm{\bbeta_2^* - \bbeta_1^*}_1R(\btheta^*) T$.

However, this linear term vanishes in the regular regret bound. This occurs because the regular regret compares our algorithm against oracle actions selected using the true value of $\btheta^*$, which itself incurs a constant $R(\btheta^*)$ misclassification rate. The two constants cancel each other, leaving only an $O(\sqrt{T\log d})$ term that emerges from the failure of the event defined in Assumption \ref{B2}.

\begin{remark}\label{rmk:margin}
	Assumption \ref{B2} is essential for establishing bound \eqref{eq:upper-regular} for the regular regret under latent heterogeneity. Violation of this assumption makes it fundamentally difficult to distinguish the optimal arm from the sub-optimal arms, resulting in sub-optimal selection, i.e.,  $\hat{a}_{i}=\argmax_k \angles{\bx_{i, k}, \hbbeta_{g}} \neq \tilde{a}_{i, g}$. This issue remains manageable for regret analysis within a single group since $\angles{\bx_{i, \tilde{a}_{i,g}}, \bbeta^*_{g}}-\angles{\bx_{i, \hat{a}_{i}}, \bbeta^*_{g}}=\max_k \angles{\bx_{i, k}, \bbeta^*_{g}}-\max_k \angles{\bx_{i, k}, \hbbeta_{g}}+\angles{\bx_{i, \hat{a}_i}, \hbbeta_{g}}-\angles{\bx_{i, \hat{a}_{i}}, \bbeta^*_{g}} \leq 2\sup \abs{\angles{\bx_{i, k},\hbbeta_{g}- \bbeta^*_{g}}}$ is still bounded by the $\ell_1$ estimation error $\big\|\hbbeta_{g}-\bbeta^*_{g}\big\|_1$. However, with latent heterogeneity, the consequence of sub-optimal selection becomes more severe. Consider a customer $i$ with $p(\bz_i^{\top}\btheta^*)<1/2$, who can be assigned to group $g_i=1$ with a positive probability. In such case, if the estimated action $\hat{a}_{i}=\argmax_k \angles{\bx_{i, k}, \hbbeta_{2}} \neq \tilde{a}_{i, 2}$, the regular regret $\angles{\bx_{i, \tilde{a}_{i,2}}, \bbeta^*_{1}}-\angles{\bx_{i, \hat{a}_{i}}, \bbeta^*_{1}}$ can attain a constant level, regardless of estimation accuracy. 
\end{remark}

We further develop minimax lower bounds for the strong and regular regrets.  For any policy $\hat\pi$, define \[\EE_{\hat\pi}[\reg^*_{i}]:=\EE\left[\underset{k\in [K]}{\max}\; \angles{\bx_{i,k}, \bbeta^*_{g_i}}\right] -\EE_{\hat\pi}\left[\angles{\bx_{i,\hat a_i}, \bbeta^*_{g_i}}\right], \quad \EE_{\hat\pi}[\widetilde{\reg}_{i}]:=\EE\left[\angles{\bx_{i,\widetilde{a}_i}, \bbeta^*_{g_i}}\right] -\EE_{\hat\pi}\left[\angles{\bx_{i,\hat a_i}, \bbeta^*_{g_i}}\right], \]
where $\EE_{\hat\pi}$ represents that the actions $\hat{a}_i$ are chosen according to the policy $\hat\pi$. We then establish the following lower-bound result.
\begin{theorem}\label{thm:lower-bound} Let $\mu(y, \bx, \bz; \bgamma^*)$ denote a distribution of $(y_i, \bx_i, \bz_i)$ that satisfies model \eqref{eqn:lhlb}  with $\bgamma^*=(\btheta^*, \bbeta_1^*, \bbeta_2^*)$, and define $\calP_{d, s,\overline{x},\overline{L}}$ as the collection of all the distributions $\mu(y, \bx, \bz; \bgamma^*)$ such that $\bgamma^* \in \Theta(d, s)$ and Assumptions \ref{A1}, \ref{A4}, \ref{B1}, and \ref{B2} hold with constants $\overline{x}$ and $\overline{L}$. Then we have
	\begin{equation}\label{eq:lower-strong}
			\inf_{\hat\pi}\sup_{\mu \in \calP_{d, s, \overline{x}, \overline{L}}}\sum_{i=1}^{T}\EE_{\hat\pi}[\reg^*_{i}]\gtrsim \overline{x}\overline{L}R(\btheta^*)T,  \quad	\inf_{\hat\pi}\sup_{\mu \in \calP_{d,s,\overline{x},\overline{L}}} \sum_{i=1}^{T}\EE_{\hat\pi}[\widetilde{\reg}_{i}]\ \gtrsim \overline{x}\overline{L}\sqrt{s\log d}\sqrt{T},
	\end{equation}
	where the infimum is taken over all the possible policies $\hat{\pi}$.
\end{theorem}

Compared to the upper bounds established in Theorem \ref{thm:regret}, the lower bounds replace $\norm{\bbeta_1^{*}-\bbeta_2^{*}}_1$ by $\overline{L}$, which are equivalent since $\sup\norm{\bbeta_1^{*}-\bbeta_2^{*}}_1 = 2\overline{L}$ under Assumption \ref{B1}. The lower bound $\Omega(R(\btheta^*)T)$ for the strong regret precisely matches the dominant linear term in the upper bound \eqref{eq:upper-strong}, while the lower bound $\Omega(\sqrt{Ts\log d})$ aligns with the upper bound \eqref{eq:upper-regular} in terms of $T$ and $d$, demonstrating the minimax optimality of our proposed method under latent heterogeneity. The gap between the upper and lower bounds for the regular regret amounts to a factor of $\sqrt{s\log n_0}$, which, as discussed in Remark \ref{rmk:logn0}, stems from the sample splitting procedure and can be reduced by assuming stronger conditions on the initialization.

\section{Numerical Study} \label{sec:numerical}

In this section, we evaluate the performance of our proposed heterogeneous algorithm through numerical studies. Section \ref{sec:simul} presents simulation studies to show the effectiveness of our algorithm and validate our theoretical findings under various settings. Section \ref{sec:real} further illustrates the practical utility of our method through an application to a cash bonus dataset from a mobile commerce platform.

\subsection{Simulations} \label{sec:simul}

In this section, we conduct numerical simulations to demonstrate the performance of our proposed algorithm under varying conditions. Our simulation is implemented based on the model: $y_{i, k}= \angles{\bx_{i, k}, \bbeta_{g_i}^*}+\epsilon_i$ for $\bx_{i, k} \in \RR^d$, $k \in [K]$ and $g_i \in \{1, 2\}$, with $\epsilon_i \sim \calN(0, \sigma^2)$.  We let the number of available actions $K=2$ and generate the covariates $\bx_{i, k} \sim \calN(\bmu_{k}, \bSigma_k)$ ($k = 1,2$). Each entry of $\bmu_{1}$ and $\bmu_{2}$ is independently generated from $\calN(1, (0.5)^2)$ and $\calN(-1, (0.5)^2)$, respectively, and the covariance matrix $\bSigma_1=\bSigma_2$ is set to be an AR(1) matrix with correlation $0.5$, that is, $(\bSigma_1)_{j_1, j_2}=(\bSigma_2)_{j_1, j_2}=0.5^{\abs{j_1-j_2}}$. For the parameters $\bbeta_g^*$, we set $\bbeta_{1, j}^* = \overline{L}\bbone(1\leq j \leq s) / s$ and $\bbeta_{2, j}^* = -\overline{L}\bbone(d/2\leq j \leq d/2+s) / s $ for $j=1,\dots,d$, which ensures that $\norm{\bbeta_1^*}_1=\norm{\bbeta_2^*}_1=\overline{L}$. Specifically, we fix the noise level $\sigma=1$ and the sparsity $s=20$ and vary $\overline{L} \in \{2.5, 5\}$ and $d \in \{500, 1000\}$. The group assignment probability is determined by $\PP(g_i = 1 | \bz_i) = 1 / (1 + \exp(-z^{\top}_i \btheta^*))$, where $z_i \in \RR^{50}$ is generated from $N(0, I_{50})$, and $\btheta^*$ has nonzero entries in its first 10 dimensions, drawn uniformly from $[-1, 1]$. 

To assess the performance of our proposed algorithm (Algorithm \ref{alg:em-bandit}) and verify our theoretical results, we compare the average strong and regular regrets, i.e., $\frac{1}{T}\Reg^*(T)$ and $\frac{1}{T}\widetilde{\Reg}(T)$, against a benchmark---a single LASSO method applied without considering the latent group structure. In both algorithms, the regularization parameters $\lambda$ are chosen through cross-validation, and the maximum number of iterations $t_{\max}$ in Algorithm \ref{alg:em} is set as 1 for all episodes. Figure \ref{fig:regret} presents the average strong and regular regrets of our proposed algorithm (``hetero'') and the single LASSO (``single'') for $s=20$, $d\in\{500, 1000\}$, and $\overline{L} \in \{2.5, 5\}$. All results are averaged over 100 independent runs. 

\begin{figure}[!t]
	\centering
	\begin{subfigure}[b]{0.45\textwidth}
		\centering
		\includegraphics[width=\textwidth]{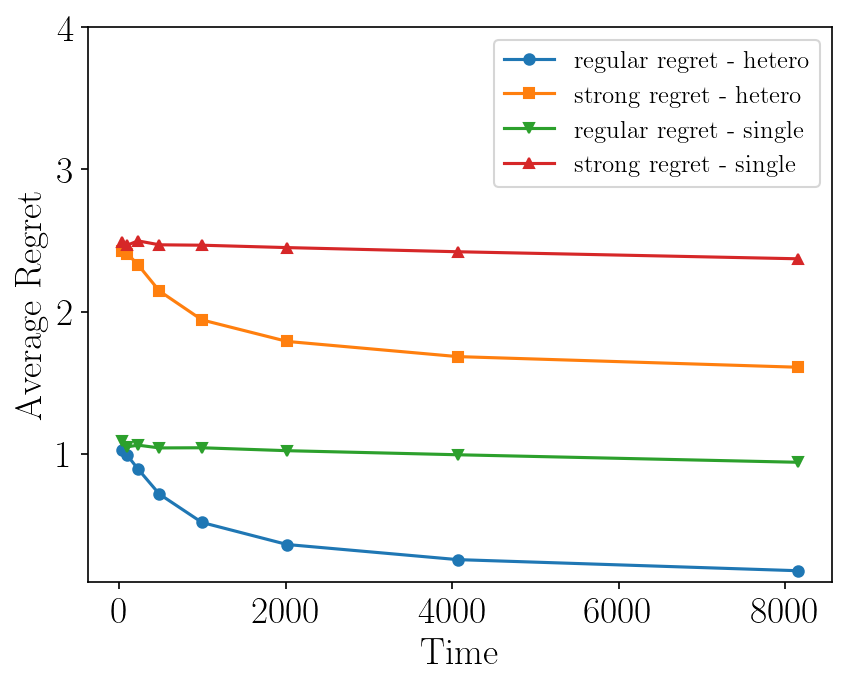}
		\caption{$d=500, \overline{L}=2.5$}
		\label{fig:regret_500_2.5}
	\end{subfigure}
	\hfill
	\begin{subfigure}[b]{0.45\textwidth}
		\centering
		\includegraphics[width=\textwidth]{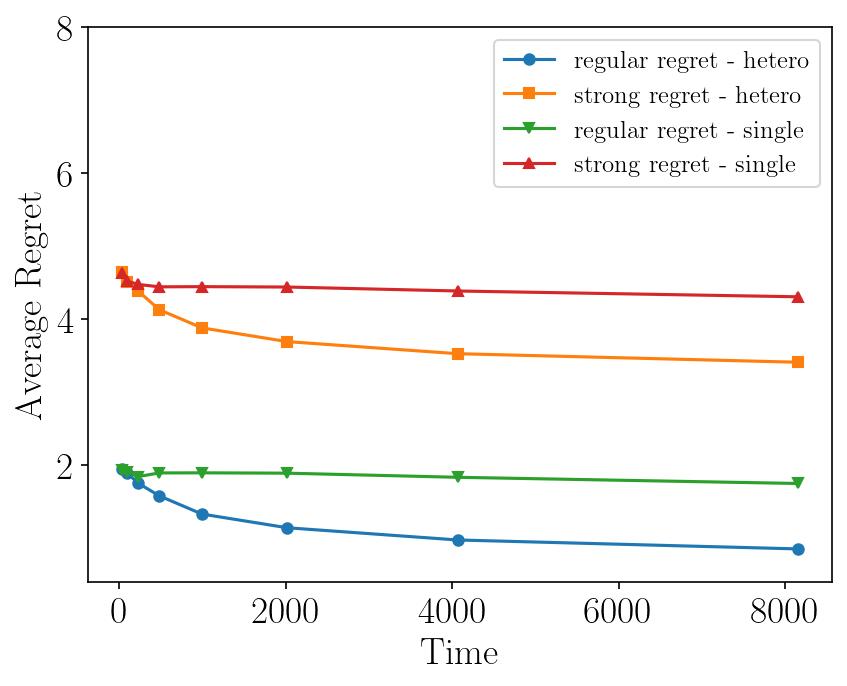}
		\caption{$d=500, \overline{L}=5$}
		\label{fig:regret_500_5}
	\end{subfigure}
	
	\begin{subfigure}[b]{0.45\textwidth}
		\includegraphics[width=\textwidth]{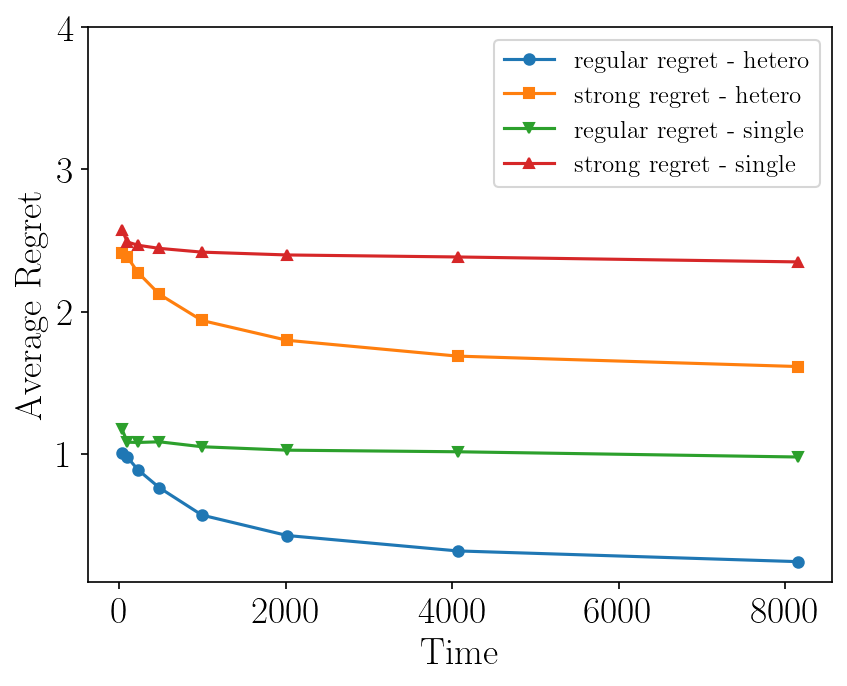}
		\caption{$d=1000, \overline{L}=2.5$}
		\label{fig:regret_1000_2.5}
	\end{subfigure}
	\hfill
	\begin{subfigure}[b]{0.45\textwidth}
		\centering
		\includegraphics[width=\textwidth]{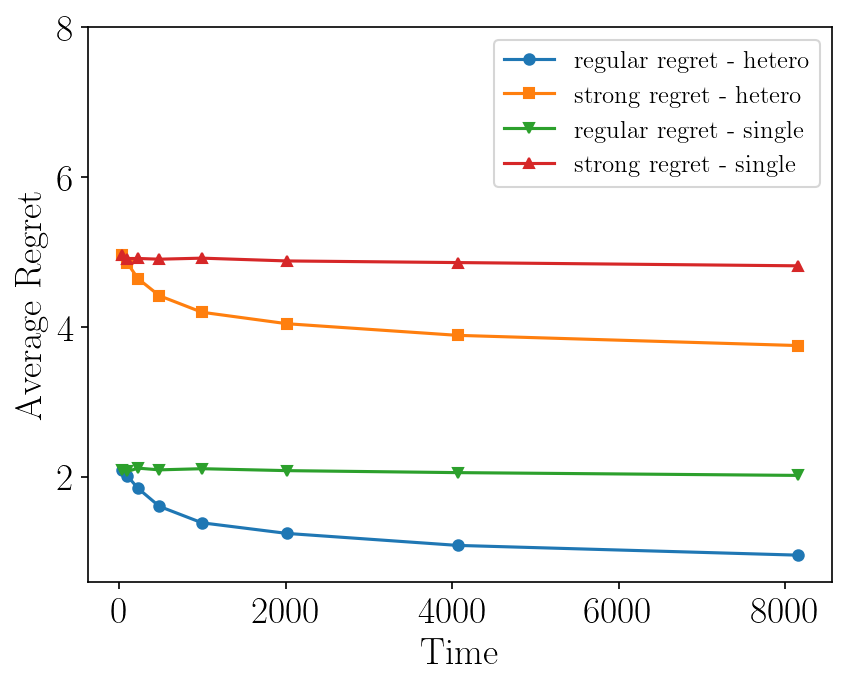}
		\caption{$d=1000, \overline{L}=5$}
		\label{fig:regret_1000_5}
	\end{subfigure}
	\caption{Average strong and regular regrets with $s=20$, $d=\{500, 1000\}$ and $ \overline{L}\in \{2.5, 5\}$. The horizontal axis ``time'' represents the sample size $T$. }
	\label{fig:regret}
\end{figure}

As shown in Figure \ref{fig:regret}, our proposed heterogeneous algorithm significantly outperforms the ``single'' algorithm, demonstrating the critical importance of identifying latent group heterogeneity. Under all scenarios, the average regular regret of our algorithm (``regular regret - hetero'') approaches zero as $T$ increases, while the average strong regret (``strong regret - hetero'') stabilizes at a constant level. This observation aligns closely with Theorem \ref{thm:regret}, where we theoretically establish that the cumulative strong regret grows linearly in $T$, and the cumulative regular regret exhibits sublinear growth.

Furthermore, when the parameter magnitude $\overline{L}$ increases from $2.5$ to $5$, the constant level of average strong regret approximately doubles. This behavior is consistent with the theoretical bound derived in equations \eqref{eq:upper-strong} since $\norm{\bbeta^*_1-\bbeta_2^*}_1 \propto \overline{L}$ in our experimental setup. Moreover, the dimensionality of the problem $d$ demonstrates minimal impact on the regret performance when varying from $500$ to $1000$, illustrating our algorithm's robustness and effectiveness in handling high-dimensional parameter spaces.

\begin{figure}[!t]
	\centering
	\begin{subfigure}[b]{0.45\textwidth}
		\centering
		\includegraphics[width=\textwidth]{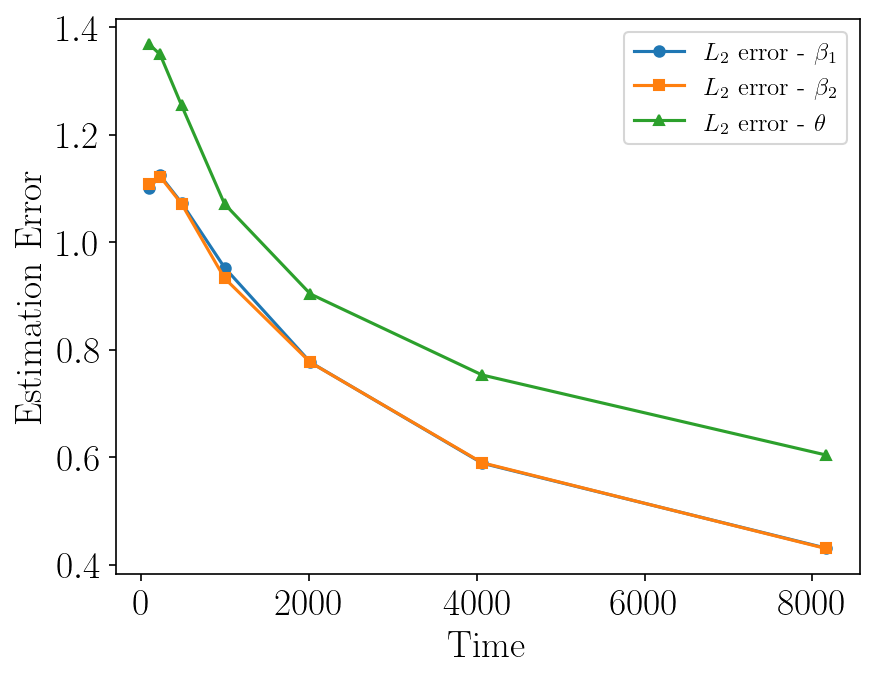}
		\caption{$d=500$}
		\label{fig:error_500_2.5}
	\end{subfigure}
	\hfill
	\begin{subfigure}[b]{0.45\textwidth}
		\centering
		\includegraphics[width=\textwidth]{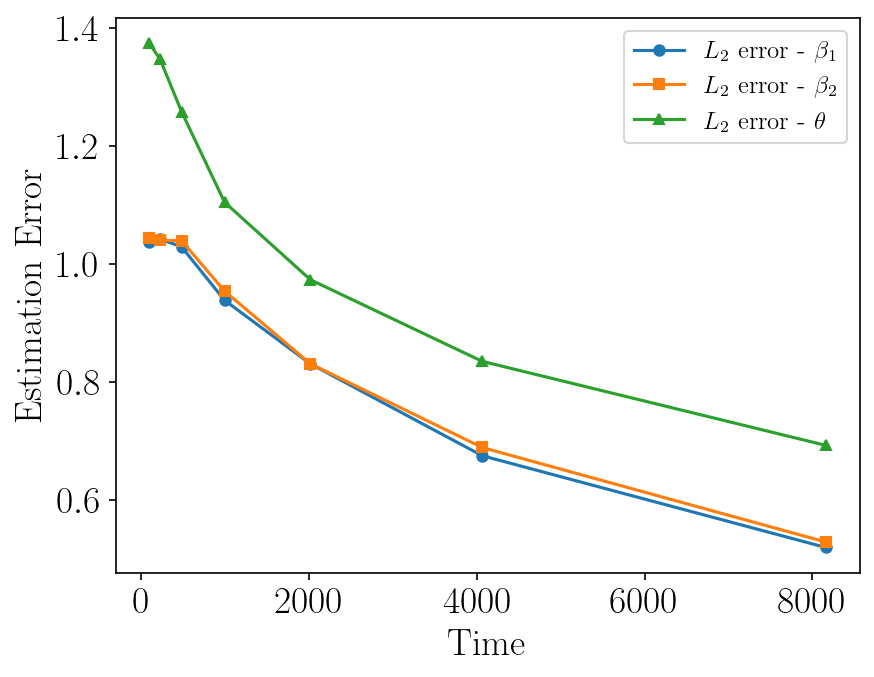}
		\caption{$d=1000$}
		\label{fig:error_1000_2.5}
	\end{subfigure}
	\caption{Estimation errors of the parameters $( \btheta^*,\bbeta_1^*, \bbeta^*_2)$ with $s=20$, $\overline{L}=2.5$, and $d\in\{500, 1000\}$. The horizontal axis ``time'' represents the sample size $T$.}
	\label{fig:error}
\end{figure}

Moreover, we present the $\ell_2$ estimation error for the parameters $(\btheta^*, \bbeta_1^*, \bbeta_2^*)$  in Figure \ref{fig:error}. Due to the potential for the algorithm to interchange the two groups, we compute the $\ell_2$ estimation errors of $\bbeta_g^*$ and $\btheta^*$ as the minimum error between the estimated group and the alternative group, that is, $\min\left\{\norm{\hat\bbeta_g-\bbeta_g^*}_2, \norm{\hat\bbeta_{\{1, 2\} \setminus g}-\bbeta_g^*}_2\right\}$ and $ \min\left\{\norm{\hat\btheta-\btheta^*}_2, \norm{\hat\btheta+\btheta^*}_2\right\}$,
respectively. As illustrated in Figure \ref{fig:error}, the estimation errors consistently decrease with increasing sample size $T$, indicating that our algorithm effectively estimates the model parameters in high-dimensional settings, which leads to the reduction in regret. The estimation errors for $d = 1000$ are only slightly larger than those for $d = 500$, which aligns with our theoretical result that the estimation error scales with $\sqrt{\log d}$.

\subsection{Real Data Analysis} \label{sec:real}


In this section, we illustrate the usefulness of our proposed method with application to the cash bonus dataset, originally presented by \cite{chen2022bcrlsp}. The dataset was collected from a mobile app, Taobao Special Offer Edition, which provided its users with a daily cash bonus that could be subtracted from the final payment at the time of purchase within 24 hours. The aim is to determine the optimal amount of cash bonus allocated to each user that leads to the highest payment.

Each observation in the dataset consists of customer features (user demographics and behavior information), a consecutive action variable $a_i$ (the amount of cash bonus) ranging from 0.25 to 1.95 with increment 0.01, and a continuous reward $y_i$ (the actual payment when the user redeems the cash bonus). To improve the computational efficiency, we compute the top 100 principal components of the customer features as the feature variable $\bz_{i}$. The contextual features $\bx_{i, a_{i}}$ are defined as $\bx_{i, a_{i}}=[a_i, a_i^2, \bz_i, a_i \bz_i, a_i^2\bz_i]$, incorporating quadratic terms of $a_i$ and their interaction terms with $\bz_i$. Furthermore, we divide the dataset into two groups based on the user's consumption level and remove this variable from the dataset to assume that it is unknown in practice. 
Following model \eqref{eqn:lhlb}, we use the entire dataset to fit an $\ell_1$-regularized logistic regression model for group classification and two LASSO models for the two groups separately. The fitted parameters are viewed as the ground truth $(\btheta^*, \bbeta_1^*, \bbeta_2^*)$, which are further used to generate the unobserved rewards $y_{i, k}$ in model \eqref{eqn:lhlb} for all of the actions $k$. 

Similar to the simulation studies, we evaluate the performance of our proposed heterogeneous method by comparing its average strong and regular regrets with a ``single'' method that applies a single LASSO in each episode without considering the heterogeneous grouping. Additionally, we also implement a ``separate'' method that applies LASSO algorithms separately for the two groups, assuming the group assignment is known. The ``separate'' method indicates the optimal performance one can expect in the heterogeneous setting with unknown groups. 

\begin{figure}[!t]
	\centering
	\begin{subfigure}[b]{0.45\textwidth}
		\centering
		\includegraphics[width=\textwidth]{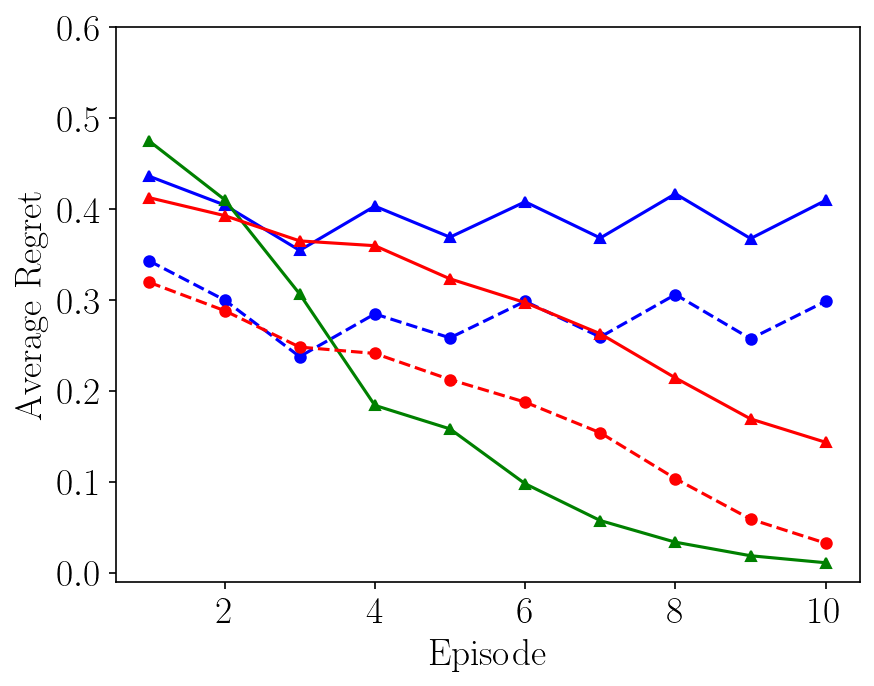}
		\caption{high-consumption group}
		\label{fig:real_g1}
	\end{subfigure}
	\hfill
	\begin{subfigure}[b]{0.45\textwidth}
		\centering
		\includegraphics[width=\textwidth]{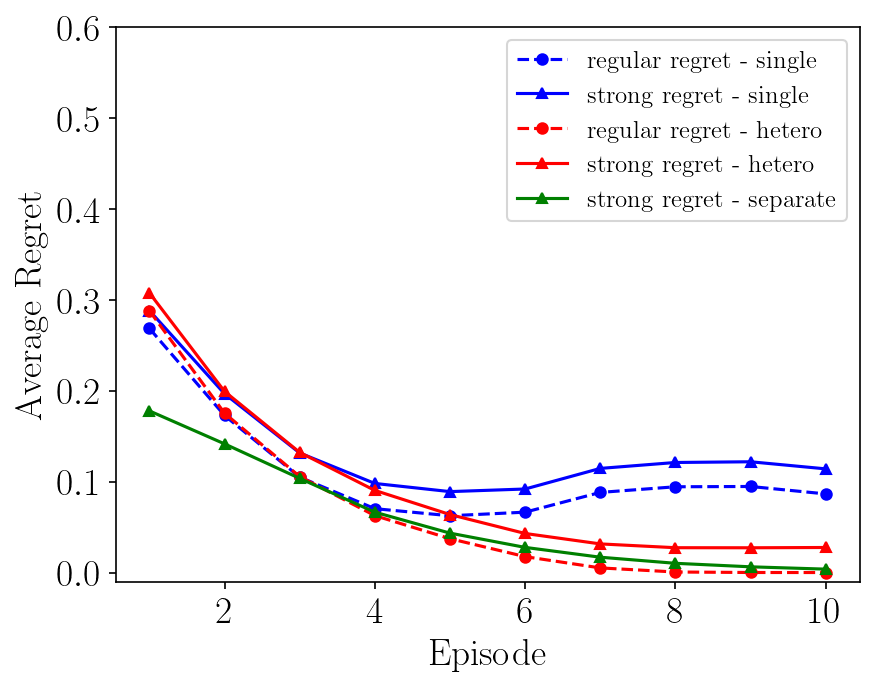}
		\caption{low-consumption group}
		\label{fig:real_g2}
	\end{subfigure}
	\caption{Average strong and regular regrets of different methods on the cash bonus dataset}
	\label{fig:real}
\end{figure}

Figure \ref{fig:real} presents a comparative analysis of the average regrets across different methods for both groups, where we set the initial sample size $n_0$ to be 256. Our proposed heterogeneous method demonstrates superior performance compared to the single LASSO approach in both groups, though with differences in the magnitude and pattern of improvement. For the high-consumption group,  the ``single'' method shows consistently higher regret levels, while our heterogeneous method achieves significantly lower regret values. For the low-consumption group, while the improvement is less clear in early episodes, it becomes particularly evident as the sample size increases. 

The performance difference of the ``single'' method between the two groups reveals the heterogeneity in purchasing behaviors and transaction values. The higher regret level in the high-consumption group reflects larger deviations from optimal bonus allocation, likely due to their higher baseline spending and larger transaction amounts. 
Since the ``single'' method fits only one LASSO model per episode, it effectively tracks only the behavior patterns of low-consumption users, resulting in more accurate predictions for this group but substantial errors for high-consumption users.

In contrast, our proposed heterogeneous method successfully captures the behavior patterns of both groups, verified by its comparable regret to the ``separate'' method that represents the performance of LASSO under known group assignments. This superior performance indicates that platforms can achieve more efficient bonus allocation strategies by incorporating latent group structures, potentially leading to improved user engagement and platform profitability across different user segments.

\section{Conclusion} \label{sec:conclusion}
This paper advances the field of online decision-making by addressing the critical challenge of latent heterogeneity in stochastic linear bandits. We introduce a novel framework that explicitly models unobserved customer characteristics affecting their responses to business actions, filling a significant gap in existing approaches primarily focusing on observable heterogeneity. Our methodology introduces an innovative algorithm that simultaneously learns both latent group memberships and group-specific reward functions, effectively handling the challenge of not having direct observations of group labels.

Our theoretical analysis reveals important insights about decision-making under latent heterogeneity. We establish that while the ``strong regret'' against an oracle with perfect group knowledge remains non-sub-linear due to inherent classification uncertainty, the ``regular regret'' against an oracle aware of only deterministic components achieves a minimax optimal rate in terms of $T$ and $d$. 

Through empirical validation using data from a mobile commerce platform, we demonstrate the practical value of our framework. The results show that our approach effectively handles real-world scenarios where customer heterogeneity plays a crucial role, such as in personalized pricing, resource allocation, and inventory management. This work provides practitioners with theoretically grounded tools for making sequential decisions under latent heterogeneity, bridging the gap between theoretical understanding and practical implementation in business applications.

Several important directions remain for future research. First, extending our framework to nonlinear structured bandit models would enable applications in more complex decision scenarios. Developing online tests for latent heterogeneity and studying heterogeneous treatment effects in bandits would further enhance the practical utility of our approach. Moreover, our finding that the ``strong regret'' remains non-sub-linear due to inherent classification uncertainty suggests a fundamental limitation that cannot be addressed through algorithmic improvements alone. This points to the necessity of mechanism design approaches that could incentivize customers to reveal their latent group memberships, potentially opening a new avenue of research at the intersection of bandit algorithms and mechanism design.

%
%
\spacingset{1.5}

\bibliographystyle{apalike}
\bibliography{0-main}

%
%
\newpage
\spacingset{1.5}
\appendix
\setcounter{page}{1}

\begin{center}
	{\Large\bf Supplementary Materials of \\
		``\TITLE''}
\end{center}


\section{Proof of Theorem \ref{thm:coeff-bound-hd}} \label{sec:proof-coeff-bound}

The EM algorithm is essentially an alternating maximization method, which alternatively optimizes between the identification of hidden labels $\{g_i\}$ and the estimation of parameter $\bgamma=(\btheta, \bbeta_1, \bbeta_2)$.
In the $\tau$-th episode, we utilize $N_{\tau-1} = n_02^{\tau - 1}$ samples from the previous episode. 
In the $(t+1)$-th iteration of the EM-algorithm, we have i.i.d samples in the set $\calI_{\tau-1}^{(t+1)}$ of size $n_{\tau}=N_{\tau-1}/t_{\tau, \max}$.

By the sub-Gaussianity of $\bx_i$, $\bz_i$, and $\epsilon_i$, there exist a constant $C$ such that \[\EE\left[(\bx_i^{\top}\bv)^{m}\right] \leq C^m \sigma_{x}^mm^{m/2}, \quad \EE[\epsilon_{i}^m]\leq C^m \sigma^mm^{m/2},
\text{ and } 
\PP\paren{\abs{\bz_i^{\top}\bv} \leq \mu} < 2e^{-\frac{\mu^2}{2\sigma^2_{z}}},\] 
for all $\mu>0$, all non-negative integers $m$, and all $\bv$ such that $\norm{\bv}_2=1$, where $\sigma_x$, $\sigma_z$, and $\sigma$ are the sub-Gaussian parameters of $\bx_i$, $\bz_i$, and $\epsilon_i$, respectively. 
Moreover, define that $\delta_{\bgamma}^{(\tau, 0)}=\begin{cases}
	\delta_{1, 0}, & \text{if } \tau=1;\\
	 \sqrt{\frac{s\log d \log n_0}{n_0}}, & \text{if } \tau \geq 2.\end{cases}$

We will show the following theorem, which is a more detailed version of Theorem $\ref{thm:coeff-bound-hd}$:
\begin{theorem}\label{thm:1-detailed}
	Let \[\left(\hat{\btheta}^{(\tau)},\hat{\bbeta}_{1}^{(\tau)},\hat{\bbeta}_{2}^{(\tau)}\right)=\left(\btheta^{(\tau,t_{\tau, \max})},\bbeta^{(\tau, t_{\tau, \max})}_{1},\bbeta^{(\tau, t_{\tau, \max})}_{2}\right)\] 
	be the output of Algorithm \ref{alg:em} in the $\tau$-th episode. There exist constants $c_1$, $c_2$,  and $\overline{C}$ such that, if Assumptions \ref{A1}--\ref{A4} hold with $\delta_{1, 0} \leq c_1\min\{\xi, \sigma C_{\rm{SNR}}\}$ and $C_{\rm{SNR}} \geq c_2$, by letting the number of iterations $t_{\tau, \max} \asymp \log n_0$ for $\tau=1$ and $t_{\tau, \max} \asymp 1$  for $\tau \geq 2$ and choosing $\kappa<(2\oC^2)^{-1}$,  $\widetilde\kappa=\oC^2\kappa<1/2$, \[\lambda_{n_{\tau}}^{(t+1)}=\frac{2\oC(1-(2\widetilde\kappa)^{t+1})}{1-2\widetilde\kappa}\sqrt{\frac{\log d}{n_{\tau}}}+\frac{\oC\kappa(2\widetilde\kappa)^{t}}{\sqrt{s}}\delta_{\bgamma}^{(\tau, 0)}, \]
	we have
	\begin{equation}
		\norm{\hat{\bbeta}_{1}^{(\tau)} - \bbeta_{1}^*}_2 + \norm{\hat{\bbeta}_{2}^{(\tau)} - \bbeta_{2}^*}_2 + \norm{\hat{\btheta}^{(\tau)} - \btheta^*}_2
		\lesssim \frac{1}{\kappa(1-2\widetilde\kappa)}\sqrt{\frac{s\log d\log n_0}{N_{\tau-1}}}, 
	\end{equation}
	and
	\begin{equation}
		\norm{\hat{\bbeta}_{1}^{(\tau)} - \bbeta_{1}^*}_1 + \norm{\hat{\bbeta}_{2}^{(\tau)} - \bbeta_{2}^*}_1 + \norm{\hat{\btheta}^{(\tau)} - \btheta^*}_1
		\lesssim   \frac{1}{\kappa(1-2\widetilde\kappa)}\sqrt{\frac{s^2\log d\log n_0}{N_{\tau-1}}},
	\end{equation}
	with probability at least $1-\frac{c\log^3 n_0}{\max^2\{N_{\tau-1}, d\}}$ for some constant $c$.
\end{theorem}

We temporarily drop the index $\tau$ in $N_{\tau-1}$, $n_{\tau}$, $\cI_{\tau-1}^{(t)}$, $t_{\tau, \max}$, $\omega_i^{(\tau, t)}$, $\bgamma^{(\tau, t)}$, and $\lambda_{n_{\tau}}^{(t)}$ when we are considering a single episode $\tau$. 
 The overall objective function $Q_n$ in the M-Step is the sum of the single-observation objective functions, that is, 
\begin{equation}
    Q_n(\bgamma \cond \bgamma^{(t)})
    \; := \; 
    \sum_{i \in \calI^{(t+1)}}\EE\brackets{ \ell\paren{\bx_i, y_i, \bz_i, g_i; \bgamma} \cond \bx_i, y_i, \bz_i; \bgamma^{(t)}}.
\end{equation}
For clearness, we replace $i \in \calI^{(t+1)}$ with $\sum_{i=1}^{n}$ when there is no ambiguity.
Simple calculation yields that 
\begin{align}\label{eqn:Qn}
    Q_n(\bgamma \cond \bgamma^{(t)})
    & = 
   \frac{1}{n} \sum_{i=1}^{n} \left[\omega_i^{(t)}\cdot \paren{ \log p(\bz_i^\top\btheta)-
    \frac{(y_i - \bx_i^\top\bbeta_1)^2 }{2\sigma^2 }} \right.   \\
    &\quad +
    \left. (1- \omega_i^{(t)}) \cdot \paren{\log(1-p(\bz_i^\top\btheta)) -
    \frac{(y_i - \bx_i^\top\bbeta_2)^2}{2\sigma^2}} \right]  \nonumber \\
   & = -  \frac{1}{2n} \sum_{i=1}^{n} \omega_i^{(t)} \cdot  \frac{(y_i - \bx_i^\top\bbeta_1)^2 }{\sigma^2} - \frac{1}{2n} \sum_{i=1}^{n} (1- \omega_i^{(t)}) \cdot \frac{(y_i - \bx_i^\top\bbeta_2)^2 }{\sigma^2}  \nonumber   \\
   & \quad + \frac{1}{n} \sum_{i=1}^{n}\left[ \omega_i^{(t)} \cdot \log p(\bz_i^\top\btheta) + (1- \omega_i^{(t)}) \cdot \log(1-p(\bz_i^\top\btheta))\right]  \nonumber 
\end{align}
where $\omega_i^{(t)} = \omega_i(\bgamma^{(t)}) = \omega(\bx_i, y_i, \bz_i ; \bgamma^{(t)})$ is defined by
\[
\omega(\bx, y, \bz ; \bgamma) 
= \frac{p(\bz^\top\btheta) \cdot
	    \phi\paren{\frac{y - \bx^\top\bbeta_1}{\sigma}}}{p(\bz^\top\btheta) \cdot \phi\paren{\frac{y - \bx^\top\bbeta_1}{\sigma}} 
	    + \paren{1-p(\bz^\top\btheta)} \cdot \phi\paren{\frac{y - \bx^\top\bbeta_2}{\sigma}}}.
\]
Let 
\begin{equation} 
\begin{aligned}
    Q_{n1}(\bbeta_1 \cond \bgamma^{(t)}) & =  -\frac{1}{2n} \sum_{i=1}^{n} \omega_i^{(t)}\cdot  \frac{(y_i - \bx_i^\top\bbeta_1)^2 }{\sigma^2}, \\
    Q_{n2}(\bbeta_2 \cond \bgamma^{(t)}) & =  -\frac{1}{2n} \sum_{i=1}^{n} (1- \omega_i^{(t)}) \cdot \frac{(y_i - \bx_i^\top\bbeta_2)^2 }{\sigma^2}, \quad\text{and}\quad \\
    Q_{n3}(\btheta \cond \bgamma^{(t)}) & =  \frac{1}{n} \sum_{i=1}^{n} \paren{ \omega_i^{(t)} \cdot \log p(\bz_i^\top\btheta) + (1- \omega_i^{(t)}) \cdot \log(1-p(\bz_i^\top\btheta)) }.
\end{aligned}
\end{equation}
Then, in the $(t+1)$-th iteration, the M-step is implemented as 
\begin{align}
    \bbeta_1^{(t+1)} & := \underset{\bbeta_1}{\arg\min} \;  -Q_{n1}\paren{\bbeta_1 \cond \bgamma^{(t)}}  
    + \lambda_n^{(t)} \norm{\bbeta_1}_1,  \label{eqn:beta-1-opt}\\
    \bbeta_2^{(t+1)} & := \underset{\bbeta_2}{\arg\min} \; -Q_{n2}\paren{\bbeta_2 \cond \bgamma^{(t)}}  
    + \lambda_n^{(t)} \norm{\bbeta_2}_1,  \label{eqn:beta-2-opt} \\
    \btheta^{(t+1)} & := \underset{\btheta}{\arg\min} \; -Q_{n3}\paren{\btheta \cond \bgamma^{(t)}} 
    + \lambda_n^{(t)} \norm{\btheta}_1.  \label{eqn:theta-opt}
\end{align}


To prove Theorem \ref{thm:1-detailed}, we first present Lemmas \ref{thm:lemma-expectation} and \ref{thm:lemma-sample} on iterative estimation bounds. Let $\omega_i^* = \omega_i(\bgamma^*) = \omega(\bx_i, y_i, \bz_i ; \bgamma^*)$.

\begin{lemma}[Population EM iterates]  
\label{thm:lemma-expectation}
Assume that Assumptions \ref{A1}--\ref{A4} hold and that $\bx_i^{\top}(\bbeta_1^*-\bbeta^*_2)/\norm{\bbeta_1^*-\bbeta^*_2}_2$ has a bounded density around 0.  For any constant $\kappa>0$,  there exist constants $c_1$, $c_2$ such that, 
if $\norm{\bbeta_1^{(t)} - \bbeta_1^*}_2 + \norm{\bbeta_2^{(t)} - \bbeta_2^*}_2 + \norm{\btheta^{(t)} - \btheta^*}_2 \le c_1\min\big\{\xi, \norm{\bbeta_1^*-\bbeta_2^*}_2\big\}$ and $\norm{\bbeta_1^*-\bbeta_2^*}_2 /\sigma > c_2$, we have
\begin{align*}
    \norm{ \EE\left[ \omega_i^{(t)} \bx_i ( \bx_i^\top \bbeta_1^* - y_i) \right] 
        - \EE\left[ \omega_i^* \bx_i ( \bx_i^\top \bbeta_1^* - y_i) \right] }_2  
    & \le \kappa \cdot \paren{\norm{\bbeta_1^{(t)} - \bbeta_1^*}_2 + \norm{\bbeta_2^{(t)} - \bbeta_2^*}_2 + \norm{\btheta^{(t)} - \btheta^*}_2}, 
    \\
    \norm{ \EE\left[ ( \omega_i^{(t)} - p(\bz_i^\top\btheta^{*}) ) \bz_i \right] 
       - \EE\left[ (\omega_i^* - p(\bz_i^\top\btheta^{*}) ) \bz_i \right] }_2
    & \le  \kappa \cdot \paren{\norm{\bbeta_1^{(t)} - \bbeta_1^*}_2 + \norm{\bbeta_2^{(t)} - \bbeta_2^*}_2 + \norm{\btheta^{(t)} - \btheta^*}_2},\\
       \norm{ \EE\left[\omega_i^{(t)} \bx_i\bx^{\top}_i\right] 
    	- \EE\left[\omega_i^*\bx_i\bx^{\top}_i \right]}_2& \leq \kappa \cdot \paren{\norm{\bbeta_1^{(t)} - \bbeta_1^*}_2 + \norm{\bbeta_2^{(t)} - \bbeta_2^*}_2 + \norm{\btheta^{(t)} - \btheta^*}_2}.
\end{align*}
\end{lemma}
 Proof of Lemma \ref{thm:lemma-expectation} is provided in Section \ref{sec:proof-lemma-expectation}. 

\begin{lemma}[Sample EM iterates]  \label{thm:lemma-sample}
Under the assumptions of Theorem \ref{thm:coeff-bound-hd}, suppose that $\bgamma^{(t)}$ is independent with $\{\bx_i, y_i, \bz_i\}$'s, then there exists some constant $C > 0$, such that, with probability at least $1 - \frac{6}{\max\{n, d\}^2}$,  
\begin{align*}
\norm{ \frac{1}{n}\sum_{i=1}^n \left[\omega_i^{(t)} \bx_i ( \bx_i^\top \bbeta_1^* - y_i) \right] 
    - \EE\left[ \omega_i^{(t)} \bx_i ( \bx_i^\top \bbeta_1^* - y_i) \right] }_\infty 
& \le  C \sqrt{ \frac{\log \max\{n, d\}}{n} },  \\
\norm{  \frac{1}{n}\sum_{i=1}^n \left[( \omega_i^{(t)} - p(\bz_i^\top\btheta^{*}) ) \bz_i \right]
	- \EE\left[ (\omega_i^{(t)} - p(\bz_i^\top\btheta^{*}) ) \bz_i \right] }_\infty
& \le C \sqrt{ \frac{\log \max\{n, d\}}{n}},\\
\norm{\frac{1}{n}\sum_{i=1}^n \left[ \omega_i^{(t)} \bx_i\bx_i^{\top} \right]
	- \EE\left[ \omega_i^{(t)}\bx_i\bx_i^{\top}  \right]}_{\max}
& \le C \sqrt{ \frac{\log \max\{n, d\}}{n}},
\end{align*}
where $\norm{A}_{\max}:=\max\limits_{j, k \in [d]} |A_{j, k}|$ for any matrix $A \in \RR^{d \times d}$.
\end{lemma}
 
 Proof of Lemma \ref{thm:lemma-sample} is provided in Section \ref{sec:proof-lemma-sample}. 

\medskip

In Steps 1 and 2 in the following proof, we use $\hat{\bbeta}_1$, $\hat{\bbeta}_2$, $\hat{\btheta}$, $\lambda$ to replace $\bbeta_1^{(t+1)}$, $\bbeta_2^{(t+1)}$,  $\btheta^{(t+1)}$, and $\lambda_n^{(t+1)}$ for simplicity. Also, we temporarily suppose that $\norm{\bbeta_1^{(t)} - \bbeta_1^*}_2 + \norm{\bbeta_2^{(t)} - \bbeta_2^*}_2 + \norm{\btheta^{(t)} - \btheta^*}_2 \le c_1 \min\braces{\xi, \norm{\bbeta_1^*-\bbeta_2^*}_2}$, which is satisfied with $t=0$ and will be shown by induction in Step 3.

\noindent
\textbf{STEP 1. Sample EM iterative bounds for $\bbeta_1$ and $\bbeta_2$.} 

We make use of the definition of \eqref{eqn:beta-1-opt} and \eqref{eqn:beta-2-opt} and a decomposition of the main objective function $Q_{n1}$ as follows:
\begin{align}
-Q_{n1}\paren{{\bbeta_1^*} \cond \bgamma^{(t)}} + Q_{n1}\paren{\hat\bbeta_1 \cond \bgamma^{(t)}}
& = \frac{1}{n\sigma^2}\angles{\sum_{i=1}^{n}\omega_i^{(t)}\paren{y_i-\bx_i^\top\bbeta_1^*} \bx_i,\; \hat\bbeta_1-\bbeta_1^*} \nonumber\\
& - \paren{\hat\bbeta_1-\bbeta_1^*}^\top 
\paren{\frac{1}{2n\sigma^2}\sum_{i=1}^n\omega_i^{(t)}\bx_i\bx_i^\top}
\paren{\hat\bbeta_1-\bbeta_1^*}. \label{eqn:Qn-decomp}
\end{align}
By \eqref{eqn:beta-1-opt}, we have
\begin{equation} \label{eqn:Qn-diff-lb-1}
\lambda\sigma^2\paren{\norm{\hat\bbeta_1}_1-\norm{\bbeta_1^*}_1} 
\le -Q_{n1}\paren{{\bbeta_1^*} \cond \bgamma^{(t)}} + Q_{n1}\paren{\hat\bbeta_1 \cond \bgamma^{(t)}}.
\end{equation}

Combine \eqref{eqn:Qn-decomp} and \eqref{eqn:Qn-diff-lb-1}, we have
\begin{equation}
\begin{aligned}
&\quad  \paren{\hat\bbeta_1-\bbeta_1^*}^\top 
\paren{\frac{1}{2n}\sum_{i=1}^n\omega_i^{(t)}\bx_i\bx_i^{\top}}
\paren{\hat\bbeta_1-\bbeta_1^*} \\
&\le 
\frac{1}{n}\angles{\sum_{i=1}^{n}\omega_i^{(t)}\paren{y_i-\bx_i^\top\bbeta_1^*} \bx_i,\; \hat\bbeta_1-\bbeta_1^*} - \lambda\sigma^2\paren{\norm{\hat\bbeta_1}_1-\norm{\bbeta_1^*}_1}.
\end{aligned}
\end{equation}
In addition, we have the fact that $\EE\brackets{\omega_i^*\paren{y_i-\bx_i^\top\bbeta_1^*} \bx_i}=0$. To see this, note that the probability density function of $y_i$ given $\bx_i, \bz_i$ is
\[f(y_i;\bgamma^*) = p(\bz_i^{\top}\btheta^*)\phi\left(\frac{y_i-\bx_i^{\top}\bbeta_1^*}{\sigma}\right)+(1-p(\bz_i^{\top}\btheta^*))\exp\left(\frac{y_i-\bx_i^{\top}\bbeta_2^*}{\sigma}\right).\]
It is well known that  
\[\EE_y\left[\frac{\partial \log f(y; \bgamma^*)}{\partial \bgamma}\right] = \EE_y\left[\frac{1}{f(y;\bgamma^*)}\frac{\partial f(y; \bgamma^*)}{\partial \bgamma}\right]=\int f(y;\bgamma^*) \frac{1}{f(y;\bgamma^*)} \frac{\partial f(y; \bgamma^*)}{\partial \bgamma} \mathrm{d}y=0,\]
which yields that
\begin{align*}
	0&=\EE_{y_i}\left[\frac{\partial \log f(y_i; \bgamma^*)}{\partial \bbeta_1}\right]\\
	&=\EE_{y_i}\left[\frac{1}{f(y_i;\bgamma^*)}p(\bz_i^{\top}\btheta^*)\left(y_i-\bx_i^\top \bbeta_1^*\right) \bx_i \phi\left(\frac{y_i-\bx_i^{\top}\bbeta_1^*}{\sigma}\right)\right]\\
	&=\EE_{y_i}\left[\omega_i^*\left(y_i-\bx^\top \bbeta_1^*\right) \bx_i \right].
\end{align*}

Applying the decomposition
\begin{align*}
&\quad \angles{\frac{1}{n}\sum_{i=1}^{n}\omega_i^{(t)}\paren{y_i-\bx_i^\top\bbeta_1^*} \bx_i,\; \hat\bbeta_1-\bbeta_1^*} \\
& \le \abs{\angles{\frac{1}{n}\sum_{i=1}^{n}\omega_i^{(t)}\paren{y_i-\bx_i^\top\bbeta_1^*} \bx_i - \EE\brackets{\omega_i^{(t)}\paren{y_i-\bx_i^\top\bbeta_1^*} \bx_i},\; \hat\bbeta_1-\bbeta_1^*}} \\
& ~~~ + \abs{\angles{\EE\brackets{\omega_i^{(t)}\paren{y_i-\bx_i^\top\bbeta_1^*} \bx_i} - \EE\brackets{\omega_i^*\paren{y_i-\bx_i^\top\bbeta_1^*} \bx_i},\; \hat\bbeta_1-\bbeta_1^*}} \nonumber
\end{align*} 
and Lemmas \ref{thm:lemma-expectation} and \ref{thm:lemma-sample}, we obtain, with probability at least $1-2/\max\{n, d\}^2$, 
\begin{equation} \label{eqn:Qn-diff-ub}
\begin{aligned}
& \angles{\frac{1}{n}\sum_{i=1}^{n}\omega_i^{(t)}\paren{y_i-\bx_i^\top\bbeta_1^*} \bx_i,\; \hat\bbeta_1-\bbeta_1^*} \\
& \le  
C \sqrt{\frac{\log \max\{n, d\} }{ n}}\cdot\norm{\hat\bbeta_1-\bbeta_1^*}_1 
+ \kappa\cdot\paren{\norm{\bbeta_1^{(t)} - \bbeta^*_1}_2 + \norm{\bbeta_2^{(t)} - \bbeta^*_2}_2 + \norm{\btheta^{(t)} - \btheta^*}_2}\cdot\norm{\hat\bbeta_1-\bbeta_1^*}_2. 
\end{aligned}
\end{equation}

\begin{lemma} \label{thm:diff-L1-Bound}
Let $S = {\rm supp}(\bbeta_1^*)$ and $s=|S|$. When \[\lambda\sigma^2\ge 3C\sqrt{\frac{\log \max\{d, n\}}{ n}} + \frac{\kappa}{\sqrt{s}}\cdot\paren{\norm{\bbeta_1^{(t)} - \bbeta^*_1}_2 + \norm{\bbeta_2^{(t)} - \bbeta^*_2}_2 + \norm{\btheta^{(t)} - \btheta^*}_2},\] we have
\begin{equation}
\norm{ \hat\bbeta_1 - \bbeta_1^*}_1 \leq 5\sqrt{s}\norm{ \hat\bbeta_1 - \bbeta_1^*}_2. 
\end{equation}
\end{lemma}
\begin{proof}{Proof of Lemma \ref{thm:diff-L1-Bound}}
Let $\bu = \hat\bbeta_1 - \bbeta_1^*$. Combining the inequality from the definition of each iterates
\begin{equation} 
\lambda\sigma^2\paren{\norm{\hat\bbeta_1}_1-\norm{\bbeta_1^*}_1} 
\le -Q_{n1}\paren{{\bbeta_1^*} \cond \bgamma^{(t)}} + Q_{n1}\paren{\hat\bbeta_1 \cond \bgamma^{(t)}} 
\end{equation}
and the inequality that
\begin{equation*}
\norm{\hat\bbeta_1}_1-\norm{\bbeta_1^*}_1 
\ge \norm{\bbeta_1^* + \bu_{S^C}}_1 - \norm{\bu_{S}}_1-\norm{\bbeta_1^*}_1
= \norm{\bu_{S^C}}_1 - \norm{\bu_{S}}_1,
\end{equation*}
we obtain
\begin{equation} \label{eqn:Qn-diff-lb}
\lambda\sigma^2\paren{\norm{\bu_{S^C}}_1 - \norm{\bu_{S}}_1}
\le 
-Q_{n1}\paren{{\bbeta_1^*} \cond \bgamma^{(t)}} + Q_{n1}\paren{\hat\bbeta_1 \cond \bgamma^{(t)}}.
\end{equation}
Combining \eqref{eqn:Qn-decomp}, \eqref{eqn:Qn-diff-ub} and \eqref{eqn:Qn-diff-lb}, we have 
\begin{align*}
\lambda\sigma^2\paren{\norm{\bu_{S^C}}_1 - \norm{\bu_{S}}_1}
& \le -Q_{n1}\paren{{\bbeta_1^*} \cond \bgamma^{(t)}} + Q_{n1}\paren{\hat\bbeta_1 \cond \bgamma^{(t)}} \\
& \le \frac{1}{n}\angles{\sum_{i=1}^{n}\omega_i^{(t)}\paren{y_i-\bx_i^\top\bbeta_1^*} \bx_i,\; \hat\bbeta_1-\bbeta_1^*} \\
& \le 
C \sqrt{\frac{\log \max\{d, n\}} {n}}\cdot\norm{\bu}_1 \\
&\quad + \frac{\kappa}{\sqrt{s}}\cdot\paren{\norm{\bbeta_1^{(t)} - \bbeta^*}_2 + \norm{\bbeta_2^{(t)} - \bbeta^*}_2 + \norm{\btheta^{(t)} - \btheta^*}_2}\cdot\sqrt{s}\norm{\bu}_2. 
\end{align*} 
Let \[\lambda\sigma^2\ge 3C\sqrt{\frac{\log \max\{d, n\}}{ n}} + \frac{\kappa}{\sqrt{s}}\cdot\paren{\norm{\bbeta_1^{(t)} - \bbeta^*}_2 + \norm{\bbeta_2^{(t)} - \bbeta^*}_2 + \norm{\btheta^{(t)} - \btheta^*}_2},\] we have
\begin{equation*}
\norm{\bu_{S^C}}_1 \le 2\norm{\bu_{S}}_1 + 3/2\sqrt{s}\norm{\bu}_2 \le 4\sqrt{s}\norm{\bu}_2,
\end{equation*}
and hence $\norm{\bu} \leq \norm{\bu_{S}}_1+\norm{\bu_{S^c}}_1 \leq 5\sqrt{s}\norm{\bu}_2$. 
\end{proof}

Then, applying Lemma \ref{thm:diff-L1-Bound} to \eqref{eqn:Qn-diff-ub}, we have
\begin{equation}
\begin{aligned}
&\quad  \paren{\hat\bbeta_1-\bbeta_1^*}^\top 
\paren{\frac{1}{2n}\sum_{i=1}^n\omega_i^{(t)}\bx_i\bx_i^{\top}}
\paren{\hat\bbeta_1-\bbeta_1^*} \label{eqn:Q1n-2nd-mm-ub}\\
& \lesssim \sqrt{\frac{s\log \max\{d, n\}} {n}}\cdot\norm{\hat\bbeta_1-\bbeta_1^*}_2
+ \kappa\cdot\paren{\norm{\bbeta_1^{(t)} - \bbeta_1^*}_2 + \norm{\bbeta_2^{(t)} - \bbeta_2^*}_2 + \norm{\btheta^{(t)} - \btheta^*}_2}\cdot\norm{\hat\bbeta_1-\bbeta_1^*}_2  \\
&\quad + \lambda\sigma^2\sqrt{s}\norm{\hat\bbeta_1-\bbeta_1^*}_2. \nonumber
\end{aligned}
\end{equation}
Second, we establish a lower bound of the second-order term. 
Assumption \ref{A1} ensures that $p(\bz_i^{\top}\btheta^*) \in (\xi, 1-\xi)$. By Lemmas \ref{thm:lemma-expectation} and \ref{thm:lemma-sample}, 
we have, with probability at least $1-\frac{2}{\max\{d, n\}^2}$,  \[\norm{\frac{1}{n}\sum_{i=1}^{n}\left[ \omega_i^{(t)} \bx_i\bx_i^{\top} \right]-\EE\left[\omega_i^{(t)}\bx_i\bx_i^{\top}\right]}_{\max} \leq C \sqrt{\frac{\log \max \{d, n\}}{n}},\]
and
\[\norm{\EE\left[\omega_i^{(t)}\bx_i\bx_i^{\top}\right]-\EE\left[\omega_i^*\bx_i\bx_i^{\top}\right]}_2\leq \kappa \paren{\norm{\bbeta_1^{(t)} - \bbeta_1^*}_2 + \norm{\bbeta_2^{(t)} - \bbeta_2^*}_{2} + \norm{\btheta^{(t)} - \btheta^*}_2}.\]
For any vector $\bv$ such that $\norm{\bv}_2=1$, we have $\bv^{\top}\EE[\omega_i^*\bx_i\bx_i^{\top}]\bv =\EE[\omega_i^*(\bv^\top\bx_i)^2]\geq \PP(\omega_i^* \geq \xi/2)\cdot (\xi/2M)$, where $ \PP(\omega_i^* \geq \xi/2)$ is a positive constant since $\EE[\omega_i^*] = \EE[p(\bz_i^{\top}\btheta^*)] \geq \xi$. Since $\norm{\bbeta_1^{(t)} - \bbeta_1^*}_2 + \norm{\bbeta_2^{(t)} - \bbeta_2^*}_2 + \norm{\btheta^{(t)} - \btheta^*}_2 \leq c_1 \xi$, by requiring $c_1$ sufficiently small, we can ensure that $\norm{\EE\left[\omega_i^{(t)}\bx_i\bx_i^{\top}\right]}_2 \geq c_3$ for some constant $c_3>0$. As a result, for sufficiently large $n$,
\begin{align}
& \paren{\hat\bbeta_1-\bbeta_1^*}^\top 
\paren{\frac{1}{2n}\sum_{i=1}^n\omega_i^{(t)}\bx_i\bx_i^\top}
\paren{\hat\bbeta_1-\bbeta_1^*} 
\nonumber\\
&= \frac{1}{2}\paren{\hat\bbeta_1-\bbeta_1^*}^\top 
\paren{\frac{1}{n}\sum_{i=1}^n\omega_i^{(t)}\bx_i\bx_i^\top - \EE\left[\omega_i^{(t)}\bx_i\bx_i^{\top}\right]} 
\paren{\hat\bbeta_1-\bbeta_1^*}+ \frac{1}{2}\paren{\hat\bbeta_1-\bbeta_1^*}^\top 
\EE\left[\omega_i^{(t)}\bx_i\bx_i^{\top}\right]
\paren{\hat\bbeta_1-\bbeta_1^*}\nonumber\\
& \ge \frac{c_3}{2} \norm{\hat\bbeta_1-\bbeta_1^*}_2^2-\frac{25}{2} Cs\sqrt{\frac{\log \max \{d, n\}}{n}}\norm{\hat\bbeta_1-\bbeta_1^*}_2^2 \geq  \frac{c_3}{4} \norm{\hat\bbeta_1-\bbeta_1^*}_2^2,
\label{eqn:Q1n-2nd-mm-lb}
\end{align}
where we use 
the result in Lemma  \ref{thm:diff-L1-Bound} and the assumption that $s\sqrt{\frac{\log \max \{n, d\}}{n}}=o(1)$.
Combining \eqref{eqn:Qn-diff-ub} and \eqref{eqn:Q1n-2nd-mm-lb}, we have with probability at least $1-c_0\max\{n, d\}^{-2}$, 
\begin{equation} \label{eqn:beta1-one-iter}
\norm{\hat\bbeta_1-\bbeta_1^*}_2 
\lesssim    
\sqrt{ \frac{s\log  \max\{d, n\} } {n}}
+ \kappa\cdot\paren{\norm{\bbeta_1^{(t)} - \bbeta_1^*}_2 + \norm{\bbeta_2^{(t)} - \bbeta_2^*}_2 + \norm{\btheta^{(t)} - \btheta^*}_2}
+ \lambda\sigma^2\sqrt{s},
\end{equation}
for some constant $c_0>0$.
Similarly, we obtain that with with probability at least $1-c_0\max\{n, d\}^{-2}$,
\begin{equation}\label{eqn:beta2-one-iter}
\norm{\hat\bbeta_2-\bbeta_2^*}_2
\lesssim \sqrt{ \frac{s\log  \max\{d, n\} } {n}}
+ \kappa\cdot\paren{\norm{\bbeta_1^{(t)} - \bbeta_1^*}_2 + \norm{\bbeta_2^{(t)} - \bbeta_2^*}_2 + \norm{\btheta^{(t)} - \btheta^*}_2}
+ \lambda\sigma^2\sqrt{s}.
\end{equation}

\medskip
\noindent
\textbf{STEP 2. Sample EM iterative bounds for $\btheta$.} 
We make use of the definition of \eqref{eqn:theta-opt}. 
Let $\bu = \hat\btheta$ and $S = {\rm supp}(\btheta^*)$.
Firstly, by the definition of estimator, we have that 
\begin{equation} \label{eqn:Qn3-diff-lb}
\lambda\paren{\norm{\hat\btheta}_1 - \norm{\btheta^*}_1} 
\le 
-Q_{n3}\paren{\btheta^* \cond \bgamma^{(t)}} + Q_{n3}\paren{\hat\btheta \cond \bgamma^{(t)}} .
\end{equation}
In addition, we have
\begin{align}
-Q_{n3}\paren{\hat\btheta \cond \bgamma^{(t)}}
+ Q_{n3}\paren{{\btheta^*} \cond \bgamma^{(t)}}
& = 
- \frac{1}{n} \sum_{i=1}^{n} \paren{ \omega_i^{(t)} \cdot \log p(\bz_i^\top\hat\btheta) + (1- \omega_i^{(t)}) \cdot \log(1-p(\bz_i^\top\hat\btheta)) } \nonumber\\
& + \frac{1}{n} \sum_{i=1}^{n} \paren{ \omega_i^{(t)} \cdot \log p(\bz_i^\top\btheta^*) + (1- \omega_i^{(t)}) \cdot \log(1-p(\bz_i^\top\btheta^*)) } \nonumber\\
& = \angles{-\frac{1}{n} \sum_{i=1}^{n} \paren{\omega_i^{(t)}- p(\bz_i^\top\btheta^*)}\bz_i, \; (\hat\btheta-\btheta^*) } \nonumber\\
& + (\hat\btheta-\btheta^*)^\top \paren{\frac{1}{n} \sum_{i=1}^{n}p(\bz_i^\top\btheta')(1-p(\bz_i^\top\btheta'))\bz_i\bz_i^\top} (\hat\btheta-\btheta^*), \label{eqn:Qn3-decomp}
\end{align}
for some $\btheta'$ between $\hat\btheta$ and $\btheta^*$. 
Thus, we have
\begin{align*}
&\quad (\hat\btheta-\btheta^*)^\top \paren{\frac{1}{n} \sum_{i=1}^{n}p(\bz_i^\top\btheta')(1-p(\bz_i^\top\btheta'))\bz_i\bz_i^\top} (\hat\btheta-\btheta^*)\\
& \le \abs{\angles{-\frac{1}{n} \sum_{i=1}^{n} \paren{\omega_i^{(t)}- p(\bz_i^\top\btheta^*)}\bz_i, \; (\hat\btheta-\btheta^*) }} + \lambda \paren{\norm{\btheta^*}_1 - \norm{\hat\btheta}_1}  \\
& \le \abs{\angles{\frac{1}{n} \sum_{i=1}^{n} (\omega_i^{(t)}- p(\bz_i^\top\btheta^*))\bz_i - \EE\brackets{(\omega_i^{(t)}- p(\bz_i^\top\btheta^*))\bz_i}, \; (\hat\btheta-\btheta^*) }} \\
& + \abs{\angles{\EE\brackets{(\omega_i^{(t)}- p(\bz_i^\top\btheta^*))\bz_i} - \EE\brackets{(\omega_i^*- p(\bz_i^\top\btheta^*))\bz_i}, \; (\hat\btheta-\btheta^*) }}  + \lambda \paren{\norm{\btheta^*}_1 - \norm{\hat\btheta}_1} . 
\end{align*}
Applying Lemmas \ref{thm:lemma-expectation} and \ref{thm:lemma-sample}, we have
\begin{align*}
&  \abs{\angles{-\frac{1}{n} \sum_{i=1}^{n} \paren{\omega_i^{(t)}- p(\bz_i^\top\btheta^*)}\bz_i, \; (\hat\btheta-\btheta^*) }}  \\
& \le 
C \sqrt{\log \max\{d,n\} / n}\cdot\norm{\hat\btheta-\btheta^*}_1 
+ \kappa\cdot\paren{\norm{\bbeta_1^{(t)} - \bbeta^*}_2 + \norm{\bbeta_2^{(t)} - \bbeta^*}_2 + \norm{\btheta^{(t)} - \btheta^*}_2}\cdot\norm{\hat\btheta-\btheta^*}_2.
\end{align*}
Similar to Lemma \ref{thm:diff-L1-Bound}, by taking $\lambda > 3C\sqrt{\frac{\log \max\{d,n\}}{ n}} + \frac{\kappa}{\sqrt{s}}\cdot\paren{\norm{\bbeta_1^{(t)} - \bbeta^*}_2 + \norm{\bbeta_2^{(t)} - \bbeta^*}_2 + \norm{\btheta^{(t)} - \btheta^*}_2}$, it holds that
\begin{equation}\label{eqn:Q3n-2nd-mm-upper}
	\begin{aligned}
		& (\hat\btheta-\btheta^*)^\top \paren{\frac{1}{n} \sum_{i=1}^{n}p(\bz_i^\top\btheta')(1-p(\bz_i^\top\btheta'))\bz_i\bz_i^\top} (\hat\btheta-\btheta^*)\\
		& \le 
		\angles{-\frac{1}{n} \sum_{i=1}^{n} \paren{\omega_i^{(t)}-\log p(\bz_i^\top\btheta^*)}\bz_i, \; (\hat\btheta-\btheta^*) } + \lambda \paren{\norm{\hat\btheta}_1 - \norm{\btheta^*}_1} \\
		& \le 
		C \sqrt{\frac{s \log \max\{d,n\} }{ n}}\cdot\norm{\hat\btheta-\btheta^*}_2 
		+ \kappa\cdot\paren{\norm{\bbeta_1^{(t)} - \bbeta^*}_2 + \norm{\bbeta_2^{(t)} - \bbeta^*}_2 + \norm{\btheta^{(t)} - \btheta^*}_2}\cdot\norm{\hat\btheta-\btheta^*}_2 \\
		&\quad + \sqrt{s}\lambda \norm{\hat\btheta - \btheta^*}_2.
	\end{aligned}
\end{equation}

Since $\btheta'$ is between $\hbtheta$ and $\btheta^*$ and $\bz_i$ is sub-Gaussian, it is straightforward to show that $\abs{\EE\brackets{p(\bz_i^\top\btheta')(1-p(\bz_i^\top\btheta'))-p(\bz_i^\top\btheta^*)(1-p(\bz_i^\top\btheta^*))}} \leq \frac{\xi^2}{2}$ for sufficiently small $c_1$, which yields $\EE\brackets{p(\bz_i^\top\btheta')(1-p(\bz_i^\top\btheta'))} \geq \xi^2/2$ since $\EE\left[p(\bz_i^\top\btheta^*)(1-p(\bz_i^\top\btheta^*))\right] \geq \xi^2$. 
Then, similar to Step 1, for some constant $c_3'>0$, we can obtain 
\begin{equation} \label{eqn:Q3n-2nd-mm-lower}
(\hat\btheta-\btheta^*)^\top \paren{\frac{1}{n} \sum_{i=1}^{n}p(\bz_i^\top\btheta')(1-p(\bz_i^\top\btheta'))\bz_i\bz_i^\top} (\hat\btheta-\btheta^*)
\ge c_3' \norm{\hat\btheta-\btheta^*}_2^2.
\end{equation}
Combining \eqref{eqn:Q3n-2nd-mm-upper} and \eqref{eqn:Q3n-2nd-mm-lower}, we have with probability at least $1-c_0\max\{n, d\}^{-2}$, 
\begin{equation}\label{eqn:theta-one-iter}
\norm{\hat\btheta-\btheta^*}_2
\lesssim    
 \sqrt{s \log \max\{d,n\} / n}
+ \kappa\cdot\paren{\norm{\bbeta_1^{(t)} - \bbeta_1^*}_2 + \norm{\bbeta_2^{(t)} - \bbeta_2^*}_2 + \norm{\btheta^{(t)} - \btheta^*}_2}
+ 
\lambda\sqrt{s}.
\end{equation}

\medskip
\noindent
\textbf{STEP 3. Proof by induction.}
Combining \eqref{eqn:beta1-one-iter},  \eqref{eqn:beta2-one-iter}, and \eqref{eqn:theta-one-iter}, we have that, with probability at least $1-c\max\{n, d\}^{-2}$
\begin{equation}\label{eq:cgamma}
	\begin{aligned}
		&\norm{\bbeta^{(t+1)}_1 - \bbeta_1^*}_2 + \norm{\bbeta^{(t+1)}_2 - \bbeta_2^*}_2 + \norm{\btheta^{(t+1)} - \btheta^*}_2 \\
		&\leq C_{\gamma} \brackets{
			\sqrt{s \log \max\{d,n\} / n} 
			+ \kappa\cdot\paren{\norm{\bbeta_1^{(t)} - \bbeta_1^*}_2 + \norm{\bbeta_2^{(t)} - \bbeta_2^*}_2 + \norm{\btheta^{(t)} - \btheta^*}_2}
			+ 
			\lambda_n^{(t+1)}\sqrt{s}}, 
	\end{aligned}
\end{equation}

when
\begin{equation*}
\lambda^{(t+1)}_n \ge C_\lambda \sqrt{\log \max\{d,n\} / n} + \kappa/\sqrt{s}\cdot\paren{\norm{\bbeta_1^{(t)} - \bbeta_1^*}_2 + \norm{\bbeta_2^{(t)} - \bbeta_2^*}_2 + \norm{\btheta^{(t)} - \btheta^*}_2},
\end{equation*}
for some absolute constants $c$, $C_{\gamma}$ and $C_{\lambda}$. Let $\oC: = \max\braces{C_{\lambda}, C_{\gamma}, 1}$, choose $\kappa<(2\oC^2)^{-1}$,  let $\widetilde\kappa:=\oC^2\kappa<1/2$, and define
\[\delta_{\bgamma}^{(t)}:=  \norm{\bbeta_1^{(t)} - \bbeta_1^*}_2 + \norm{\bbeta_2^{(t)} - \bbeta_2^*}_2 + \norm{\btheta^{(t)} - \btheta^*}_2.\]

We will show by induction that, by choosing 
\[\lambda_n^{(t+1)}=\frac{2\oC(1-(2\widetilde\kappa)^{t+1})}{1-2\widetilde\kappa}\sqrt{\frac{\log \max\{d,n\}}{n}}+\frac{\oC\kappa(2\widetilde\kappa)^{t}}{\sqrt{s}}\delta_{\bgamma}^{(0)}, \]
it holds that
\[\delta_{\bgamma}^{(t)} \leq \frac{2(1-(2\widetilde\kappa)^{t+1})}{\kappa(1-2\widetilde\kappa)}\sqrt{\frac{s\log \max\{d,n\}}{n}}+(2\widetilde\kappa)^t\delta_{\bgamma}^{(0)}, \]
and 
\[\lambda_n^{(t+1)} \geq C_\lambda\sqrt{\frac{\log \max\{d,n\}}{n}} + \frac{\kappa}{\sqrt{s}}\delta_{\bgamma}^{(t)}.\]

The case $t=0$ is trivial. Assume that the above two inequalities are true for $t$. Consider $t+1$. By \eqref{eq:cgamma}, we have
\begin{align*}
	\delta_{\bgamma}^{(t+1)} 
	&\leq \oC \brackets{1+\frac{2(1-(2\widetilde\kappa)^{t+1})}{1-2\widetilde\kappa}+\frac{2\oC(1-(2\widetilde\kappa)^{t+1})}{1-2\widetilde\kappa}}\sqrt{\frac{s\log \max\{d,n\}}{n}} +\paren{\oC\kappa+\oC^2\kappa}(2\widetilde{\kappa})^t\delta_{\bgamma}^{(0)}\\
	& \leq \frac{1}{\kappa} \brackets{\oC\kappa+\frac{2\oC\kappa(1-(2\widetilde\kappa)^{t+1})}{1-2\widetilde\kappa}+\frac{2\oC^2\kappa(1-(2\widetilde\kappa)^{t+1})}{1-2\widetilde\kappa}}\sqrt{\frac{s\log \max\{d,n\}}{n}}+(2\widetilde{\kappa})^{t+1}\delta_{\bgamma}^{(0)}\\
	& \leq \frac{1}{\kappa} \brackets{2+\frac{2\widetilde\kappa(1-(2\widetilde\kappa)^{t+1})}{1-2\widetilde\kappa}+\frac{2\widetilde\kappa(1-(2\widetilde\kappa)^{t+1})}{1-2\widetilde\kappa}}\sqrt{\frac{s\log \max\{d,n\}}{n}}+(2\widetilde{\kappa})^{t+1}\delta_{\bgamma}^{(0)}\\
	&\leq\frac{2(1-(2\widetilde\kappa)^{t+2})}{\kappa(1-2\widetilde\kappa)}\sqrt{\frac{s\log \max\{d,n\}}{n}}+(2\widetilde{\kappa})^{t+1}\delta_{\bgamma}^{(0)}.
\end{align*}
Furthermore,
\begin{align*}
	&\quad C_\lambda\sqrt{\frac{\log \max\{d,n\}}{n}} + \frac{\kappa}{\sqrt{s}}\delta_{\bgamma}^{(t+1)}\\
	& \leq  \brackets{\oC+\oC\kappa+\frac{2\oC\kappa(1-(2\widetilde\kappa)^{t+1})}{1-2\widetilde\kappa}+\frac{2\oC^2\kappa(1-(2\widetilde\kappa)^{t+1})}{1-2\widetilde\kappa}}\sqrt{\frac{\log \max\{d,n\}}{n}}+\frac{\kappa}{\sqrt{s}}(2\widetilde{\kappa})^{t+1}\delta_{\bgamma}^{(0)}\\
	& \leq \brackets{2\oC+\frac{2\oC^3\kappa(1-(2\widetilde\kappa)^{t+1})}{1-2\widetilde\kappa}+\frac{2\oC^3\kappa(1-(2\widetilde\kappa)^{t+1})}{1-2\widetilde\kappa}}\sqrt{\frac{\log \max\{d,n\}}{n}}+\frac{\kappa}{\sqrt{s}}(2\widetilde{\kappa})^{t+1}\delta_{\bgamma}^{(0)}	\leq\lambda_{n}^{(t+2)}.
\end{align*}
 
Therefore, we have shown that
\begin{equation*}
	\begin{aligned}
&\quad \norm{\bbeta_1^{(t)} - \bbeta_1^*}_2 + \norm{\bbeta_2^{(t)} - \bbeta_2^*}_2 + \norm{\btheta^{(t)} - \btheta^*}_2 \\
& \leq
(2\widetilde\kappa)^t\cdot\paren{\norm{\bbeta_1^{(0)} - \bbeta_1^*}_2 + \norm{\bbeta_2^{(0)} - \bbeta_2^*}_2 + \norm{\btheta^{(0)} - \btheta^*}_2}
+ \frac{2}{\kappa(1-2\widetilde\kappa)}\sqrt{\frac{s\log \max\{d,n\}}{n}}.
	\end{aligned}
\end{equation*}

 Since we focus on the high-dimensional setting where $\log N_{\tau-1}  \lesssim \log d$ for all $\tau$, we replace $\log\max\{d, n\}$ with  $\log d$ from now on. Recall that we use sample splitting $n_{\tau}=N_{\tau-1}/t_{\tau, \max}$ to make sure the solutions obtained in each iteration are independent with each other (which satisfies the assumption of Lemma \ref{thm:lemma-sample}). When $\tau=1$, we have $N_0=n_0$ and  $\delta_{\gamma}^{(0)}\leq \delta_{0}$. Hence,  by taking $t_{1, \max} \asymp \frac{1}{2\log(1/2\tilde{\kappa})}\log (n_0)$, we obtain 
\[
\norm{\hat{\bbeta}_1^{(1)} - \bbeta_1^*}_2 + \norm{\hbbeta_2^{(1)} - \bbeta_2^*}_2 + \norm{\hbtheta^{(1)} - \btheta^*}_2  \lesssim \sqrt{\frac{s\log d \log n_0}{n_0}},
\]
with probability at least $1-ct_{1,\max}^3 / \max\{n_0, d\}^2$.
For $\tau=2$, the initials $\bgamma^{(2, 0)}=\hat{\bgamma}^{(1)}$, and thus $\delta_0^{(2, 0)}\lesssim \sqrt{\frac{s \log d \log n_0}{n_0}}$. By taking $t_{2, \max} \asymp \log(\sqrt{2})/\log(1/(2\tilde{\kappa}))$, which is a constant, we obtain
 \[
 \norm{\hat{\bbeta}_1^{(2)} - \bbeta_1^*}_2 + \norm{\hbbeta_2^{(2)} - \bbeta_2^*}_2 + \norm{\hbtheta^{(2)} - \btheta^*}_2  \lesssim \sqrt{\frac{s\log d \log n_0}{2n_0}} = \sqrt{\frac{s\log d \log n_0}{N_1}}.
 \]
 The same argument shows that, for all $\tau \geq 2$, by taking $t_{\tau, \max} \asymp \log(\sqrt{2})/\log(1/(2\tilde{\kappa}))$,
  \[
 \norm{\hat{\bbeta}_1^{(\tau)} - \bbeta_1^*}_2 + \norm{\hbbeta_2^{(\tau)} - \bbeta_2^*}_2 + \norm{\hbtheta^{(\tau)} - \btheta^*}_2  \lesssim \sqrt{\frac{s\log d \log n_0}{N_{\tau-1}}},
 \]
 with probability at least $1-c t_{\tau, \max}^3 / \max\{N_{\tau-1}, d\}^2$.
  By Lemma \ref{thm:diff-L1-Bound}, we also have that
$$
 	\norm{\hbbeta_1^{(\tau)} - \bbeta_1^*}_1 + \norm{\hbbeta_2^{(\tau)} - \bbeta_2^*}_1 + \norm{\hbtheta^{(\tau)} - \btheta^*}_1
 	\lesssim
 	\sqrt{\frac{s^2\log d \log n_0}{N_{\tau-1}}},
$$ which concludes the proof of Theorem \ref{thm:1-detailed}.



\section{Proof for the Regret Results}
\label{sec:proof-regret}

\subsection{Proof for the Excess Misclassification Rate}\label{sec:proof-misclustering}

\textbf{Proof of Theorem \ref{thm:miss-class-rate}}
The misclassification error for $\btheta^*$ can be expressed by
\begin{align*}
	R(\btheta^*) & = \EE\brackets{\EE\big[\bbone(g_i^* \neq G_{\btheta^*}(\bz_i)) \cond \bz_i\big]} \\
	&=\EE\brackets{\EE\big[\bbone(g_i^* =1, \bz_i^{\top}\btheta^*\leq0) \cond \bz_i\big]+\EE\big[\bbone(g_i^* =2, \bz_i^{\top}\btheta^*>0) \cond \bz_i\big]} \\
	&=\EE\brackets{\bbone(\bz_i^{\top}\btheta^* \leq 0)p(\bz_i^{\top}\btheta^*)+\bbone(\bz_i^{\top}\btheta^* > 0)(1-p(\bz_i^{\top}\btheta^*))} \\
	& = \EE \Big[\min\braces{(1 - p(\bz_i^\top\btheta^*), p(\bz_i^\top\btheta^*)}\Big],
\end{align*}
and 
\begin{align*}
	R\big(\hat\btheta\big) & = \EE\brackets{\EE\Big[\bbone(g_i^* \neq G_{\hat\btheta}(\bz_i)) \cond \bz_i\Big]}\\
	&=\EE\Big[\EE\big[\bbone(g_i^* =1, \bz_i^{\top}\btheta^*\leq0, \bz_i^{\top}\hat{\btheta} \leq 0) \cond \bz_i\big]+\EE\big[\bbone(g_i^* =1, \bz_i^{\top}\btheta^*>0, \bz_i^{\top}\hat{\btheta} \leq 0) \cond \bz_i\big]\\
	&\quad\quad +\EE\big[\bbone(g_i^* =2, \bz_i^{\top}\btheta^*\leq0, \bz_i^{\top}\hat{\btheta} >0) \cond \bz_i\big]+\EE\big[\bbone(g_i^* =2, \bz_i^{\top}\btheta^*>0, \bz_i^{\top}\hat{\btheta} > 0) \cond \bz_i\big]\Big] . 
\end{align*}
Note that
\[\EE\big[\bbone(g_i^* =1, \bz_i^{\top}\btheta^*\leq0, \bz_i^{\top}\hat{\btheta} \leq 0) \cond \bz_i\big] \leq \EE\big[\bbone(g_i^* =1, \bz_i^{\top}\btheta^*\leq0) \cond \bz_i\big] \leq R(\btheta^*),\]
\begin{align*}
	\EE\big[\bbone(g_i^* =1, \bz_i^{\top}\btheta^*>0, \bz_i^{\top}\hat{\btheta} \leq 0) \cond \bz_i\big] & \leq \EE\big[\bbone\left(\abs{\bz_i^{\top}(\hat\btheta-\btheta^*)}\geq \bz_i^{\top}\btheta^*>0\right)p(\bz_i^{\top}\btheta^*) \cond \bz_i\big] \\
	&\leq R(\btheta^*) + \EE\brackets{\bbone\left(\abs{\bz_i^{\top}(\hat\btheta-\btheta^*)}\geq \bz_i^{\top}\btheta^*>0\right)(2p(\bz_i^{\top}\btheta^*)-1)}\\
	&\leq R(\btheta^*) + \EE\brackets{2p\paren{\abs{\bz_i^{\top}(\hat\btheta-\btheta^*)}}-1},
\end{align*}
\begin{align*}
	\EE\big[\bbone(g_i^* =2, \bz_i^{\top}\btheta^*\leq0, \bz_i^{\top}\hat{\btheta} > 0) \cond \bz_i\big] &\leq \EE\big[\bbone\left(\abs{\bz_i^{\top}(\hat\btheta-\btheta^*)} \leq\bz_i^{\top}\btheta^*\leq0\right)(1-p(\bz_i^{\top}\btheta^*)) \cond \bz_i\big]\\
	&\leq R(\btheta^*) + \EE\brackets{\bbone\left(\abs{\bz_i^{\top}(\hat\btheta-\btheta^*)}\leq \bz_i^{\top}\btheta^*<0\right)(1-2p(\bz_i^{\top}\btheta^*))}\\
	&\leq R(\btheta^*) + \EE\brackets{1-2p\paren{\abs{\bz_i^{\top}(\hat\btheta-\btheta^*)}}},
\end{align*}
\[\EE\big[\bbone(g_i^* =2, \bz_i^{\top}\btheta^*>0, \bz_i^{\top}\hat{\btheta} > 0) \cond \bz_i\big] \leq \EE\big[\bbone(g_i^* =2, \bz_i^{\top}\btheta^*>0) \cond \bz_i\big] \leq R(\btheta^*),\]
and only one of the above four indicator functions equals one. Therefore, we obtain that
\[R\big(\hat\btheta\big)-R(\btheta^*) \leq \EE\brackets{\abs{2p\left(\abs{\bz_i^{\top}(\hat\btheta-\btheta^*)}\right)-1}} \leq 2\EE\brackets{p\left(\abs{\bz_i^{\top}(\hat\btheta-\btheta^*)}\right)-p(0)} \leq \frac{1}{2}\EE\brackets{\abs{\bz_i^{\top}(\hat\btheta-\btheta^*)}}.\]
Furthermore, since $\lambda_{\max}(\EE[\bz_i\bz_i^{\top}]) \leq M$, we finally obtain
\[R\big(\hat\btheta\big)-R(\btheta^*) \lesssim \EE\brackets{\abs{\bz_i^{\top}(\hat\btheta-\btheta^*)}} \leq \sqrt{\EE\brackets{(\hat\btheta-\btheta^*)^{\top}\bz_i\bz_i^{\top}(\hat\btheta-\btheta^*)}}.\]
Define a ``good'' event $\cE_i$ as
\begin{equation*}
	\cE_i := \left\{ \norm{\hat{\bbeta}_{1}^{(\tau)} - \bbeta_{1}^*}_{2} + \norm{\hat{\bbeta}_{2}^{(\tau)} - \bbeta_{2}^*}_{2}  + \norm{\hat{\btheta}^{(\tau)} - \btheta^*}_{2}
	\leq C	\sqrt{\frac{s\log d \log n_0}{N_{\tau-1}}}\right\},
\end{equation*}
where $\hat{\bbeta}_{1}^{(\tau)}$, $\hat{\bbeta}_{2}^{(\tau)}$, $\hat{\btheta}^{(\tau)}$ are the estimators obtained using the samples in the $(\tau-1)$-th phase.  By Theorem \ref{thm:1-detailed}, it holds that $\PP(\cE_i^c) \leq c \frac{\log^3 n_0}{\max\{N_{\tau-1}, d\}^2}  \lesssim  \frac{1}{N_{\tau-1}}$ for some constants $c$ and $C$, which yields
\begin{equation*}
	R\left(\hat\btheta^{(\tau)}\right)-R(\btheta^*) \lesssim 1/N_{\tau-1} + \sqrt{\EE\brackets{\big(\hat\btheta^{(\tau)}-\btheta^*\big)^{\top}\bz_i\bz_i^{\top}\big(\hat\btheta^{(\tau)}-\btheta^*\big) \cond \cE_i}} \lesssim \sqrt{\frac{s\log d \log n_0}{N_{\tau-1}}}.
\end{equation*} 

\subsection{Proof for the Regret Upper Bound}
\label{sec:proof-upper}
\textbf{Proof of Theorem \ref{thm:regret}}
Let $\calI_{\tau}$ be the set of indices in the $\tau$-th episode and $N_{\tau}=n_02^{\tau}$ be the cardinality of $\calI_{\tau}$. 
The expected cumulative regret over a length of horizon $T$ can be expressed as
\begin{equation*}
\Reg(T) = \sum_{\tau = 0}^{\tau_{\max}} \sum_{i\in\calI_{\tau}} \EE\brackets{\reg_i},
\end{equation*}
where $\tau_{\max}=[\log_2(T/n_0+1)]-1$.

We will show the following results for the instant regret:
Suppose the conditions in Theorem \ref{thm:regret} hold. For an observation $i$ in the $\tau$-th episode ($\tau \geq 2$), we have
\begin{equation} \label{eqn:E[reg_i]-strong}
\EE[\reg_i^{*}] \lesssim   \frac{\overline{x}^2s^2\log d \log n_0}{N_{\tau-1}} + \overline{x} \norm{\bbeta_2^* - \bbeta_1^*}_1\cdot R(\btheta^*), 
\end{equation}
and
\begin{equation} \label{eqn:E[reg_i]-weak}
\EE[\tilde{\reg}_i] \lesssim \overline{x} \norm{\bbeta_2^*-\bbeta_1^*}_1\sqrt{\frac{s\log d \log n_0}{N_{\tau-1}}}. 
\end{equation}
	
We first deal with the instant strong regret. Let $\left(\hat{\btheta}^{(\tau)}, \hat{\bbeta}_{1}^{(\tau)}, \hat{\bbeta}_{2}^{(\tau)}\right)$ be the estimator in the $\tau$-th episode, which is obtained using the data collected in the $(\tau-1)$-th episode. By the proof of Theorem \ref{thm:coeff-bound-hd}, we have that 
	\[
	\norm{\hat{\bbeta}_{1}^{(\tau)} - \bbeta_{1}^*}_1 + \norm{\hat{\bbeta}_{2}^{(\tau)} - \bbeta_{2}^*}_1 + \norm{\hat{\btheta}^{(\tau)} - \btheta^*}_1
	\leq C
	\sqrt{\frac{s^2\log d \log n_0}{N_{\tau-1}}},
	\]
	with probability at least $1-c \frac{\log^3 n_0}{\max\{N_{\tau-1}, d\}^2}$ for some constant $c, C$.
Define ``good'' events $\cE_i$ and $\cG_i$ as
	\begin{align*}
		\cE_i& := \left\{ \norm{\hat{\bbeta}_{1}^{(\tau)} - \bbeta_{1}^*}_{1} + \norm{\hat{\bbeta}_{2}^{(\tau)} - \bbeta_{2}^*}_{1} + \norm{\hat{\btheta}^{(\tau)} - \btheta^*}_{1}
		\leq C
		\sqrt{\frac{s^2\log d \log n_0}{N_{\tau-1}}}\right\}\\
			\cG_i&:=\left\{\angles{\bx_{i, \tilde{a}_{i, g}}, \bbeta_g^*} > \max_{a\ne \tilde{a}_{i, g}}\angles{\bx_{i, a}, \bbeta^*_{g}} + 2C\overline{x}\sqrt{\frac{s^2\log d  \log n_0}{N_{\tau-1}}}, \text{ for } g= 1, 2 \right\},
	\end{align*}
	where $\tilde{a}_{i, g}:=\underset{a\in [K]}{\arg \max}\; \angles{\bx_{i,a}, \bbeta^*_g}$.  Then it holds that $\PP(\cE_i^c) \leq c \frac{\log^3 n_0}{\max\{N_{\tau-1}, d\}^2}  \lesssim  \frac{1}{N_{\tau-1}}$, which yields
	\begin{equation}\label{eq:Eec-strong}
		\PP(\cE_i^c)\EE[\reg^*_i \cond \cE_i^c] \lesssim \overline{R}/N_{\tau-1}.
	\end{equation}
 By Assumption \ref{B2}, we have that
	\[\angles{\bx_{i, \tilde{a}_{i, g}}, \bbeta_g^*} \ge \max_{a\ne \tilde{a}_{i, g}}\angles{\bx_{i, a}, \bbeta^*_{g}} + 2C\overline{x}\sqrt{\frac{s^2\log d  \log n_0}{N_{\tau-1}}}, \text{ for } g= 1, 2,\]
	hold with probability at least $1-2C_1C\overline{x}\sqrt{\frac{s^2\log d  \log n_0}{N_{\tau-1}}}$. Therefore,
	\[
	\PP(\cG_i^c) \lesssim \overline{x}\sqrt{\frac{s^2\log d  \log n_0}{N_{\tau-1}}}.
	\]

	Since $\norm{\bx_{i,a}}_{\infty} \leq \overline{x}$, under $\cE_i$, we have for $g=1,2$,
	\[\max_{a \in [K]}\angles{\bx_{i,a}, (\bbeta_g^* - \hat \bbeta_g)} \leq \overline{x}\norm{\bbeta_g^* - \hat \bbeta_g}_1 \leq C\overline{x}\sqrt{\frac{s^2\log d  \log n_0}{N_{\tau-1}}}.\]
	If $\cE_i\cap\cG_i$ holds, then
	\[\max\limits_{a\in [K]}\abs{\angles{\bx_{i,a}, \hat \bbeta_{g}} - \angles{\bx_{i,a}, \bbeta^*_{g}}} \leq \frac{1}{2}\left( \angles{\bx_{i, \tilde{a}_{i, g}}, \bbeta_g^*} - \max_{a\ne \tilde{a}_{i, g}}\angles{\bx_{i, a}, \bbeta^*_{g}}\right),\]
	which implies that, for any $a \neq \tilde{a}_{i, g}= \underset{a\in [K]}{\arg \max}\; \angles{\bx_{i,a}, \bbeta^*_{g}}$,
	\begin{equation}
		\begin{aligned}
			\angles{\bx_{i, a}, \hat\bbeta_{g}} & \leq \abs{\angles{\bx_{i,a}, \hat \bbeta_{g}} - \angles{\bx_{i,a}, \bbeta^*_{g}}}  + \angles{\bx_{i, a}, \bbeta^*_{g}}	\\
			&\leq \max\limits_{a\in [K]}\abs{\angles{\bx_{i,a}, \hat \bbeta_{g}} - \angles{\bx_{i,a}, \bbeta^*_{g}}}  +\max_{a\ne \tilde{a}_{i, g}}\angles{\bx_{i, a}, \bbeta^*_{g}}\\
			& \leq \angles{\bx_{i, \tilde{a}_{i, g}}, \bbeta_g^*}-\frac{1}{2}\left( \angles{\bx_{i, \tilde{a}_{i, g}}, \bbeta_g^*} - \max_{a\ne \tilde{a}_{i, g}}\angles{\bx_{i, a}, \bbeta^*_{g}}\right)\\
			& \leq  \angles{\bx_{i, \tilde{a}_{i, g}}, \hat{\bbeta}_g}.
		\end{aligned}
	\end{equation}
	Therefore, we have $\underset{a\in [K]}{\arg \max}\; \angles{\bx_{i,a}, \hat\bbeta_{g}}=\underset{a\in [K]}{\arg \max}\; \angles{\bx_{i,a}, \bbeta^*_{g}}$ for $g=1,2$.

Now we consider two different cases of $g_i$ and $\hat g_i$. When $\hat g_i = g_i$, $\reg^{*}_i = \underset{a\in [K]}{\max}\; \angles{\bx_{i,a}, \bbeta_{g_i}^*} - \angles{\bx_{i,\hat a_i}, \bbeta_{g_i}^*}$. Under $\cE_i\cap\cG_i$, since $\underset{a\in [K]}{\arg \max}\; \angles{\bx_{i,a}, \hat\bbeta_{g}}=\underset{a\in [K]}{\arg \max}\; \angles{\bx_{i,a}, \bbeta^*_{g}}$ for $g=g_i$, we have $\reg^{*}_i=0$. Otherwise
\begin{align*} 
	&\quad \EE(\reg^{*}_i \cond \hat g_i = g_i, \cE_i \cap \cG_i^c) \\
    &= \EE\brackets{\underset{a\in [K]}{\max}\; \angles{\bx_{i,a}, \bbeta_{g_i}^*} - \angles{\bx_{i,\hat a_i}, \bbeta_{g_i}^*} \cond \cE_i \cap \cG_i^c}  \PP( \cE_i \cap \cG_i^c) \\
& = \EE\brackets{ \underset{a\in [K]}{\max}\; \angles{\bx_{i,a}, \bbeta_{g_i}^*} - \underset{a\in [K]}{\max}\;  \angles{\bx_{i,a}, \hat \bbeta_{g_i}} 
+ \angles{\bx_{i,\hat a_i}, \hat\bbeta_{g_i}} - \angles{\bx_{i, \hat a_i}, \bbeta_{g_i}^*} }\PP( \cE_i \cap \cG_i^c)
\\
& \le 2\; \EE\abs{\underset{a\in [K]}{\max}\; \angles{\bx_{i,a}, (\bbeta_{g_i}^* - \hat \bbeta_{g_i})}} \PP( \cE_i \cap \cG_i^c)\\
& \lesssim \frac{\overline{x}^2s^2\log d  \log n_0}{N_{\tau-1}}.
\end{align*}

As a result, we obtain that 
\begin{equation*}
 \EE(\reg^{*}_i \cond \hat g_i =  g_i, \cE_i) \lesssim \frac{\overline{x}^2s^2\log d  \log n_0}{N_{\tau-1}}.
\end{equation*}

If $g_i = 1$, but the algorithm mistakenly clusters it to $\hat g_i = 2$, the greedy policy prescribes
$\hat a_i = \underset{a\in [K]}{\arg \max}\; \angles{\bx_{i,a}, \hat \bbeta_2}$. 
The instant strong regret is
\begin{equation}\label{eq:51}
	\begin{aligned}
		&\quad \EE(\reg^{*}_i \cond \hat g_i = 2, g_i = 1, \cE_i)  =  \EE\brackets{\underset{a\in [K]}{\max}\; \angles{\bx_{i,a}, \bbeta_1^*} - \angles{\bx_{i,\hat a_i}, \bbeta_1^*}}  \\
		& = \EE\Big[ \underset{a\in [K]}{\max}\; \angles{\bx_{i,a}, \bbeta_1^*} - \underset{a\in [K]}{\max}\; \angles{\bx_{i,a}, \bbeta_2^*}
		+ \underset{a\in [K]}{\max}\; \angles{\bx_{i,a}, \bbeta_2^*} \\
		 &\quad \quad -\angles{\bx_{i,\hat a_i}, \bbeta^*_2} +\angles{\bx_{i,\hat a_i},\bbeta^*_2}-\angles{\bx_{i,\hat a_i}, \bbeta^*_1}\Big]
		\\
		& \le \EE\brackets{2\abs{\underset{a\in [K]}{\max}\; \angles{\bx_{i,a}, \bbeta_2^* - \bbeta_1^*}} 
			+  \underset{a\in [K]}{\max}\; \angles{\bx_{i,a}, \bbeta_2^*}-\angles{\bx_{i,\hat a_i}, \bbeta^*_2}} \\
		& \lesssim \overline{x}\norm{\bbeta_2^*-\bbeta_1^*}_1
		+  \frac{\overline{x}^2s^2\log d \log n_0}{N_{\tau-1}}.
	\end{aligned}
\end{equation}

By a similar argument, we have 
\begin{equation}\label{eq:52}
\EE(\reg^{*}_i \cond \hat g_i = 1, g_i = 2,\cE_i) \lesssim \overline{x}\norm{\bbeta_2^*-\bbeta_1^*}_1
+  \frac{\overline{x}^2s^2\log d \log n_0}{N_{\tau-1}}.
\end{equation}
In summary, we have
\begin{align*}
&\quad	\EE[\reg_i^{*}\cond \cE_{i}] \\
& = \EE[\reg^{*}_i \cond \hat g_i = g_i, \cE_i] \cdot \left(1-R\big(\hat\btheta\big)\right) + \EE[\reg^{*}_i \cond \hat g_i \neq g_i, \cE_i] R\left(\hat\btheta\right)\\
& \lesssim \frac{\overline{x}^2s^2\log d  \log n_0}{N_{\tau-1}} + \paren{\overline{x}\norm{\bbeta_2^* - \bbeta_1^*}_1 +    \frac{\overline{x}^2s^2\log d \log n_0}{N_{\tau-1}} } \cdot \paren{R(\btheta^*) + \sqrt{\frac{s\log d \log n_0}{N_{\tau-1}}}} \\
& \lesssim   \frac{\overline{x}^2s^2\log d  \log n_0}{N_{\tau-1}} + \overline{x} \norm{\bbeta_2^* - \bbeta_1^*}_1\cdot R(\btheta^*),
\end{align*}  
where we apply Theorem \ref{thm:miss-class-rate}. Combing \eqref{eq:Eec-strong} with the above inequality leads to \eqref{eqn:E[reg_i]-strong}.

Now we deal with the instant regular regret. Similar to \eqref{eq:Eec-strong}, we have 
	\begin{equation}\label{eq:Eec-regular}
	\EE[\tilde{\reg}_i \cond \cE_i^c] \lesssim \overline{R}/N_{\tau-1}, \quad \EE[\tilde{\reg}_i \cond \cG_i^c] \lesssim \overline{x}\sqrt{\frac{s^2\log d  \log n_0}{N_{\tau-1}}}.
\end{equation}
%

Note that
\begin{equation}\label{eq:Ee}
	\begin{aligned}
		&\quad \EE[\tilde{\reg}_i\cond \cE_{i} \cap \cG_{i}]\\
		&=\EE\brackets{\angles{\bx_{i,\tilde{a}_i}, \bbeta^*_{g_i}} - \angles{\bx_{i,\hat a_i}, \bbeta^*_{g_i}} \cond \cE_{i} \cap \cG_{i} }\\
		& =\EE\brackets{\underset{a\in [K]}{\max}\; \angles{\bx_{i,a}, \bbeta^*_{g_i}} - \angles{\bx_{i,\hat a_i}, \bbeta^*_{g_i}} \cond \cE_{i} \cap \cG_{i} } - \EE\brackets{\underset{a\in [K]}{\max}\; \angles{\bx_{i,a}, \bbeta^*_{g_i}} - \angles{\bx_{i, \tilde{a}_i}, \bbeta^*_{g_i}}\cond \cE_{i} \cap \cG_{i} }\\
		& = \EE\brackets{\underset{a\in [K]}{\max}\; \angles{\bx_{i,a}, \bbeta^*_{g_i}} - \angles{\bx_{i,\hat a_i}, \bbeta^*_{g_i}} \cond g_i=\widehat{g}_i, \cE_{i} \cap \cG_{i}} \left(1-R\big(\hat{\btheta}\big)\right) \\
		&\quad + \EE\brackets{\underset{a\in [K]}{\max}\; \angles{\bx_{i,a}, \bbeta_{g_i}^*} - \angles{\bx_{i,\hat a_i}, \bbeta_{g_i}^*} \cond g_i\neq\widehat{g}_i,  \cE_{i} \cap \cG_{i}} R\big(\hat{\btheta}\big)\\
		&  \quad - \EE\brackets{\underset{a\in [K]}{\max}\; \angles{\bx_{i,a}, \bbeta^*_{g_i}} - \angles{\bx_{i,\tilde a_i}, \bbeta^*_{g_i}} \cond g_i=\tilde{g}_i,  \cE_{i} \cap \cG_{i}} \left(1-R(\btheta^*)\right)\\
		&\quad - \EE\brackets{\underset{a\in [K]}{\max}\; \angles{\bx_{i,a}, \bbeta_{g_i}^*} - \angles{\bx_{i,\tilde a_i}, \bbeta_{g_i}^*} \cond g_i\neq\tilde{g}_i,  \cE_{i} \cap \cG_{i}} R(\btheta^*)
	\end{aligned}
\end{equation}

Since $\underset{a\in [K]}{\arg \max}\; \angles{\bx_{i,a}, \hat\bbeta_{g}}=\underset{a\in [K]}{\arg \max}\; \angles{\bx_{i,a}, \bbeta^*_{g}}$ for $g=1,2$ under $\cE_i \cap \cG_i$, we have $\hat{a}_i = \underset{a\in [K]}{\arg \max}\; \angles{\bx_{i,a}, \hat \bbeta_{g_i}}= \underset{a\in [K]}{\arg \max}\; \angles{\bx_{i,a}, \bbeta^*_{g_i}} =\tilde{a}_i$ when $g_i=\hat{g}_i=\tilde{g}_i$. 
Therefore, 
\begin{align*}
	&\quad \EE\brackets{\underset{a\in [K]}{\max}\; \angles{\bx_{i,a}, \bbeta^*_{g_i}} - \angles{\bx_{i,\hat a_i}, \bbeta^*_{g_i}} \cond g_i=\widehat{g}_i,  \cE_{i} \cap \calG_{i}} \left(1-R\big(\hat{\btheta}\big)\right) \\
	&\quad - \EE\brackets{\underset{a\in [K]}{\max}\; \angles{\bx_{i,a}, \bbeta^*_{g_i}} - \angles{\bx_{i,\tilde a_i}, \bbeta^*_{g_i}} \cond g_i=\tilde{g}_i, \cE_{i} \cap \calG_{i}} \left(1-R(\btheta^*)\right)\\
	& = 0.
\end{align*}
And 
\begin{align*}
	&\quad \EE\brackets{\underset{a\in [K]}{\max}\; \angles{\bx_{i,a}, \bbeta^*_{g_i}} - \angles{\bx_{i,\hat a_i}, \bbeta^*_{g_i}} \cond g_i \neq \widehat{g}_i,  \cE_{i} \cap \calG_{i}} R\big(\hat{\btheta}\big) \\
	&\quad - \EE\brackets{\underset{a\in [K]}{\max}\; \angles{\bx_{i,a}, \bbeta^*_{g_i}} - \angles{\bx_{i,\tilde a_i}, \bbeta^*_{g_i}} \cond g_i\neq \tilde{g}_i, \cE_{i} \cap \cG_i} R(\btheta^*)\\
	& = \EE\brackets{\underset{a\in [K]}{\max}\; \angles{\bx_{i,a}, \bbeta^*_{g_i}} - \angles{\bx_{i,\hat a_i}, \bbeta^*_{g_i}} \cond g_i \neq\widehat{g}_i,  \cE_{i} \cap \calG_{i}} \left(R\big(\hat{\btheta}\big)-R(\btheta^*)\right)\\
	& \lesssim \left(\overline{x}\norm{\bbeta_2^*-\bbeta_1^*}_1
	+  \overline{x}\sqrt{\frac{s^2\log d \log n_0 }{N_{\tau-1}}}\right) \cdot \sqrt{\frac{s\log d \log n_0}{N_{\tau-1}}},
\end{align*}
where the last inequality follows from \eqref{eq:51}, \eqref{eq:52}, and Theorem \ref{thm:miss-class-rate}. Combining \eqref{eq:Eec-regular} and the above two inequalities, we obtain that 
\[ \EE[\tilde{\reg}_i] \lesssim \overline{x}\norm{\bbeta_2^*-\bbeta_1^*}_1\sqrt{\frac{s\log d \log n_0}{N_{\tau-1}}}.\]

Now we return to the cumulative regrets. The regret accumulated in the $\tau$-th phase can be bounded in two different cases. 
\begin{enumerate}[label=(\roman*)]
	\item When $\tau\leq1$, we have $N_{\tau} \leq 2n_0$, then the boundedness of rewards in Assumption \ref{B1} implies that
	\begin{align*}
		& \sum_{i\in\calN_{\tau}} \EE\brackets{\reg^{*}_i} \leq 2\overline{R} N_{\tau} \lesssim \overline{R}n_0, \quad  \sum_{i\in\calN_{\tau}} \EE\brackets{\tilde{\reg}_i} \leq 2\overline{R} N_{\tau} \lesssim \overline{R}n_0. 
	\end{align*}
	\item When $\tau \geq 2$, by \eqref{eqn:E[reg_i]-strong}, the expected strong regret in the $\tau$-th phase satisfies
	\begin{align*}
		\sum_{i\in \calN_{\tau}}\EE[\reg^{*}_i] \lesssim
		\overline{x}^2s^2\log d \log n_0 + \overline{x} \norm{\bbeta_2^* - \bbeta_1^*}_1\cdot R(\btheta^*)N_{\tau-1}.
	\end{align*} 
	By \eqref{eqn:E[reg_i]-weak}, the expected regular regret in the $\tau$-th phase satisfies
	\begin{align*}
		\sum_{i\in \calN_{\tau}}\EE[\tilde{\reg}_i] \lesssim \overline{x} \norm{\bbeta_2^*-\bbeta_1^*}_1\sqrt{s\log d \log n_0 \cdot N_{\tau-1}}.
	\end{align*} 
\end{enumerate}

Hence, the total expected {\em strong} regret is 
\begin{align*}
	\Reg^{*}(T) & = \sum_{\tau = 0}^{\tau_{\max}} \sum_{i\in\calN_{\tau}} \EE\brackets{\reg^{*}_i} \\
	& \lesssim  \overline{R}n_0 +\overline{x}^2s^2\log d \log n_0 \cdot \log T+ \overline{x} \norm{\bbeta_2^* - \bbeta_1^*}_1\cdot R(\btheta^*) \cdot T.
\end{align*}
The total expected regular regret is
\begin{align*}
	\tilde{\Reg}(T) & = \sum_{\tau = 1}^{\tau_{\max}} \sum_{i\in\calN_{\tau}} \EE\brackets{\tilde{\reg}_i}  \\
	& \lesssim \overline{R}n_0 + \sum_{\tau = 2}^{\tau_{\max}}  \overline{x} \norm{\bbeta_2^*-\bbeta_1^*}_1\sqrt{sn_02^{\tau-1}\log d \log n_0 }\\
	& \lesssim \overline{R}n_0+ \overline{x} \norm{\bbeta_2^*-\bbeta_1^*}_1\sqrt{s\log d \log n_0 } \sqrt{T}.
\end{align*}

\subsection{Proof for the Regret Lower Bound}\label{sec:proof-lower}
\textbf{Proof of Theorem \ref{thm:lower-bound}}
	We first show the lower bound for the instant strong regret. Given constants $\overline{L}>0$ and $\overline{x}>0$, let $\bbeta_{1}^{*}=(\overline{L}, 0, \dots, 0)$, $\bbeta_2^*=(-\overline{L}, 0, \dots, 0)$, and the $j$-th entry of $\bx_{i, a_i}$ be $x_{ij} + \frac{\overline{x}}{2}(3- 2a_i)$, where $x_{ij} \stackrel{i.i.d.}{\sim} \calU[-\overline{x}/2, \overline{x}/2]$ for $j=1,2,\dots, d$, and $a_i \in \left\{1, 2\right\}$. For simplicity, we denote $(\bx_{i, 1}, \bx_{i, 2})$ by $\bx_i$ for any $i$. The parameter $\btheta^* \in \RR^d$ and the distribution of $\bz_i$ can be chosen arbitrarily as long as they satisfy $\norm{\btheta^*}_0 \leq s$, \ref{A1}, \ref{A4}, and $\bz_i$ is independent of $x_{i1}$. Then it is straightforward to verify that this choice of $\mu(\bx, y, \bz;\bgamma^*)$ belongs to $ \calP_{d,s,\overline{x},\overline{L}}$. 
	
	We have that 
	\[\angles{\bx_{i,a_i}, \bbeta^*_{g_i}} =\overline{L} \left(x_{i1}+\frac{\overline{x}}{2}(3- 2a_i)\right)(3-2g_i)=\overline{L}x_{i1}(3-2g_i)+\overline{L}\overline{x}(3 - 2a_i)(3-2g_i)/2,\] and hence 
	\begin{equation}\label{eq:max-regret-lower}
		\max_{a_i \in [2]}\angles{\bx_{i,a_i}, \bbeta^*_{g_i}} =\overline{L}x_{i1}(3-2g_i)+\overline{L}\overline{x}/2.
	\end{equation}
	
	On the other hand, let $\hat\pi(a_i \mid \bx_i, \bz_i, \calH_{i-1})$ denote a policy for choosing $a_i$, i.e., a conditional distribution of $a_i$ given the present features $(\bx_i, \bz_i)$ and the past history $\calH_{i-1}:=(\bx_{i-1}, \bz_{i-1}, y_{i-1}, \dots, \bx_{1}, \bz_{1}, y_{1})$. Let $\hat\pi_1:=\hat\pi(a_i=1 \mid \bx_i, \bz_i, \calH_{i-1})$ and $p_1:=\PP(g_i=1 \mid \bx_i, \bz_i, \calH_{i-1})=p(\bz_i^{\top}\btheta^*)$. Note that given $\bz_i$, the group $g_i$ is independent of the action $\widehat{a}_i$ under $\pi$. Then the conditional expected reward with $\widehat{a}_i \sim \hat\pi(a_i \mid \bx_i, \bz_i, \calH_{i-1})$ can be written as
	\begin{equation}\label{eq:pi-hat-regret}
	\begin{aligned}
			&\quad \EE_{\hat\pi}\left[\angles{\bx_{i,\widehat{a}_i}, \bbeta^*_{g_i}} \mid  \bx_i, \bz_i, \calH_{i-1} \right]\\
			&=\overline{L}\Big\{p_1\left[(x_{i1}+\overline{x}/2)\hat\pi_1+(x_{i1}-\overline{x}/2)(1-\hat\pi_1)\right]+(1-p_1)\left[(-x_{i1}-\overline{x}/2)\hat\pi_1+(\overline{x}/2-x_{i1})(1-\hat\pi_1)\right]\Big\},
	\end{aligned}
	\end{equation}
	and hence, for all $\hat\pi(a_i \mid \bx_i, \bz_i, \calH_{i-1})$,
	\begin{align*}
	 \EE_{\hat\pi}\left[\angles{\bx_{i,\widehat{a}_i}, \bbeta^*_{g_i}} \mid  \bx_i, \bz_i, \calH_{i-1} \right] &\leq \overline{L}(2p_1-1) \left(x_{i1}+\overline{x}\sgn[2p_1-1]/2\right)\\
		&=\overline{L}(2p(\bz_i^{\top}\btheta^*)-1)x_{i1}+\frac{\overline{x}\overline{L}}{2}\abs{2p(\bz_i^{\top}\btheta^*)-1}.
	\end{align*}
	Note that \eqref{eq:max-regret-lower} implies
	\begin{equation}\label{eq:expected-max-regret-lower}
		\begin{aligned}
				\EE\left[\max_{a_i \in [2]}\angles{\bx_{i,a_i}, \bbeta^*_{g_i}} \mid \bx_i, \bz_i, \calH_{i-1}\right] 
				&= p_1(\overline{L}x_{i1}+\overline{L}\overline{x}/2)+(1-p_1)(-\overline{L}x_{i1}+\overline{L}\overline{x}/2)\\
				&=\overline{L}(2p(\bz_i^{\top}\btheta^*)-1)x_{i1}+\overline{L}\overline{x}/2.
		\end{aligned}
	\end{equation}
	As  $R(\btheta^*)=\EE \Big[\min\braces{(1 - p(\bz_i^\top\btheta^*), p(\bz_i^\top\btheta^*)}\Big]=\frac{1}{2}\EE\left[1-\abs{2p(\bz_i^{\top}\btheta^*)-1}\right]$, we obtain, for any policy $\hat\pi$, 
	\[\sup_{\mu \in \calP_{s,d,\overline{x},\overline{L}}}\EE_{\hat\pi}[\reg^*_{i}] \geq \frac{\overline{x}\overline{L}}{2}\EE\left[1-\abs{2p(\bz_i^{\top}\btheta^*)-1}\right] \gtrsim \overline{x}\overline{L}R(\btheta^*).\]
	Hence, the cumulative regret 
		\[\inf_{\hat\pi}\sup_{\mu \in \calP_{s,d,\overline{x},\overline{L}}}\sum_{i=1}^{T}\EE_{\hat\pi}[\reg^*_{i}]\gtrsim \overline{x}\overline{L}R(\btheta^*)T.\]
		
	We then show the lower bound for the instant regular regret, where we first introduce the following lemma on the lower bound of the excess misclassification rate for the sparse logistic model:
	\begin{lemma}[\cite{abramovich2018high}, Section \uppercase\expandafter{\romannumeral 6\relax}]\label{lem:risk-lower-bound}
		Define a sparse logistic model $(y, \bz) \sim \calL_{\btheta^*}$ as $y  \sim \mathrm{Bernoulli}(p)$ with $p=\frac{\exp(\bz^{\top}\btheta^*)}{1+\exp(\bz^{\top}\btheta^*)}$, where $\btheta^*\in \RR^d$ and $\norm{\btheta^*}_0 \leq s$. Then we have
		\[\inf_{\hat\eta}\sup_{\norm{\btheta^*}_0 \leq s} \big[\EE_{\{(y_i, \bz_i)\}_{i=1}^n \sim \calL_{\btheta^*}}[R_{\btheta^*}(\hat\eta)]-R_{\btheta^*}(\eta^*)\big] \gtrsim \sqrt{\frac{s\log(d/s)}{n}},\]
		where $R_{\btheta^*}(\eta):=\EE_{(y, \bz) \in \calL_{\btheta^*}}[\one(\eta(\bz) \neq y)]$, $\eta^*(\bz)=\one(\bz^{\top}\btheta^*>0)$, and the infimum is taken over all classifiers $\hat\eta: \RR^d \to \{0, 1\}$ learned from random samples $\{(y_i, \bz_i)\}_{i=1}^n$.
	\end{lemma}

Let $\calH_{i-1, \bz}:=\{\bz_{1}, \bz_{2}, \dots, \bz_{i-1}\}$. 
By \eqref{eq:pi-hat-regret} and \eqref{eq:expected-max-regret-lower}, we have that 
\begin{align*}
	&\quad \EE\left[\max_{a_i \in [2]}\angles{\bx_{i,a_i}, \bbeta^*_{g_i}} \mid \bz_i,  \calH_{i-1, \bz}\right] - \EE_{\hat\pi}\left[\angles{\bx_{i,\widehat{a}_i}, \bbeta^*_{g_i}} \mid  \bz_i,  \calH_{i-1, \bz}\right] \\
		&=\EE_{\bx}\left[\EE\left[\max_{a_i \in [2]}\angles{\bx_{i,a_i}, \bbeta^*_{g_i}} \mid \bx_i, \bz_i, \calH_{i-1}\right]\right] - \EE_{\bx}\left[\EE_{\hat\pi}\left[\angles{\bx_{i,\widehat{a}_i}, \bbeta^*_{g_i}} \mid  \bx_i, \bz_i, \calH_{i-1} \right]\right] \\
		&= \overline{x}\overline{L}\Big[p_1(1-\EE_{\bx}[\hat\pi_1])+(1-p_1)\EE_{\bx}[\hat\pi_1]\Big],
\end{align*}
where $\EE_{\bx}$ is taken over $\bx_1,\dots, \bx_{i}$. In particular,
\[	\EE\left[\max_{a_i \in [2]}\angles{\bx_{i,a_i}, \bbeta^*_{g_i}} \mid   \bz_i, \calH_{i-1,\bz}\right] - \EE\left[\angles{\bx_{i,\widetilde{a}_i}, \bbeta^*_{g_i}} \mid \bz_i, \calH_{i-1,\bz} \right] = \overline{x}\overline{L}[p_1(1-\widetilde\pi_1)+(1-p_1)\widetilde\pi_1],\]
where $\widetilde\pi_1=\one(\bz_i^{\top}\btheta^* \geq 0)$. Note that $\EE_{\bx}[\hat\pi_1]=\EE_{\bx}[\hat\pi(a_i=1 \mid \bx_i, \bz_i, \calH_{i-1})]$ can be viewed a function of $\bz_i$ that is learned based on $\calH_{i-1, \bz}$, and thus we can correspondingly define an estimated classifier $\hat\eta_{\hat\pi}$ such that $\hat\eta_{\hat\pi}=1$ with probability $\EE_{\bx}[ \hat\pi_1]$ and $\hat\eta_{\hat\pi}=2$ with probability $1-\EE_{\bx}[\hat\pi_1]$. Then, using $\EE_{\bz}$ to denote the expectation taken over $\bz_1,\dots, \bz_{i}$, we have $\EE_{\bz}\big[p_1(1-\EE_{\bx}[\hat\pi_1])+(1-p_1)\EE_{\bx}[\hat\pi_1]\big]=\EE_{\bz}[R_{\btheta^*}(\hat\eta_{\hat\pi})]$ and $\EE_{\bz}[p_1(1-\widetilde\pi_1)+(1-p_1)\widetilde\pi_1]=R_{\btheta^*}(\eta^*)$, where $\eta^*(\bz)$ is the classifier such that $g_i=1$ if $\bz^{\top}\btheta^*\geq 0$ and $g_i=2$ otherwise. Therefore, for any policy $\hat\pi$, 
\begin{align*}
	&\quad \EE \left[\angles{\bx_{i, \widetilde{a}_i}, \bbeta^*_{g_i}}\right] - \EE_{\hat\pi} \left[\angles{\bx_{i, \widehat{a}_i}, \bbeta^*_{g_i}}\right]\\
	&=\EE_{\bz}\bigg[\EE\left[\max_{a_i \in [2]}\angles{\bx_{i,a_i}, \bbeta^*_{g_i}} \mid \bz_i,  \calH_{i-1, \bz}\right] - \EE_{\hat\pi}\left[\angles{\bx_{i,\widehat{a}_i}, \bbeta^*_{g_i}} \mid  \bz_i,  \calH_{i-1, \bz}\right]\bigg]\\
	&\quad  - \EE_{\bz}\bigg[\EE\left[\max_{a_i \in [2]}\angles{\bx_{i,a_i}, \bbeta^*_{g_i}} \mid   \bz_i, \calH_{i-1,\bz}\right] - \EE\left[\angles{\bx_{i,\widetilde{a}_i}, \bbeta^*_{g_i}} \mid \bz_i, \calH_{i-1,\bz} \right] \bigg]\\
	&=\overline{x}\overline{L}[\EE_{\bz}[R_{\btheta^*}(\hat\eta_{\hat\pi})]-R_{\btheta^*}(\eta^*)].
\end{align*}
By Lemma \ref{lem:risk-lower-bound}, we obtain that
\[ \inf_{\hat\pi}\sup_{\mu \in \calP_{d,s,\overline{x},\overline{L}}}\EE \left[\angles{\bx_{i, \widetilde{a}_i}, \bbeta^*_{g_i}}\right] - \EE_{\hat\pi} \left[\angles{\bx_{i, \widehat{a}_i}, \bbeta^*_{g_i}}\right] \gtrsim \overline{x}\overline{L}\sqrt{\frac{s\log d}{i-1}}.\]
Hence, for $n_0\gtrsim s\log d$, the cumulative regular regret 
\[ \inf_{\hat\pi}\sup_{\mu \in \calP_{d,s,\overline{x},\overline{L}}}\sum_{i=n_0}^{T}\bigg[\EE \left[\angles{\bx_{i, \widetilde{a}_i}, \bbeta^*_{g_i}}\right] - \EE_{\hat\pi} \left[\angles{\bx_{i, \widehat{a}_i}, \bbeta^*_{g_i}}\right]\bigg] \gtrsim \overline{x}\norm{\bbeta_1^*-\bbeta_2^*}_1\sqrt{\frac{s\log d}{i-1}}\gtrsim \overline{x}\overline{L}\sqrt{sT\log d},\]
where we use the fact that $\sqrt{n}\leq \sum_{i=1}^{n}\frac{1}{\sqrt{i}} \leq 2\sqrt{n}$.

\section{Proof for the Technical Lemmas}

\subsection{Proof of Lemma \ref{thm:lemma-expectation}} \label{sec:proof-lemma-expectation}

Recall that $\bgamma=(\btheta, \bbeta_1, \bbeta_2)$, $\bgamma^*=(\btheta^*, \bbeta_1^*, \bbeta_2^*)$, $\bgamma^{(t)}=\left(\btheta^{(t)}, \bbeta_1^{(t)}, \bbeta_2^{(t)}\right)$,

\begin{equation}
	\omega(\bx, y, \bz ; \bgamma) 
	= \frac{p(\bz^\top\btheta) \cdot
		\phi\paren{\frac{y - \bx^\top\bbeta_1}{\sigma}}}{p(\bz^\top\btheta) \cdot \phi\paren{\frac{y - \bx^\top\bbeta_1}{\sigma}} 
		+ \paren{1-p(\bz^\top\btheta)} \cdot \phi\paren{\frac{y - \bx^\top\bbeta_2}{\sigma}}},
\end{equation}
$\omega_i^*=\omega(\bx_i, y_i, \bz_i; \bgamma^*)$, and $\omega_i^{(t)}=\omega(\bx_i, y_i, \bz_i; \bgamma^{(t)})$.
We calculate the partial derivatives of $\omega(\bx, y, \bz ; \bgamma)$ with respect to each component of $\bgamma$ as
\begin{equation}\label{eq:partial}
	\begin{aligned}
		&\quad \frac{\partial \omega}{\partial \btheta} \\
		& =\frac{p(\bz^\top\btheta)(1-p(\bz^\top\btheta))\bz}{\left[p(\bz^\top\btheta)+(1-p(\bz^\top\btheta))\exp\left(\frac{(y-\bx^{\top}\bbeta_1)^2-(y-\bx^{\top}\bbeta_2)^2}{2\sigma^2}\right)\right]\left[p(\bz^\top\btheta)\exp\left(\frac{(y-\bx^{\top}\bbeta_2)^2-(y-\bx^{\top}\bbeta_1)^2}{2\sigma^2}\right)+1-p(\bz^\top\btheta)\right]}, \\
		&\quad \frac{\partial \omega}{\partial \bbeta_1} \\
		& = \frac{p(\bz^\top\btheta)(1-p(\bz^\top\btheta))(y-\bx^\top\bbeta_1)\bx/\sigma^2}{\left[p(\bz^\top\btheta)+(1-p(\bz^\top\btheta))\exp\left(\frac{(y-\bx^{\top}\bbeta_1)^2-(y-\bx^{\top}\bbeta_2)^2}{2\sigma^2}\right)\right]\left[p(\bz^\top\btheta)\exp\left(\frac{(y-\bx^{\top}\bbeta_2)^2-(y-\bx^{\top}\bbeta_1)^2}{2\sigma^2}\right)+1-p(\bz^\top\btheta)\right]},\\
		&\quad \frac{\partial \omega}{\partial \bbeta_2} \\
		& = \frac{-p(\bz^\top\btheta)(1-p(\bz^\top\btheta))(y-\bx^\top\bbeta_2)\bx/\sigma^2}{\left[p(\bz^\top\btheta)+(1-p(\bz^\top\btheta))\exp\left(\frac{(y-\bx^{\top}\bbeta_1)^2-(y-\bx^{\top}\bbeta_2)^2}{2\sigma^2}\right)\right]\left[p(\bz^\top\btheta)\exp\left(\frac{(y-\bx^{\top}\bbeta_2)^2-(y-\bx^{\top}\bbeta_1)^2}{2\sigma^2}\right)+1-p(\bz^\top\btheta)\right]}.
	\end{aligned}
\end{equation}

We note that the three inequalities in Lemma \ref{thm:lemma-expectation} can be shown following the same procedure. Therefore, we only show the proof of the first inequality, which is also the most difficult one.
Let  $\bdelta_{\bgamma}^{(t)}=\bgamma^{(t)}-\bgamma^*$ and $\bgamma_u^{(t)}=\bgamma^*+u\bdelta_{\bgamma}^{(t)}$ for $u\in[0,1]$, then
\begin{align*}
& \quad \EE\left[\omega_i^{(t)} \bx_i ( \bx_i^\top \bbeta_1^* - y_i)\right] - \EE\left[\omega_i^* \bx_i ( \bx_i^\top \bbeta_1^* - y_i)\right]= \EE\brackets{\int_0^1\angles{\frac{\partial\omega}{\partial \bgamma} \big|_{\bgamma_u^{(t)}}  \cdot \bx^{\top} ( \bx^\top \bbeta_1^* - y), \bdelta_{\bgamma}^{(t)}} \mathrm{d} u}.
\end{align*}
It suffices to show that, for any constant $\kappa>0$, when $\delta_{t}:=\norm{\bbeta_1^{(t)} - \bbeta_1^*}_2 + \norm{\bbeta_2^{(t)} - \bbeta_2^*}_2 + \norm{\btheta^{(t)} - \btheta^*}_2$ is sufficiently small and $c_{\mathrm{SNR}}:=\norm{\bbeta_1^*-\bbeta_2^*}_2 / \sigma$ is sufficiently large, we have
\begin{equation}\label{eq:partial-bound}
	\sup_{u \in [0, 1]}\norm{\EE\left[\frac{\partial\omega}{\partial \bbeta_g}\big|_{\bgamma_{u}^{(t)}}\bx^{\top} ( \bx^\top \bbeta_1^* - y)\right]}_2 \leq \kappa \text{ for } g=1, 2, \text{ and } \sup_{u \in [0, 1]}\norm{\EE\left[\frac{\partial\omega}{\partial \btheta}\big|_{\bgamma_{u}^{(t)}}\bx^{\top} ( \bx^\top \bbeta_1^* - y)\right]}_2 \leq \kappa.
\end{equation}

We first show that $\sup\limits_{u \in [0, 1]}\norm{\EE\left[\frac{\partial\omega}{\partial \bbeta_1}\big|_{\bgamma_{u}^{(t)}}\bx^{\top} ( \bx^\top \bbeta_1^* - y)\right]}_2 \leq \kappa$, i.e., for any $\bv \in \RR^d$ such that $\norm{\bv}_2=1$, we will show that
\[\EE\left[\bv^{\top}\frac{\partial\omega}{\partial \bbeta_1}\big|_{\bgamma_{u}^{(t)}}(\bx^{\top}\bv) ( \bx^\top \bbeta_1^* - y)\right] \leq \kappa.\] 
In the sequel, we omit the subscript $u$, i.e., we use $(\bbeta_1, \bbeta_2, \btheta)$ to denote an arbitrary parameter between $(\bbeta_1^*, \bbeta_2^*, \btheta^*)$ and $\big(\bbeta_1^{(t)}, \bbeta_2^{(t)}, \btheta^{(t)}\big)$. Define 
\[\cE_1= \braces{\abs{\bz^{\top}(\btheta-\btheta^*)}<\mu_1\norm{\btheta-\btheta^*}_2, \abs{\bx^{\top}\bv}\leq \mu_1,  \abs{\bx^{\top}(\bbeta_g-\bbeta_g^*)} \leq \mu_1\norm{\bbeta_g-\bbeta_g^*}_2, g=1, 2},\]
where $\mu_1$ is a constant to be determined. By the sub-Gaussianity of $\bx, \bz$, we have that $\PP\paren{\cE_1^c} < 2e^{-\mu_1^2/(2\sigma_z^2)}+6e^{-\mu_1^2/(2\sigma_x^2)}$. 

Under $\cE_1$, note that $p(\bz^{\top}\btheta)(1-p(\bz^{\top}\btheta)) \leq \frac{1}{4}$ and $\abs{p(\bz^{\top}\btheta)-p(\bz^{\top}\btheta^*)} \leq \frac{1}{4}\abs{\bz^{\top}(\btheta-\btheta^*)} \leq  \frac{1}{4}\mu_1\norm{\btheta-\btheta^*}_2<\xi/2$ if $\norm{\btheta-\btheta^*}_2 \leq  \delta_{0}$ for $\delta_{0} \leq 2\xi/\mu_1$. Then we obtain that $\xi/2 <p(\bz^{\top}\btheta) < 1-\xi/2$, and thus \[\left[p(\bz^\top\btheta)+(1-p(\bz^\top\btheta))\exp\left(\frac{(y-\bx^{\top}\bbeta_1)^2-(y-\bx^{\top}\bbeta_2)^2}{2\sigma^2}\right)\right] \geq \frac{\xi}{2}\brackets{1+ \exp\left(\frac{(y-\bx^{\top}\bbeta_1)^2-(y-\bx^{\top}\bbeta_2)^2}{2\sigma^2}\right)},\]
	\[\left[p(\bz^\top\btheta)\exp\left(\frac{(y-\bx^{\top}\bbeta_2)^2-(y-\bx^{\top}\bbeta_1)^2}{2\sigma^2}\right)+1-p(\bz^\top\btheta)\right]\geq \frac{\xi}{2}\brackets{1+ \exp\left(\frac{(y-\bx^{\top}\bbeta_2)^2-(y-\bx^{\top}\bbeta_1)^2}{2\sigma^2}\right)},\]
which implies that
\begin{equation}
	\label{eq:E-partial-oemga-1}
	\norm{\EE\left[\bv^{\top}\frac{\partial\omega}{\partial \bbeta_1}(\bx^{\top}\bv) ( \bx^\top \bbeta_1^* - y) \big| \cE_1 \right]}_2 \leq \frac{\mu_1^2}{\xi^2}\EE\brackets{\frac{\abs{y-\bx^{\top}\bbeta_1}\abs{y-\bx^{\top}\bbeta_1^*}/\sigma^2}{\exp\paren{\abs{\frac{(y-\bx^{\top}\bbeta_1)^2-(y-\bx^{\top}\bbeta_2)^2}{2\sigma^2}}}} \,\big|\, \cE_1} .
	\end{equation}

Furthermore, we define $\cE_2=\left\{\abs{\bx^{\top}(\bbeta_1^*-\bbeta_2^*)}\geq \mu_2\norm{\bbeta_1^*-\bbeta_2^*}_2\right\}$, where $\mu_2$ is a constant to be determined. Since $\bx^{\top}(\bbeta_1^*-\bbeta_2^*)/\norm{\bbeta_1^*-\bbeta_2^*}_2$ has a bounded density around 0, we have $\PP(\cE_2^c) \lesssim \mu_2$ when $\mu_2$ is sufficiently small.

When $g=2$, i.e., $y=\bx^{\top}\bbeta_2^*+\varepsilon$, the numerator in \eqref{eq:E-partial-oemga-1}

\begin{equation}\label{eq:60}
	\begin{aligned}
		\abs{y-\bx^{\top}\bbeta_1}\abs{y-\bx^{\top}\bbeta_1^*}/\sigma^2 &\leq \frac{1}{\sigma^2} \left[\frac{1}{2}(y-\bx^{\top}\bbeta_1)^2+\frac{1}{2}(y-\bx^{\top}\bbeta_1^*)^2\right]\\
		& \leq\frac{1}{\sigma^2} \left[3(\bx^{\top}(\bbeta_1^*-\bbeta_2^*))^2+3\varepsilon^2+(\bx^{\top}(\bbeta_1^*-\bbeta_1))^2\right],
	\end{aligned}
\end{equation}
and the denominator
\begin{equation}
	\begin{aligned}
		&\quad \exp\paren{\abs{\frac{(y-\bx^{\top}\bbeta_1)^2-(y-\bx^{\top}\bbeta_2)^2}{2\sigma^2}}} \\
		&\geq  	\exp\paren{\frac{1}{4\sigma^2}(\bx^{\top}(\bbeta_1^*-\bbeta_2^*))^2-\frac{1}{2\sigma^2}\left((\bx^{\top}(\bbeta_1^*-\bbeta_1))^2+(\bx^{\top}(\bbeta_2^*-\bbeta_2))^2\right)-\frac{1}{\sigma^2}\varepsilon\abs{\bx^{\top}(\bbeta_1-\bbeta_2)}}\\
		& \geq \exp\paren{\frac{1}{8\sigma^2}(\bx^{\top}(\bbeta_1^*-\bbeta_2^*))^2-\frac{5}{8\sigma^2}\left((\bx^{\top}(\bbeta_1^*-\bbeta_1))^2+(\bx^{\top}(\bbeta_2^*-\bbeta_2))^2\right)-\frac{6}{\sigma^2}\varepsilon^2}.
	\end{aligned}
\end{equation}
Therefore, under $\cE_1$,
\begin{equation}\label{eq:Eeps-E1}
	\begin{aligned}
		&\quad \EE_{\varepsilon}\brackets{\frac{\abs{y-\bx^{\top}\bbeta_1}\abs{y-\bx^{\top}\bbeta_1^*}/\sigma^2}{\exp\paren{\abs{\frac{(y-\bx^{\top}\bbeta_1)^2-(y-\bx^{\top}\bbeta_2)^2}{2\sigma^2}}}} \,\big|\, \cE_1, g=2}  \\
		&\leq \EE_{\varepsilon}\left[ \frac{1}{\sigma^2} \left[3(\bx^{\top}(\bbeta_1^*-\bbeta_2^*))^2+3\varepsilon^2+\mu_1^2\delta_{t}^2\right]\exp\paren{\frac{6\varepsilon^2+\mu_1^2\delta_t^2}{\sigma^2}}\exp\paren{-\frac{1}{8\sigma^2}(\bx^{\top}(\bbeta_1^*-\bbeta_2^*))^2}  \mid \cE_1, g=2\right]\\
		&\lesssim  \frac{1}{\sigma^2}(\bx^{\top}(\bbeta_1^*-\bbeta_2^*))^2\exp\paren{-\frac{1}{8\sigma^2}(\bx^{\top}(\bbeta_1^*-\bbeta_2^*))^2+\frac{2\mu_1^2\delta_t^2}{\sigma^2}},
	\end{aligned}
\end{equation}
where we use the fact that $\varepsilon$ is sub-Gaussian. Under $\cE_2$, we have $(\bx^{\top}(\bbeta_1^*-\bbeta_2^*))^2/\sigma^2 \geq \mu_2^2c_{\mathrm{SNR}}^2$, then there exist constant $c_1$, $c_2$ such that, if $\delta_t \leq c_1 \min\{\xi, \norm{\bbeta_1^*-\bbeta_2^*}_2\}$ and $c_{\rm{SNR}}  \geq c_2$, it holds that $\EE\brackets{\frac{\abs{y-\bx^{\top}\bbeta_1}\abs{y-\bx^{\top}\bbeta_1^*}/\sigma^2}{\exp\paren{\abs{\frac{(y-\bx^{\top}\bbeta_1)^2-(y-\bx^{\top}\bbeta_2)^2}{2\sigma^2}}}} \,\big|\, \cE_1 \cap \cE_2, g=2} \leq \kappa/4$.

On $\cE_1 \cap \cE_2^c$, we upper bound \eqref{eq:Eeps-E1} using the fact that $x^2 e^{-x^2/8} \leq 8/e$, which leads to that, for some constant $C$, \[\EE\brackets{\frac{\abs{y-\bx^{\top}\bbeta_1}\abs{y-\bx^{\top}\bbeta_1^*}/\sigma^2}{\exp\paren{\abs{\frac{(y-\bx^{\top}\bbeta_1)^2-(y-\bx^{\top}\bbeta_2)^2}{2\sigma^2}}}} \,\big|\, \cE_1, g=2} \leq \frac{\kappa}{4} +C\exp\left(\frac{2\mu_1^2c_1^2\xi^2}{\sigma^2}\right) \mu_2  \leq \frac{\kappa}{2},\]
by picking $\mu_2$ such that $C\exp\left(\frac{2\mu_1^2c_1^2}{\sigma^2}\right) \mu_2 \leq \kappa/4$.

On $\cE_1^c$, we apply the fact that 
\[\frac{p(\bz^\top\btheta)(1-p(\bz^\top\btheta))}{\left[p(\bz^\top\btheta)+(1-p(\bz^\top\btheta))\exp\left(\frac{(y-\bx^{\top}\bbeta_1)^2-(y-\bx^{\top}\bbeta_2)^2}{2\sigma^2}\right)\right]\left[p(\bz^\top\btheta)\exp\left(\frac{(y-\bx^{\top}\bbeta_2)^2-(y-\bx^{\top}\bbeta_1)^2}{2\sigma^2}\right)+1-p(\bz^\top\btheta)\right]} \leq \frac{1}{4},\]
and the result above that
\begin{equation*}
	\begin{aligned}
		\EE  \left[(y-\bx^\top\bbeta_1)(y-\bx^\top\bbeta_1^*)(\bx^{\top}\bv)/\sigma^2\right]& \lesssim \EE \left\{\frac{(\bx^{\top}\bv)}{\sigma^2} \left[(\bx^{\top}(\bbeta_1^*-\bbeta_2^*))^2+\varepsilon^2+(\bx^{\top}(\bbeta_1^*-\bbeta_1))^2\right]\right\} \\
		&\lesssim 1+c_{\rm{SNR}}^2,
	\end{aligned}
\end{equation*}
where the second inequality follows from the assumptions that $\EE[\bx\bx^{\top}]$ has bounded eigenvalues and $\delta_t \leq c_1 \norm{\bbeta_1^*-\bbeta_2^*}_2$. Therefore,
\[\EE\brackets{\bv^{\top}\frac{\partial\omega}{\partial \bbeta_1}(\bx^{\top}\bv) ( \bx^\top \bbeta_1^* - y)\,\big|\,g=2}  \lesssim \frac{\kappa}{2}+ \left(1+c_{\rm{SNR}}^2\right)\left(e^{-\mu_1^2/(2\sigma_z^2)}+e^{-\mu_1^2/(2\sigma_x^2)}\right) \leq \kappa,\]
by picking $\mu_1$ such that $\left(1+c_{\rm{SNR}}^2\right)\left(e^{-\mu_1^2/(2\sigma_z^2)}+e^{-\mu_1^2/(2\sigma_x^2)}\right)  \leq \kappa/2$.

For the case for $g=1$, i.e.,  $y=\bx^{\top}\bbeta_1^*+\varepsilon$, note that the numerator in \eqref{eq:E-partial-oemga-1}
\begin{equation}
	\begin{aligned}
		\abs{y-\bx^{\top}\bbeta_1}\abs{y-\bx^{\top}\bbeta_1^*}/\sigma^2 &\leq \frac{1}{\sigma^2} \left[\frac{1}{2}(\bx^{\top}\bbeta_1^*-\bx^{\top}\bbeta_1)^2+\frac{3}{2}\varepsilon^2\right],
	\end{aligned}
\end{equation}
which is less than the bound in \eqref{eq:60}, and then the following procedure is the same as that for $g=2$. The other two inequalities in \eqref{eq:partial-bound} can be shown in the same way.

\subsection{Proof of Lemma \ref{thm:lemma-sample}} \label{sec:proof-lemma-sample}

%
%
	We first prove the first concentration inequality in Lemma \ref{thm:lemma-sample}: For some constant $C$, with probability at least $1-\frac{2}{\max\{n, d\}^2}$,
	\begin{equation*}
		 \norm{ \frac{1}{n}\sum_{i=1}^n \omega_i(\bgamma^{(t)}) \bx_i ( \bx_i^\top \bbeta_1^* - y_i) 
			- \EE[ \omega_i(\bgamma^{(t)}) \bx_i ( \bx_i^\top \bbeta_1^* - y_i) ] }_\infty \le  C \sqrt{ \frac{\log \max\{n, d\}}{n}},
	\end{equation*}
	which is equivalent to 
	\begin{equation*}
		\max_{j \in [d]} \abs{ \frac{1}{n}\sum_{i=1}^n \omega_i(\bgamma^{(t)}) x_{ij} ( \bx_i^\top \bbeta_1^* - y_i) 
			- \EE[ \omega_i(\bgamma^{(t)}) x_{ij} ( \bx_i^\top \bbeta_1^* - y_i) ] } \le  C \sqrt{ \frac{\log \max\{n, d\}}{n} } .
	\end{equation*}
	Let $\ba_{i} = (\bx_i, y_i, \bz_i)$ and $f(\ba_{i})=\omega(\bgamma^{(t)};\bx_i, y_i, \bz_i) x_{ij} ( \bx_i^\top \bbeta_1^* - y_i)$. 
	Note that $\bgamma^{(t)}$ is independent of $(\bx_i, y_i, \bz_i)$, and thus, given $\bgamma^{(t)}$, the $\{f(\ba_{i})\}$ are i.i.d. Furthermore, since $\abs{f(\ba_i)} \leq \abs{x_{ij}(\bx_i^\top \bbeta_1^* - y_i)}$, we have
	\[\norm{f(\ba_i)-\EE[f(\ba_i)]}_{\psi_{1}} \lesssim \norm{x_{ij}}_{\psi_2}\norm{\bx_i^\top \bbeta_1^* - y_i}_{\psi_{2}} \lesssim \sigma_x^2c_{\rm{SNR}}\sigma+\sigma_x\sigma < \infty.\]
	By Proposition 5.16 of \cite{vershynin2010introduction}, we have
	\begin{equation}\label{eq:subexp-concentration}
	\PP\paren{\abs{\sum_{i=1}^n \left(f(\ba_i)-\EE[f(\ba_i)]\right)} > t} \leq 2\exp\brackets{-c'\min\left\{\frac{t^2}{n(C')^2}, \frac{t}{C'}\right\}},
	\end{equation}
	for some constant $c'$ and $C'$. Letting $t = \sqrt{3(C')^2n \log \max\{n, d\}/c'}$ yields that 
	\[\PP\paren{\abs{\sum_{i=1}^n \left(f(\ba_i)-\EE[f(\ba_i)]\right)} > \sqrt{3(C')^2n \log \max\{n, d\}/c'}} \leq \frac{2}{\max\{n, d\}^3}.\]

Hence,
\begin{equation*}
\PP\paren{\max_j\abs{ \frac{1}{n}\sum_{i=1}^n \omega_i^{(t)} x_{ij} ( \bx_i^\top \bbeta_1^* - y_i) 
- \EE[ \omega_i^{(t)} x_{ij} ( \bx_i^\top \bbeta_1^* - y_i) ] } > C\sqrt{\frac{ \log \max\{n, d\}}{n}} } \le \frac{2}{\max\{n, d\}^2}. 
\end{equation*}
where $C=C'\sqrt{3/c'}$.

Similary, the second concentration inequality
\begin{equation*}
\norm{  \frac{1}{n}\sum_{i=1}^n \big( \omega_i^{(t)} - p(\bz_i^\top\btheta^{*}) \big) \bz_i 
- \EE\left[ \big(\omega_i^{(t)} - p(\bz_i^\top\btheta^{*}) \big) \bz_i \right] }_\infty \le C \sqrt{ \frac{\log \max\{n, d\}}{n} }
\end{equation*}
is equivalent to 
\begin{equation*}
\max_j \abs{  \frac{1}{n}\sum_{i=1}^n \big( \omega_i^{(t)} - p(\bz_i^\top\btheta^{*}) \big) z_{i, j}
- \EE\left[ \big(\omega_i^{(t)} - p(\bz_i^\top\btheta^{*}) \big) z_{i, j} \right] } \le C \sqrt{ \frac{\log \max\{n, d\}}{n} }.
\end{equation*}
Since $\omega_i^{(t)}$ and $p(\bz_i^{\top}\btheta^*)$ are bounded,
\begin{equation*}
\norm{\big( \omega_i^{(t)} - p(\bz_i^\top\btheta^{*}) \big) z_{i, j} }_{\psi_1} \lesssim \norm{ \omega_i^{(t)} - p(\bz_i^\top\btheta^{*})}_{\psi_2} \norm{z_{i, j} }_{\psi_2}  < \infty,
\end{equation*}
and hence we can similarly establish the desired result. 

The third inequality, 
\[\norm{\frac{1}{n}\sum_{i=1}^n \left[ \omega_i^{(t)} \bx_i\bx_i^{\top} \right]
	- \EE\left[ \omega_i^{(t)}\bx_i\bx_i^{\top}  \right]}_{\max}
 \le C \sqrt{ \frac{\log \max\{n, d\}}{n}},\]
 is equivalent to
\[\max_{j, k}\abs{\frac{1}{n}\sum_{i=1}^n \left[ \omega_i^{(t)} x_{i, j}x_{i, k} \right]
	- \EE\left[ \omega_i^{(t)}x_{i, j}x_{i, k}  \right]}
\le C \sqrt{ \frac{\log \max\{n, d\}}{n}}.\]
Since 
\[\norm{ \omega_i^{(t)}x_{i, j}x_{i, k}  }_{\psi_1} \lesssim \norm{x_{i, j}}_{\psi_2} \norm{x_{i, k} }_{\psi_2}  < \infty,\]
similar to \eqref{eq:subexp-concentration}, there exist constants $c''$ and $C''$ such that
\begin{equation}
	\PP\paren{\abs{\sum_{i=1}^n \left(\omega_i^{(t)}x_{i, j}x_{i, k}-\EE\left[\omega_i^{(t)}x_{i, j}x_{i, k}\right]\right)} > t} \leq 2\exp\brackets{-c''\min\left\{\frac{t^2}{n(C'')^2}, \frac{t}{C''}\right\}}.
\end{equation}
Letting $t = \sqrt{4(C'')^2n \log \max\{n, d\}/c''}$ yields that 
\[\PP\paren{\abs{\sum_{i=1}^n\left(\omega_i^{(t)}x_{i, j}x_{i, k}-\EE\left[\omega_i^{(t)}x_{i, j}x_{i, k}\right]\right)} > \sqrt{4(C'')^2n \log \max\{n, d\}/c''}} \leq \frac{2}{\max\{n, d\}^4}.\]
Hence,
\begin{equation*}
	\PP\paren{\max_{j, k}\abs{ \frac{1}{n}\sum_{i=1}^n\left(\omega_i^{(t)}x_{i, j}x_{i, k}-\EE\left[\omega_i^{(t)}x_{i, j}x_{i, k}\right]\right)} > C\sqrt{\frac{ \log \max\{n, d\}}{n}} } \le \frac{2}{\max\{n, d\}^2}. 
\end{equation*}
where $C=2C''/\sqrt{c''}$.


\end{document}